\documentclass[runningheads,envcountsect,11pt]{llncs}
\usepackage{fullpage}
\usepackage[T1]{fontenc}
\usepackage{graphicx}
\usepackage{booktabs}
\usepackage[misc]{ifsym}

\usepackage{amsmath}
\usepackage{amsfonts}
\usepackage{centernot}
\usepackage[colorlinks,citecolor=blue,urlcolor=blue,linkcolor=blue]{hyperref}
\usepackage[capitalize]{cleveref}
\usepackage[hide=true,setmargin=false,marginparwidth=0.65in]{marginalia}

\usepackage{algorithm}
\usepackage{algorithmic}
\let\classAND\AND
\let\AND\relax
\usepackage{algorithmic}

\let\AND\classAND
\AtBeginEnvironment{algorithmic}{\let\AND\algoAND}

\usepackage{xspace}

\newcommand{\WLC}{WLC\xspace}
\newcommand{\SLC}{SLC\xspace}
\newcommand{\WLCmP}{\WLC mod P\xspace}
\newcommand{\SLCmP}{\SLC mod P\xspace}

\newcommand{\cF}{\mathcal F}
\newcommand{\Perms}{\mathcal S}

\newcommand{\defn}[1]{\emph{#1}}

\newcommand{\similarity}{\text{sim}}

\newcommand{\permuteop}[2]{#1[#2]}

\makeatletter
\newcommand{\printfootnotesymbol}[1]{%
  \textsuperscript{\@fnsymbol{#1}}%
}
\makeatother

\begin{document}

\title{Simultaneous linear connectivity of neural networks modulo permutation}

\author{Ekansh Sharma \thanks{Equal contribution. \\ Correspondence to: \href{mailto:ekansh@cs.toronto.edu}{\texttt{ekansh@cs.toronto.edu}}, \href{mailto:devin.kwok@mail.mcgill.ca}{\texttt{devin.kwok@mail.mcgill.ca}}, 
\href{mailto:gkdz@google.com}{\texttt{gkdz@google.com}}%
} \inst{1}  \and
Devin Kwok \printfootnotesymbol{1} \inst{2} \and
Tom Denton\inst{3} \and
Daniel M. Roy\inst{1} \and \\
David Rolnick\inst{2} \and 
Gintare Karolina Dziugaite\inst{4}
}
\authorrunning{E. Sharma et al.}
\authorrunning{E. Sharma et al.}

\institute{University of Toronto; Vector Institute
\and
McGill University; Mila Quebec AI Institute
\and
Google Research
\and
Google DeepMind
}

\maketitle              %

\begin{abstract}
Neural networks typically exhibit permutation symmetry, as reordering neurons in each layer does not change the underlying function they compute.
These symmetries contribute to the non-convexity of the networks' loss landscapes, since linearly interpolating between two permuted versions of a trained network tends to encounter a high loss barrier.
Recent work has argued that permutation symmetries are the \emph{only} sources of non-convexity, meaning there are essentially no such barriers between trained networks if they are permuted appropriately.
In this work, we refine these arguments into three distinct claims of increasing strength. 
We show that existing evidence only supports ``weak linear connectivity''---that for each pair of networks belonging to a set of SGD solutions, there exist (multiple) permutations that linearly connect it with the other networks.
In contrast, the claim ``strong linear connectivity''---that for each network, there exists one permutation that \emph{simultaneously} connects it with the other networks---is both intuitively and practically more desirable.
This stronger claim would imply that the loss landscape is convex after accounting for permutation, and enable linear interpolation between three or more independently trained models without increased loss.
In this work, we introduce an intermediate claim---\NA{that for certain sequences of networks, there exists one permutation that simultaneously aligns matching pairs of networks \emph{from these sequences}}.
Specifically, we discover that a single permutation aligns sequences of iteratively trained as well as iteratively pruned networks, meaning that two networks exhibit low loss barriers at each step of their optimization and sparsification trajectories respectively.
Finally, we provide the first evidence that strong linear connectivity may be possible under certain conditions, by showing that barriers decrease with increasing network width when interpolating among three networks.

\end{abstract}

\section{Introduction}

The loss landscape of neural networks is well known to be non-convex. Indeed, linearly interpolating between the weights of two independently trained networks will typically traverse significant loss/error \defn{barriers}, meaning the interpolated networks have much higher loss / 0--1 error relative to the two endpoints. There are a number of situations where barriers do not arise, however. Empirically, stochastic gradient descent (SGD) tends to produce training trajectories and basins of attraction that are surprisingly ``convex-like''. For instance, low error barriers have been observed between certain global minima \cite{frankle2020linear,nagarajan2019uniform,wortsman2021learning}, and interpolating between points on the same training trajectory leads to monotonically decreasing error \cite{goodfellow2014qualitatively,lucas2021analyzing,vlaar2022can}.

In this work, we focus on \defn{linear (mode) connectivity} \cite{frankle2020linear,nagarajan2019uniform} where a lack of significant loss/error barriers between networks indicates a certain degree of flatness or even convexity in a region of the loss/error landscape. While linear connectivity has previously been observed in the specific situations described above, recent work conjectures that all \defn{SGD solutions} are linearly connected \defn{modulo permutation}---meaning that there exists permutations (function-preserving per-layer reshuffling of neurons) that enable any networks trained with the same dataset and SGD procedure to be linearly interpolated with low barriers \cite{entezari2021role,ainsworth2022git,benzing2022random}.

What is not yet clear from previous work is which networks can be aligned by a particular permutation. At least two competing claims have emerged: one, where a single permutation is sufficient to mutually connect all network pairs in a set, and two, where a single permutation is only guaranteed to connect one pair of networks. Prior work has described and modelled the former claim, which we call \defn{strong} linear connectivity (\cref{def:strong-lcmodp}), with the goal of showing that the loss landscape is convex for the set of SGD solutions after accounting for permutations \cite{entezari2021role}. However, experimental work has thus far focused on the latter claim, which we call \defn{weak} linear connectivity (\cref{def:weak-lcmodp}), in which individual pairs of networks can be aligned after training \cite{entezari2021role,ainsworth2022git,pena2023re} or at initialization \cite{benzing2022random}.

In this work, we make the first distinction between different claims of linear mode connectivity modulo permutation in order to disambiguate between existing results. 
We define and give evidence for an intermediary claim called \defn{simultaneous weak} linear connectivity (\cref{def:simultaneous-lcmodp}), where a single permutation connects matching pairs of networks along two sequences of iteratively transformed networks.
We verify simultaneous connectivity for sequences of networks resulting from SGD training trajectories, as well as sequences of networks resulting from iterative magnitude pruning (IMP). Finally, we provide the first support for the existence of strong connectivity, by showing that barriers among permuted network triplets reduce as network width increases. Based on this evidence, we conjecture that strong connectivity is possible given sufficiently wide networks.

\paragraph{Contributions.} 
\begin{enumerate}
    \item  We define three different versions of linear mode connectivity modulo permutation, situating prior work into a more precise framework.
    We argue that prior evidence only exists for \defn{weak linear connectivity modulo permutation} (see \cref{sec:lc-mod-p}, and \cref{app:definitions} for a detailed discussion contrasting our definitions and prior work).

    \item  We demonstrate the existence of \defn{simultaneous weak linear connectivity modulo permutation} between training trajectories. Given a pair of networks and their SGD training trajectories, a single permutation of neurons can be found that aligns not merely the final trained networks, but also each corresponding pair of partially trained networks so that they are linearly connected with low error barrier.
    Moreover, we find that this result holds for the entire linearly connected basin around each network, so that the same permutation aligns all networks from the same linearly connected mode to a target network from a different mode (\cref{sec:trajectories}).

    \item 
    We show that simultaneous weak connectivity also exists between iteratively pruned networks.
    The same permutation aligning two dense trained networks can also be used to align successive iterations of sparse subnetworks found via iterative magnitude pruning (IMP) for lottery tickets \cite{frankle2018lottery}.
    We also find that same alignment allows the sparse mask (i.e., ``winning ticket'') found via IMP on one network to be reused on another network, meaning that applying the permuted mask at initialization allows the latter network to be trained to the same accuracy as the original dense network (\cref{sec:matchingviadense}).

    \item
    We provide the first evidence towards \defn{strong linear connectivity modulo permutation}, by showing that the barrier between two independently trained networks can be reduced by alignment with a third, unrelated network.
    As this effect scales with width, we conjecture that sufficiently wide networks exhibit strong connectivity (\cref{sec:evidence-strong-lcmodp}).
    
    \item We identify specific limitations of weight matching algorithms: namely that they are equivalent in performance to activation matching only on trained networks, they rely on larger magnitude weights which appear later in training, they align lower layers earlier in training, and their success or failure is unrelated to a network's stability to SGD noise.
    
\end{enumerate}

\section{Methods}

\subsection{Preliminaries}

Let $A$ and $B$ denote two parameterizations of a given neural network architecture, which are optimized over the loss function $\ell$. We take $\ell(X)$ to mean evaluating the network $X$ over a fixed data distribution.

\paragraph{Linear Mode Connectivity and Loss Barrier.} 
The \defn{loss barrier} 
is defined as the maximum increase in loss on the linear path between $A$ and $B$: 
\begin{align}
\label{equation-barrier}
\sup_{\alpha \in (0,1)} \ell(\alpha A + (1-\alpha) B)  - \alpha \ell(A) - (1-\alpha) \ell(B).
\end{align}
Barriers are typically maximized at $\alpha=0.5$.
Our experiments measure barriers in terms of both cross entropy and 0--1 classification error. We refer to barriers in the latter case as \defn{error barriers}.
A pair of trained networks are said to be \defn{linearly connected} if the error barrier between them is below a threshold, which is typically set to the expected Monte Carlo noise of the empirical loss estimator.
As our experiments mainly compare the relative size of barriers, we informally say that networks are \defn{essentially linearly connected} if the error barrier between them is significantly lowered after applying a permutation (i.e., the loss basin of the network is determined up to some small error).

\paragraph{Sparsity and Iterative Magnitude Pruning.}

We study sparse subnetworks obtained by an iterative magnitude pruning (IMP) procedure which is used to find ```winning tickets'' in the lottery ticket hypothesis \cite{frankle2018lottery,frankle2020linear}. For a full description of the pruning algorithm, see \cref{app:algorithms-imp}. IMP has been shown to identify non-trivially sparse networks (also called ``winning tickets'') which are \defn{matching},  meaning they have the same accuracy as their dense counterpart \cite{frankle2020linear}. We refer to such networks as \defn{IMP subnetworks}. Sparse networks pruned for $L$ iterations of IMP are also linearly connected to the next iteration $L+1$, provided that both networks are matching \cite{paul2022unmasking}.

\subsection{Aligning networks via permutation}
\label{sec:mainalgs}

The intermediate neurons of a neural network can be relabelled without changing its output. Formally, we define a fully-connected feedforward network
\begin{align*}
f({\bf x}_0) = {\bf x}_K,
&&
{\bf x}_{i} = \sigma({\bf W}_i {\bf x}_{i-1} + {\bf b}_i),
\end{align*}
with input $x_0$, layers $i \in \{1, 2, \dots, K\}$, weights $W_i$, biases $b_i$, and pointwise activation function $\sigma$.
Each ${\bf x}_i$ with $i \in \{1, 2, \dots, K-1\}$ can be permuted as
\begin{align*}
{\bf P}_i {\bf x}_i = \sigma \left({\bf P}_i {\bf W}_i {\bf P}^\top_{i-1} {\bf x}_{i-1} + {\bf P}_i {\bf b}_i\right),
\end{align*}
where ${\bf P}_i$ are permutation matrices in the set of permutation symmetries $\mathcal{S}_{d_i}$, $d_i$ being the dimensions of ${\bf x}_i$.
If ${\bf P}_0$ and ${\bf P}_K$ are the identity, $f$ can be rewritten:\footnote{See \cref{app:algorithms} for details on handling other types of layers.}
\begin{align*}
&{\bf W}'_i = {\bf P}_i {\bf W}_i {\bf P}^\top_{i-1},
\qquad {\bf b}'_i = {\bf P}_i {\bf b}_i,
&& \text{where } %
{\bf P}_0 = \mathbb{I}_{d_0},\ 
{\bf P}_K = \mathbb{I}_{d_K},\ 
\text{and }
{\bf P}_i \in \mathcal{S}_{d_i}.
\end{align*}
For convenience, we identify the set of permutations $\{{\bf P}_0, {\bf P}_1, \dots, {\bf P}_K \}$ as $P$, and use $\permuteop{P}{B}$ to indicate a permutation of $B$ via $P$ (i.e. ${\bf P}_i {\bf W}_i {\bf P}^\top_{i-1}$ and ${\bf P}_i {\bf b}_i$).
To determine the linear connectivity of neural networks \defn{modulo permutation}, we wish to \defn{align} networks by finding a permutation that maximizes
\begin{equation*}
P^\star = \sup_{P \in \mathcal{S}} \similarity(A, \permuteop{P}{B}),
\end{equation*}
where $A$ and $B$ are networks, $P$ is a permutation from $\mathcal{S}$, and $\similarity$ is a function measuring the similarity of two networks. To minimize error barrier with $P^\star$, $\similarity$ should ideally be the inverse of the error barrier between $A$ and $B$. However, to simplify the problem we instead define $\similarity$ using two heuristics \defn{weight matching} and \defn{activation matching} developed by prior works \cite{wang2019federated,singh2020model,ainsworth2022git}.

Weight matching takes $\similarity$ to be the $L^2$ distance between $A$ and $\permuteop{P}{B}$. We use the Greedy-SOBLAP algorithm \cite{ainsworth2022git}, based on methods from optimal transport, to approximately minimize this distance.\footnote{Weight matching is approximate, as finding the actual optimum is NP-hard \cite{ainsworth2022git,altschuler2022wasserstein}.}
Activation matching takes $\similarity$ to be the $L^2$ distance between each network's intermediate outputs on training data \cite{ainsworth2022git,singh2020model}. We use essentially the same method as weight matching to minimize this distance. Full details on both matching algorithms can be found in \cref{app:algorithms-weight-matching}.

\section{Related work}

\paragraph{Linear mode connectivity and sparsity.}

Linear mode connectivity arose from investigations on the conditions necessary for obtaining sparse IMP subnetworks (``winning tickets'') satisfying the lottery ticket hypothesis \cite{frankle2020linear}. It was found that IMP could identify winning tickets as long as the network to be pruned remained sufficiently ``stable'' to SGD noise (e.g. random batch order and augmentation). Linear mode connectivity was defined to characterize stability in terms of the loss barrier between independent trained \defn{child networks} spawned from the same parent state. Our experiments use the same approach of spawning child networks from a parent at time $t$ to isolate the effect of SGD noise (\cref{fig:trajectoryalignment} right).

Although the linearly connected mode to which dense networks converge is determined early in training (e.g. 3\% of training for ResNet-20 on CIFAR-10, 20\% for ResNet-50 on ImageNet), child networks spawned prior to this point will exhibit large error barriers for all networks larger than MNIST-scale \cite{frankle2020linear}.
Our experiments refine these observations, teasing apart the relative contributions of initialization and minibatches on feature emergence, and revealing that entire SGD trajectories can be aligned to yield small barriers.
Pretraining in the context of lottery tickets has since received considerable attention \cite{paul2022unmasking,paul2022lottery}.
In particular, winning ticket masks encode information about the final linearly connected basin, suggesting that they are symmetry-dependent \cite{paul2022lottery}.

\paragraph{Linear mode connectivity modulo permutation.}

The permutation symmetries of neural networks create copies of the global minima in the loss landscape. For overparameterized networks, it is know that these permutation-induced minima are connected by a non-convex, piece-wise linear manifold \cite{simsek2021geometry}.
In this vein, one may also consider if distinct minima found by SGD are actually permutations of the same global minimum.

Building on observations of linear mode connectivity in identically initialized networks \cite{frankle2020linear}, a conjecture was posed that independently initialized SGD solutions have no error barrier if the symmetries arising from permutation invariance are  accounted for \cite{entezari2021role}.
Preliminary evidence for this conjecture investigated the orbit of a single SGD solution when acted upon by permutations. The permuted versions of a network were found to fall into distinct linearly connected basins, with a similar distribution of loss barriers as the barriers between independently initialized and trained networks \cite{entezari2021role}.

Following this conjecture, numerous algorithms have been proposed to directly account for permutation invariance by aligning networks and maximizing their similarity \cite{benzing2022random,ainsworth2022git,jordan2022repair,tatro2020optimizing,yurochkin2019bayesian,pena2023re}.
In many natural settings, these algorithms have found permutations that considerably reduce or eliminates the loss barrier between two SGD solutions.

\paragraph{Aligning neural networks via permutation.}
An early method for finding permutations used the correlation or mutual information of activations to align neurons in a one-to-one, bipartite, or few-to-one fashion, discovering that the majority of representations between aligned networks were equivalent \cite{li2015convergent}.
Activation matching was also used to discover mode connectivity, wherein networks are joined by non-linear paths of low loss after permutation \cite{tatro2020optimizing}.

An optimal transport method for aligning network weights or activations has since been developed \cite{wang2019federated,singh2020model,ainsworth2022git}.
These methods find permutations that maximize similarity between two networks with either the same architecture, or differences only in layer width.
Applications include matching the weights of local (client-side) and global models for federated learning \cite{wang2019federated}, merging models that were fine-tuned on different datasets \cite{singh2020model,ainsworth2022git}, and compressing models by combining layers \cite{o2021layer}.
The connection between weight matching and optimal transport has been deepened to enable weight matching for recurrent networks \cite{akash2022wasserstein}.

The effect of weight and activation matching on linear mode connectivity was overlooked until the recent discovery of empirical evidence for linear mode connectivity modulo permutation on networks satisfying certain conditions \cite{ainsworth2022git}.\footnote{These conditions include sufficient width, and the use of layer normalization.} Although most works consider trained networks \cite{ainsworth2022git,jordan2022repair}, networks at initialization have also been observed to be linearly connected by a permutation found after training \cite{benzing2022random}.
Our work differs in that we observe linear connectivity along the entire SGD trajectory, and not only at initialization and convergence. Additionally, we use larger networks which, unlike small fully-connected networks, do not converge to the same linearly connected basin even when starting from the same initialization \cite{frankle2020linear}. Contrary to prior work on smaller networks \cite{benzing2022random}, we do not find that permutations computed at initialization reduce barriers after training.

Alongside these investigations, the optimal transport method for aligning networks has also been improved by iterating over randomized layer orders \cite{ainsworth2022git}, and resetting normalization statistics to enable the use of batch normalization \cite{jordan2022repair}.
Departing from weight matching, a permutation-finding method has also been developed that directly minimizes barriers using implicit differentiation and the Sinkhorn operator \cite{pena2023re}.

\section{Notions of linear connectivity modulo permutation}
\label{sec:lc-mod-p}
As informally stated in \cite{entezari2021role}, the property of linear mode connectivity modulo permutation (LC mod P) occurs when network parameterizations from a set $\cF$ \textit{can be permuted in such a way that there is no barrier on the linear interpolation between any two permuted elements in $\cF$}.

However, this informal notion of LC mod P does not make clear which permutations are shared between different parameterizations, which is important for many downstream applications.
For instance, in order to apply gradient updates from one parameterization to another, as may be needed in federated learning \cite{wang2019federated}, a single permutation must be shared throughout the training trajectories of two parameterizations.
In order to do model merging between three or more parameterizations without additional retraining \cite{singh2020model}, a single 
permutation must simultaneously linearly connect each parameterization to all of the other parameterizations.

To disambiguate these different versions of LC mod P, we start with the weakest interpretation, where each pair of networks is potentially aligned with a different permutation depending on that specific pair. 
\begin{definition}[Weak linear connectivity modulo permutation]
\label{def:weak-lcmodp}
Let $\cF$ be a set of networks.
We say that \defn{weak linear connectivity modulo permutation (\WLCmP)} 
holds on $\cF$ if there exists a map 
$Q: \cF \times \cF \to \Perms$ 
such that, for all $A, B \in \cF$, the networks $A$ and $\permuteop{Q(A, B)}{B}$ are linearly connected.
\end{definition}

In contrast to this weak notion, there is a natural strong notion,
where, for a given class of networks, there is a permutation of every network such that all pairs of permuted networks are linearly connected:
\begin{definition}[Strong linear connectivity modulo permutation]
\label{def:strong-lcmodp}
Let $\cF$ be a set of networks.
We say that \defn{strong linear connectivity modulo permutation (\SLCmP)} holds on ${\cF}$ if there exists a map $Q' : \cF \to \Perms$ such that, for all $A ,B \in \cF$, the networks $\permuteop{Q'(A)}{A}$ and $\permuteop{Q'(B)}{B}$ are \defn{linearly connected}.
\end{definition}

Note that \SLCmP implies \WLCmP. To see this, let $Q'$ witness \SLCmP for some class $\cF$. Then one can verify that $Q$ given by $Q(A, B) = Q'(A)^{-1}Q'(B)$ witnesses \WLCmP on $\cF$, because two networks $A,B$ are linearly connected if and only if $P[A], P[B]$ are linearly connected, for every $P \in \Perms$.

Using the above language, we can summarize existing work precisely. In \cite{entezari2021role}, the authors \emph{prove} that a random pair of sufficiently wide, two-layer networks possess \WLCmP  at initialization with high probability
\cite[Theorem 3.1]{entezari2021role}. 
Note that a formal conjecture in \cite[Conjecture 1]{entezari2021role} is often misquoted as stating that \WLCmP holds on a high-probability subset of SGD-trained networks.\footnote{In fact, a strictly stronger claim is made: for a certain class of networks $\mathcal{F}$, for all $\theta_1 \in \mathcal{F}$, there is a single permutation that can be applied to $\theta_1$ removing the error barrier between the permuted $\theta_1$ and any other network in the class $\mathcal{F}$. 
Note that this also means that the networks in $\mathcal{F}$ are piece-wise linearly connected \emph{before} permuting.}
Instead, this conjecture, when combined with an additional assumption of SGD solutions being closed under permutation, implies \SLCmP for the same subset of networks (but not vice versa).
A visualization in the same paper \cite[Figure 1]{entezari2021role} matches neither their proof nor their formal conjecture, but does represent \SLCmP.
See \cref{app:definitions} for an in-depth discussion with proofs.

Empirical evidence supports \WLCmP holding on sufficiently wide, high-probability sets of SGD trained networks, although evidence is still lacking in other settings such as narrow or deep models \cite{ainsworth2022git,benzing2022random,pena2023re}.
In contrast, there is scant empirical evidence for \SLCmP, at least for standard parametrizations.

Both \WLCmP and \SLCmP are properties that hold for sets of networks, but one may also want to consider properties that hold for trajectories (sequences) of networks.
\begin{definition}[Simultaneous weak linear connectivity modulo permutation]
\label{def:simultaneous-lcmodp}
Let $\mathcal{C} \subset \cF^{T}$  be a set of sequences of networks.
We say that \defn{simultaneous weak linear connectivity modulo permutation} holds on $\mathcal{C}$ if there exists a map $Q: \mathcal{C} \times \mathcal{C} \to \Perms$ such that, 
for all $A,B \in \mathcal C$ and all $t\in T$,
the networks $A_t$ and 
$\permuteop{Q(A,B)}{B_t}$ 
are linearly connected.
\end{definition}

The property of simultaneous \WLCmP 
highlights the possibility of finding a single permutation that linearly connects two sequences of networks. 
In \cref{sec:trajectories}, we provide empirical evidence showing that sequences of networks obtained via running SGD are simultaneously \WLCmP. 
Also, in \cref{sec:matchingviadense}, we empirically show that sequences of iteratively sparsified networks are simultaneously \WLCmP.
Simultaneous \WLCmP raises the prospect of interesting applications in 
federated learning and model fusion.
\section{Empirical findings}
\label{sec:empirical}

In this section, we present our empirical findings regarding \defn{simultaneous weak linear connectivity modulo permutation} (simultaneous \WLCmP). Unless stated otherwise, these results are independent of any algorithmic concerns.
In particular, to demonstrate different notions of LC mod P, we simply need to identify \emph{one} permutation that essentially eliminates the error barrier between two networks.
Logically, the existence of such a permutation does not depend on the method by which it is found. 
On the contrary, if we cannot find a permutation that makes two networks linearly connected modulo permutation, this does not rule out the possibility that such a permutation exists and can be found by some other method. In this work, we therefore do not interpret the inability to find a permutation as evidence for or against linear connectivity. In \cref{sec:algorithmicaspects}, we discuss results that are algorithm-dependent.

Our experimental setting mimics that of \cite{ainsworth2022git} (see \cref{app:expdetails} for details such as hyperparameters, architectures, etc.).
The results presented in the main paper are for VGG-16 with layer normalization \cite{ba2016layer} trained on CIFAR-10. 
Additional results with different datasets (MNIST, SVHN, CIFAR-10, CIFAR-100) and model architectures (MLP, VGG-16, ResNet-20) are shown in \cref{app:all-trajectory-results,app:sparsesubnetworks}.

\begin{figure*}[t]
\begin{center}
\includegraphics[height=5cm]{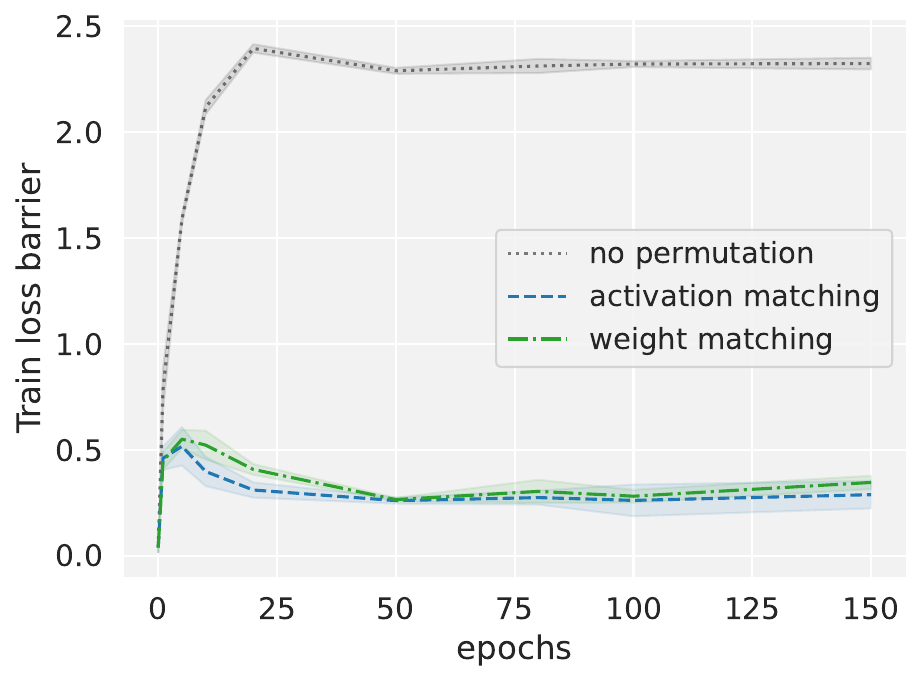}
\includegraphics[height=5cm]{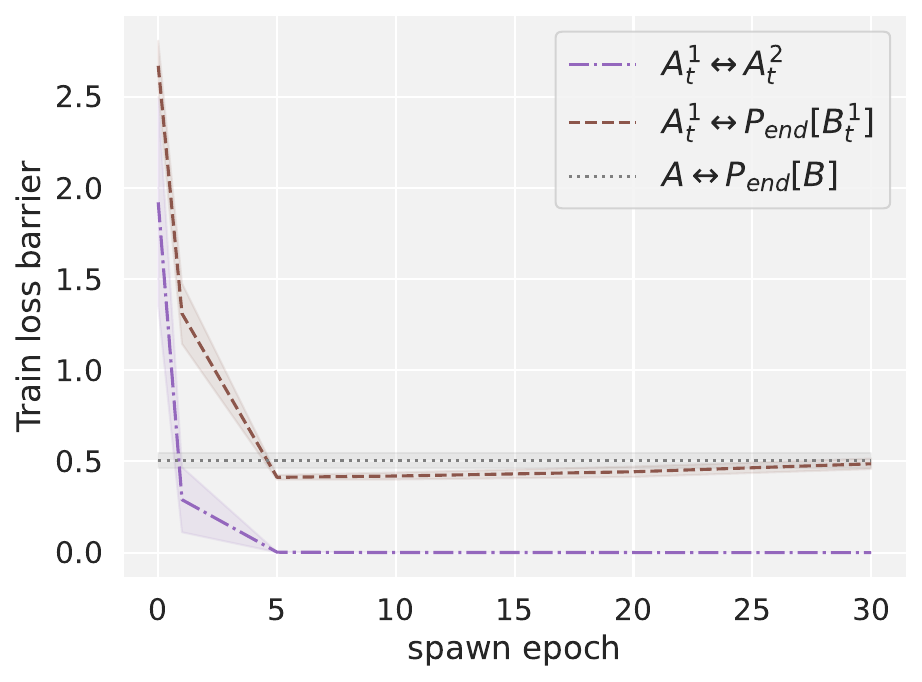}%
\caption{
\textbf{(Left)} loss barrier ($y$-axis) between networks $A_t$ and permuted $B_t$ at various training times $t$ ($x$-axis). Permutation is computed at the end of training. The green line corresponds to applying permutations found via weight matching, blue corresponds to permutations found via activation matching, and gray corresponds to no permutation.
\textbf{(Right)} loss barrier ($y$-axis) between $k$ ``child networks'' $A^k$ spawned from $A$ and permuted children $P_{\mathrm{end}}[B_t^k]$ of $B$ spawned at time $t$ ($x$-axis) from their respective ``parent'' networks (dashed brown line, $A_t^1 \leftrightarrow P_{\mathrm{end}}[B_t^1]$). $P_{\mathrm{end}}$ is computed by aligning parent networks $A$ and $B$ at the end of training (dotted grey line marks the barrier between the permuted parents, $A \leftrightarrow P_{\mathrm{end}}[B]$). These loss barriers are compared against the average loss barrier between independent child networks with the same parent (dot-dashed purple line, $A_t^1 \leftrightarrow A_t^2$). Each child is trained with a different minibatch order starting from the parent's weights at time $t$. %
}
\label{fig:trajectoryalignment}
\end{center}
\end{figure*}

\subsection{Training trajectories are simultaneously weak linearly connected modulo permutation}
\label{sec:trajectories}

We first demonstrate empirically that simultaneous \WLCmP holds for sequences of networks obtained by independent SGD training runs. 
To show this, we apply weight or activation matching after training, and observe that the resulting loss barrier between independent SGD iterates is highly reduced at all times $t$.
Formally, we consider sequences of dense networks $A_t$ and $B_t$ trained on the same data with independent initializations, minibatches, and other SGD noise, where $t$ denotes the training time. 
We compute a permutation $P_{\mathrm{end}}$ on the networks after training using weight or activation matching

In \cref{fig:trajectoryalignment} (left), we plot the loss barrier between $\permuteop{P_{\mathrm{end}}}{B_t}$ and $A_t$
As the figure shows, the error barrier after aligning with $P_{\mathrm{end}}$ drops from around $70\%$ (as measured between two differently initialized networks after convergence) to $3-8\%$ at any time $t$ in training.
This implies that the two networks remain essentially linearly connected modulo permutation throughout training.
Previously, \cite{ainsworth2022git} only showed that the final weights learned by SGD were linearly connected modulo permutation.
Our new observations provide evidence for the strictly stronger conjecture \emph{that SGD trajectories are simultaneously weak linearly connected modulo permutation}. 

In \cref{sec:weightmatchingfails}, we present the results of using permutations $P_t$ computed from earlier in training. \Cref{sec:algorithmicaspects} discusses algorithm-specific takeaways. \fTBD{GKD: can add back old text if we have space at the end}

\paragraph{The same permutation can align multiple networks from the same basin.}
\label{sec:lmc}

We next show that a single permutation that linearly connects two SGD solutions after  training also connects \emph{other} networks from the same linearly connected modes as each SGD solution. We show this by computing a permutation from the end of two independent training runs, and comparing loss barriers modulo this fixed permutation between networks obtained via an \defn{instability analysis} procedure \cite{frankle2020linear}. 

\cite{frankle2020linear} defines \defn{instability} to SGD noise at time $t$ as the error barrier after training between two child networks, $A_t^1, A_t^2$, obtained from a single parent network $A$ at time $t$, and trained independently to convergence. 
The first time $t$ from which pairs of children, on average, have no error barrier between them is referred to as the \defn{onset of linear mode connectivity}.
Such child networks are not only linearly connected at the end of training, but are also linearly connected throughout their entire training trajectories.

In \cref{fig:trajectoryalignment} (right), we plot the error barrier between the child networks $A_t^1$ and $\permuteop{P_{\mathrm{end}}}{B_t^1}$, where the permutation $P_\mathrm{end}$ brings the parent networks $A$ and $B$ to the same linearly connected mode. 
We also plot instability to SGD noise for children spawned at time $t$ ($x$-axis). The figure shows the onset of linear mode connectivity at around 5 epochs of training. 
\emph{From the onset of linear mode connectivity onward, we observe that the permutation $P_\mathrm{end}$ is able to align the other child networks throughout training.}
This is quite surprising, as the permutation $P_\mathrm{end}$ was computed on a single pair of parent runs, and is not specifically adapted to align the child networks.

\subsection{Iteratively sparsified networks are simultaneously weak linearly connected modulo permutation}
\label{sec:matchingviadense}

We next look at simultaneous weak LC mod P of IMP sparsified networks. 
We show that a permutation that linearly connects independent dense networks also linearly connects the sequences of sparse networks obtained from the same dense networks via IMP.

Formally, let $A^{(0)},A^{(1)},\dots,A^{(k)}$ be a sequence of increasingly sparse, trained networks obtained by IMP, with $A^{(0)}$ being the original dense, trained network, and $A^{(k)}$ being the sparsest network (assumed to match the accuracy of the dense network).
Let $B^{(0)},\dots,B^{(k)}$ be another such sequence obtained with the same architecture and data, but using a different initialization and minibatch order.

We identify permutations under which $A^{(i)}$ and $B^{(i)}$ are essentially linearly connected for all $i$.
Specifically, both weight matching and activation matching applied to the dense networks $A^{(0)}$ and $B^{(0)}$ gives a permutation that simultaneously aligns \emph{all} of the sparse IMP subnetworks $A^{(k)}$ and $B^{(k)}$ with lower error barrier than when naively weight matching $A^{(k)}$ to $B^{(k)}$ (\cref{fig:perm-sparse-error-barrier}, left).  

Our findings are complementary to recent work that characterizes the loss landscape of lottery tickets.
\cite{paul2022unmasking} showed that for any pruning level $k$, a sparse subnetwork $A^{(k)}$ is essentially linearly connected to its dense counterpart $A^{(0)}$.
We further demonstrate that this observation holds, modulo permutation, for networks from different initializations.

\begin{figure}[ht]
\begin{center}
\centerline{\includegraphics[height=5cm]{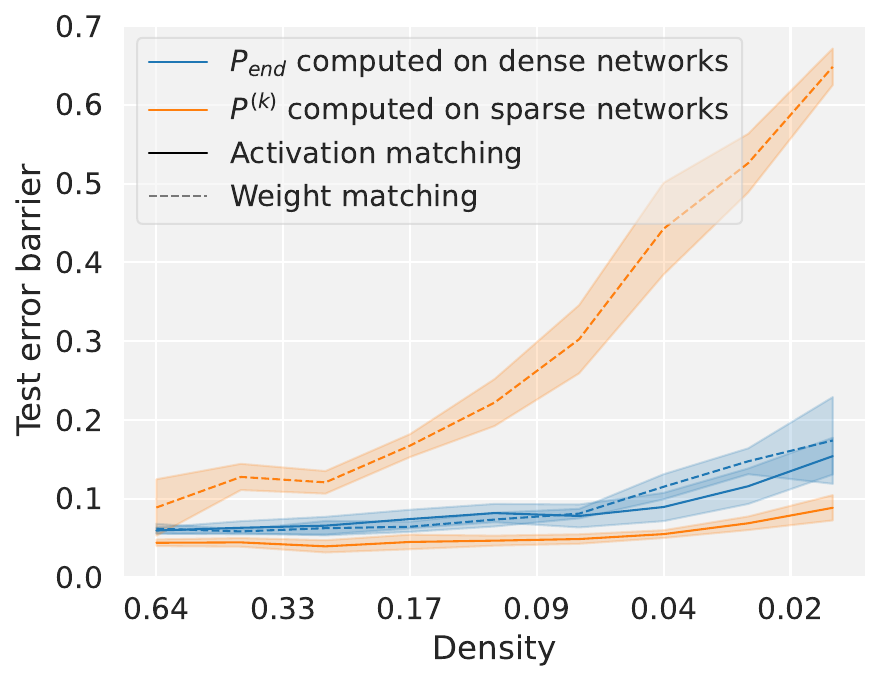}
\includegraphics[height=5cm]{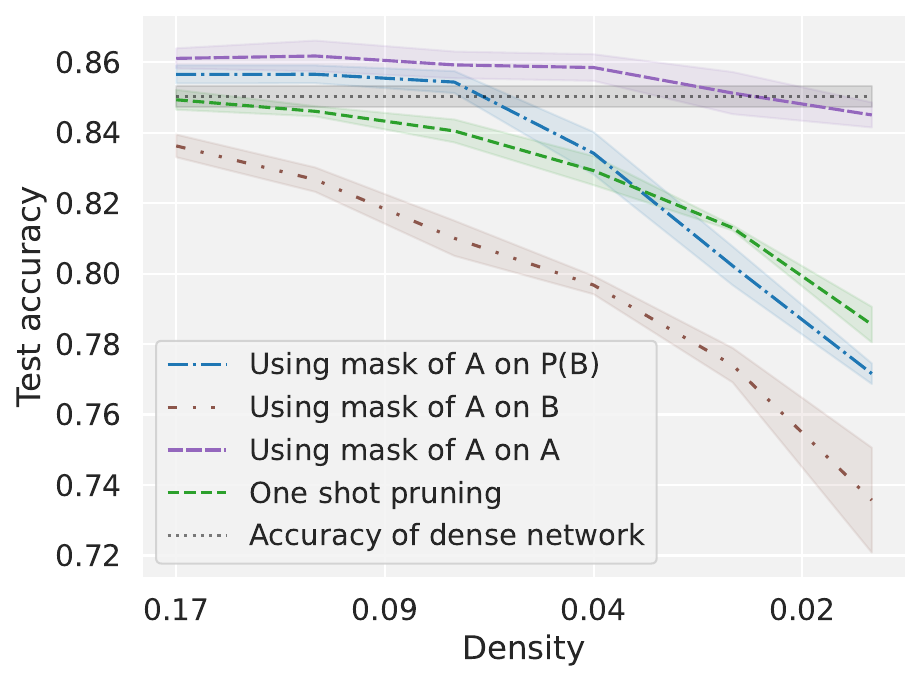}}
\caption{
\textbf{(Left)} error barrier (y-axis) between permuted sparse IMP subnetworks derived from dense networks trained with different initializations and SGD noise. At each sparsity level (x-axis), pairs of sparse networks are aligned either (1) using a permutation $P^{(k)}$ computed directly on the sparse subnetworks at sparsity level $k$ (orange), and (2) using a permutation computed from the corresponding dense networks (blue).
\textbf{(Right)} test accuracy of sparse networks using transported masks, at increasing levels of sparsity (x-axis indicates fraction of remaining weights). We use weight matching to compute the permutation.}
\label{fig:perm-sparse-error-barrier}

\end{center}
\vskip -0.3in
\end{figure}

\paragraph{Masks can be transported modulo permutation.}
\label{sec:masktransitivity}

Obtaining an IMP mask is a computationally expensive procedure that requires many iterations of pruning and retraining.
Here, we ask whether an IMP mask obtained on one network can be transported to another network via permutation.
Indeed, we find this is possible, as the transported mask outperforms one-shot pruning up to a certain sparsity level.

Formally, let $m_A^{(0)}, m_A^{(1)},\dots, m_A^{(k)}$ be a sequence of sparsity masks obtained by iteratively pruning the 20\% smallest weights of $A^{(0)}, A^{(1)}, \dots, A^{(k)}$ respectively.
Let $B$ be an independent run of the dense model which is aligned to $A^{(0)}$ at the end of training by permutation $P$.
We refer to the accuracy of the original dense model as the \defn{matching accuracy}.
At each given pruning level $i$, we consider the test accuracy of the network obtained after applying $m_A^{(i)}$ to $\permuteop{P}{B}$.
We obtain this by 
(1) rewinding the weights of $\permuteop{P}{B}$, 
(2) pruning using the mask $m_A^{(i)}$, 
and (3) retraining the pruned network to convergence.
Since we no longer run IMP on the network $B$, we compare this procedure to \defn{one-shot pruning}, where we obtain the sparse network by 
(1) pruning an equivalent fraction $(1-0.8^i)$ of the smallest weights of $B$; 
(2) rewinding the remaining weights; and (3) retraining to convergence.

\cref{fig:perm-sparse-error-barrier} (right) shows that IMP masks for network $A$ can indeed be transported to $B$ modulo permutation, while preserving matching accuracy. 
Specifically, we find that applying $m_A^{(i)}$ to $\permuteop{P}{B}$ achieves matching accuracy for IMP levels $i\leq 12$, or for sparsities $\geq 0.068$.
We also show that applying $m_A^{(i)}$ to $\permuteop{P}{B}$ achieves matching accuracy up to a higher sparsity level than the one-shot pruning procedure.
As one would expect, naively applying $m_A^{(i)}$ to $B$ without permutation fails to achieve matching accuracy at all sparsity levels.
Our results suggest that networks within the same linearly connected mode modulo permutation have similar ``winning ticket'' structures for a range of sparsity levels, and that any mask from this range achieves near-SOTA performance when applied (with appropriate permutations) to other, independently trained models. (See \cref{app:sparsesubnetworks} for other related experiments.)

\subsection{Evidence for strong linear connectivity modulo permutation}
\label{sec:evidence-strong-lcmodp}
To search for \defn{strong linear connectivity modulo permutation} (\cref{def:strong-lcmodp}) between SGD trained networks, we indirectly align two networks relative to a fixed reference network with the following procedure:
\begin{enumerate}
\item Independently train 3 networks $A, B, C$.
\item Align $A$ and $B$ with $C$ to get $\permuteop{P_{A \to C}}{A}$ and $\permuteop{P_{B \to C}}{B}$.
\item Compute the barriers between $\permuteop{P_{A \to C}}{A}$ and $\permuteop{P_{B \to C}}{B}$.
\end{enumerate}

\begin{figure}[t]
\begin{center}
    \includegraphics[height=5cm]{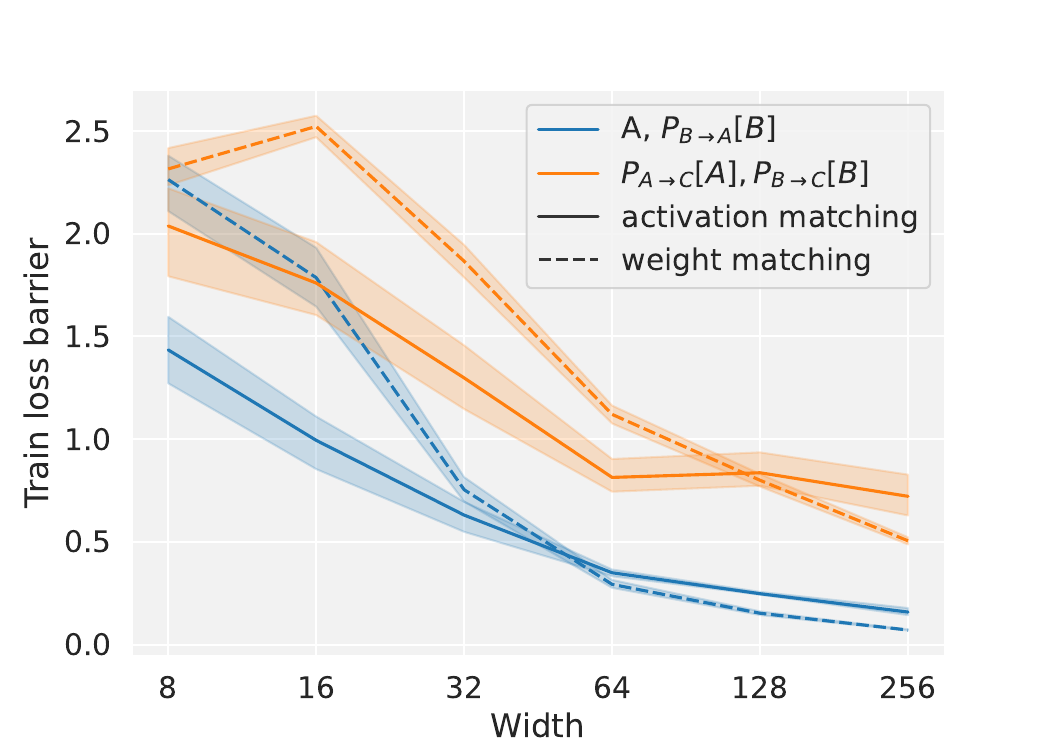}
    \includegraphics[height=5cm]{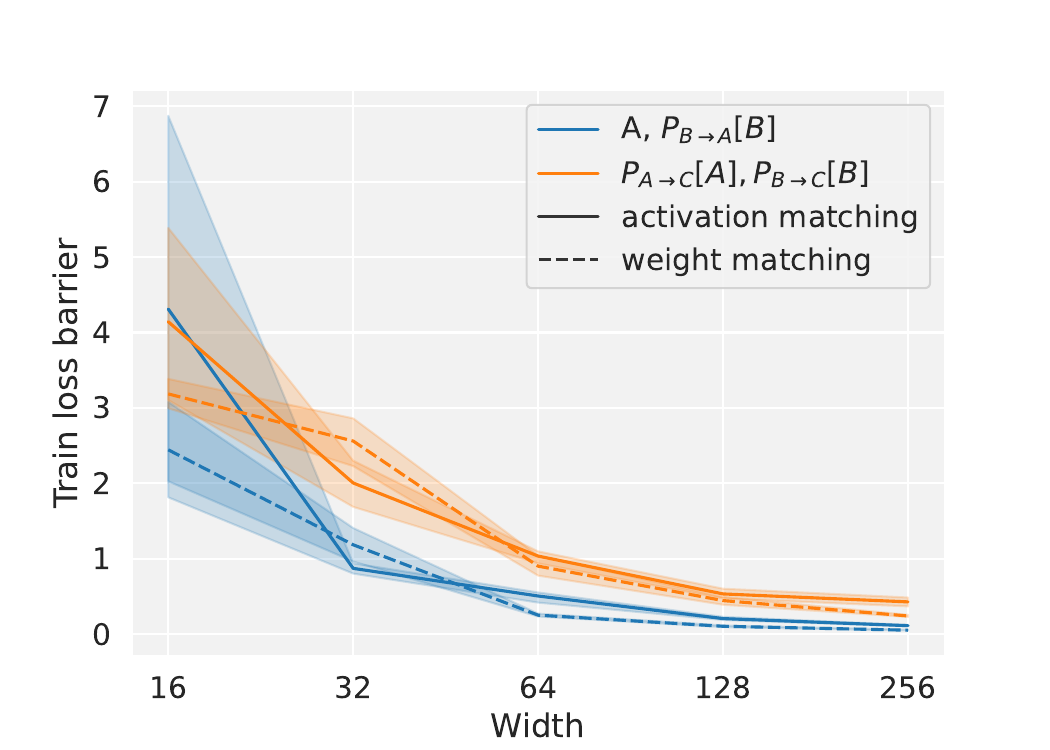}
\caption{
    Test of strong linear connectivity, comparing barriers (y-axis) of networks aligned directly or relative to a reference network (colors) via activation and weight matching (line styles).
    \textbf{(Left)} VGG-16 models of increasing width (x-axis).
    \textbf{(Right)} ResNet-20 models of increasing width (x-axis).
}
\label{fig:transitivity-main}
\end{center}
\end{figure}

If strong connectivity holds, we would expect that a single permutation of each of $A$ and $B$ would be sufficient to make them linearly connected with both $C$ and one another.

The results shown in \cref{fig:transitivity-main} show that as the width of the networks increases, the loss barrier obtained by indirect alignment between networks $A$ and $B$ decreases. However, indirect alignment still results in much higher barriers than is achievable by aligning $A$ and $B$ directly, meaning the permutations that connect to $C$ do not put $A$ and $B$ in the same linearly connected basin.
This failure to find comparable error barriers to the direct alignment may be due to limitations of the weight or activation matching algorithm, which we discuss in \cref{sec:algorithmicaspects}.
Nevertheless, \cref{fig:transitivity-main} supports the conjecture that with greatly increased width, strong linear connectivity modulo permutation is possible between SGD solutions. 

\section{Algorithmic aspects of network alignment}
\label{sec:algorithmicaspects}

As eluded earlier, %
we only have evidence supporting strong linear connectivity modulo permutation for very wide networks.
This limitation can perhaps be attributed to the weight matching and activation matching algorithms that we use to align networks (see \cref{sec:mainalgs} for descriptions and references for these two algorithms).
In this section, we identify potential weaknesses that may hinder the discovery of \SLCmP in more general settings, and summarize our algorithm-specific observations.
Our main observation is:
\textbf{activation matching typically outperforms weight matching, although neither heuristic is superior in all cases.}

\begin{figure*}[t]
  \begin{minipage}[c]{0.5\textwidth}
  \begin{center}
    \includegraphics[height=5cm]{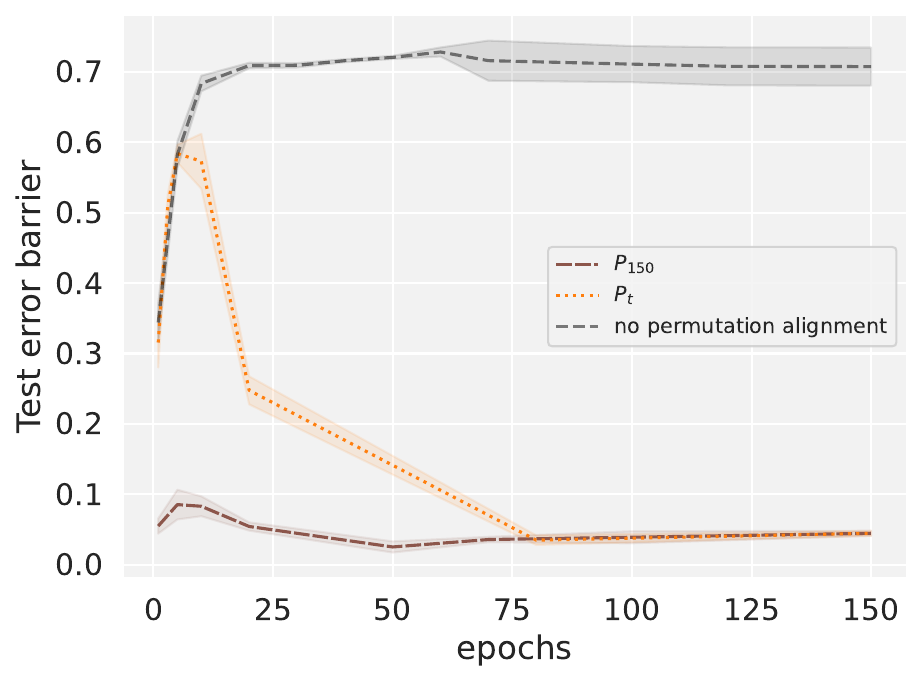}%
  \end{center}
  \end{minipage}\hfill
  \begin{minipage}[c]{0.5\textwidth}
    \caption{
    The error barrier (y-axis) between networks $A_t$ and permuted $B_t$ at different training checkpoints $t$ (x-axis). The brown line corresponds to applying a permutation $P_{\mathrm{end}}$ found at the end of training, which successfully eliminates the error barrier between the networks at nearly every epochs. The orange line corresponds to a permutation $P_t$ computed and applied at time $t$.
    } \label{fig:trajectoryalignment-time}
  \end{minipage}
\end{figure*}

\paragraph{Weight matching works as well as activation matching for wide networks at convergence.} %
In \cref{fig:trajectoryalignment-time}, we present experiments on weight-matching between networks at different iteration $t$ during training, showing that the weight matching algorithm fails to align networks before approximately 80 epochs (see  \cref{sec:weightmatchingfails} for more discussion). 
On the other hand, for pruned, sparse networks and networks not trained to convergence, weight matching fails to find permutations that eliminate the error barrier, even when we know that such permutations exist and can be found by aligning the corresponding dense networks.
See \cref{sec:matchingviadense} (\cref{fig:perm-sparse-error-barrier}) and \cref{app:sparsesubnetworks} for sparse subnetwork matching results.

\paragraph{Weight matching depends on large-magnitude weights.} In 
\cref{fig:weight-matching} (left), we experimentally show that weight matching is unaffected when weights are pruned by smallest magnitude, but performs worse when weights are randomly pruned (details in \cref{sec:weight-matching-magnitude}).
\paragraph{Weight matching can align lower layers earlier in training.} 
We determine when in training and at which layers the permutations found by weight matching become effective at reducing loss/error barriers, by aligning a subset of layers using the value of their weights from earlier in training (details in \cref{sec:partialalign}). \cref{fig:weight-matching} (right) shows that when this partial alignment is applied in a bottom-up fashion, the lowest layers are aligned sooner than higher layers.
This finding is consistent with previous work that has shown that network layers closer to the input converge first \cite{raghu2017svcca}.

\paragraph{Network ``stability'' at initialization is neither necessary nor sufficient for permutation alignment.}
\cite{frankle2020linear} defined a network to be \emph{stable} to SGD noise if at some time $t$, two independent runs (with different minibatch randomization) trained from that point forward produce networks that remain linearly connected. It so happens that the experiments of \cite{ainsworth2022git}, which largely reported the success of weight matching, are between stable networks.
However, in \cref{fig:sufficient-necessary} we perform interventions to show that linear connectivity modulo permutation does not require network stability, and vice versa (details in \cref{app:matchingifstable}). 

\begin{figure}[t]
\begin{center}
\includegraphics[width=0.45\textwidth,height=5cm]
{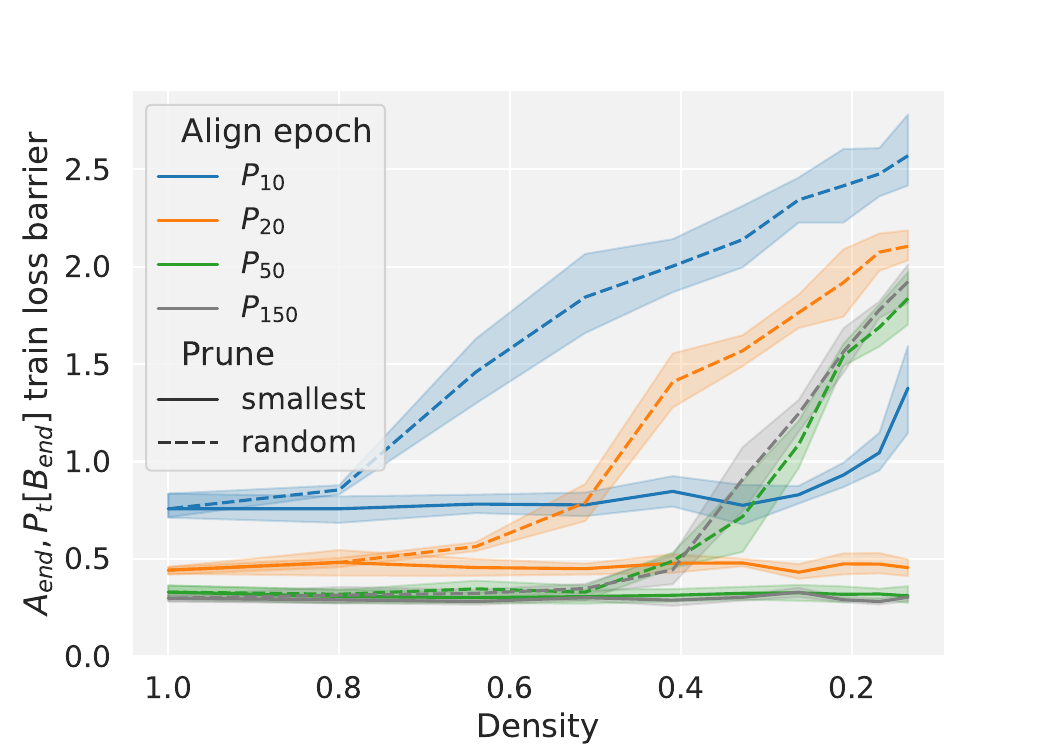}
\includegraphics[width=0.45\textwidth, height=4.6cm]
{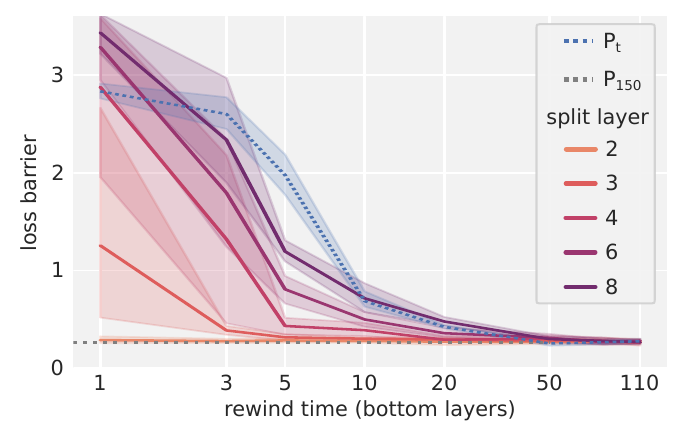}

\caption{
    \textbf{(Left)} effect of magnitude pruning on performance of weight matching algorithm. Loss barrier (y-axis) for networks after training which are aligned with a permutation found via weight matching. Permutations are computed on checkpoints at different training epochs (colors) which are first sparsified (x-axis) via random (dashed) or magnitude (solid) pruning.
    \textbf{(Right)}
loss barriers at the end of training between a pair of networks under ``bottom-up'' partial alignment, which concatenates $P_t$ from input up to layer $k$ with $P_{\mathrm{end}}$ for the remaining layers. The rewind time $t$ is the x-axis, and the split point $k$ is indicated by line color. Barriers for $P_t$ computed at time $t$ (x-axis) and at the end of training are included as baselines (dotted lines).
}
\label{fig:weight-matching}
\end{center}
\end{figure}

 \begin{figure}[h!]
\begin{center}
\includegraphics[width=0.4 \textwidth, height=6cm, trim={12cm 12cm 0 0}, clip]
{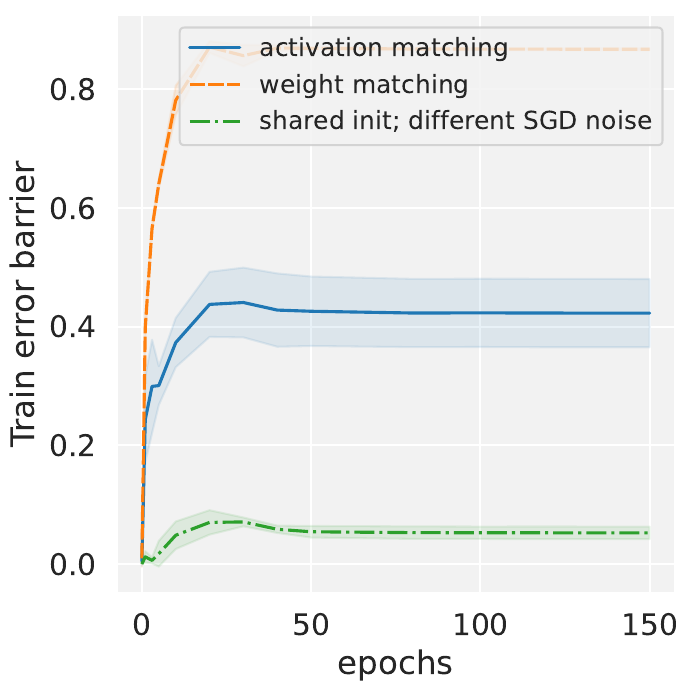}
\includegraphics[width=0.4 \textwidth, height=6cm, trim={12cm 12cm 0 0}, clip]
{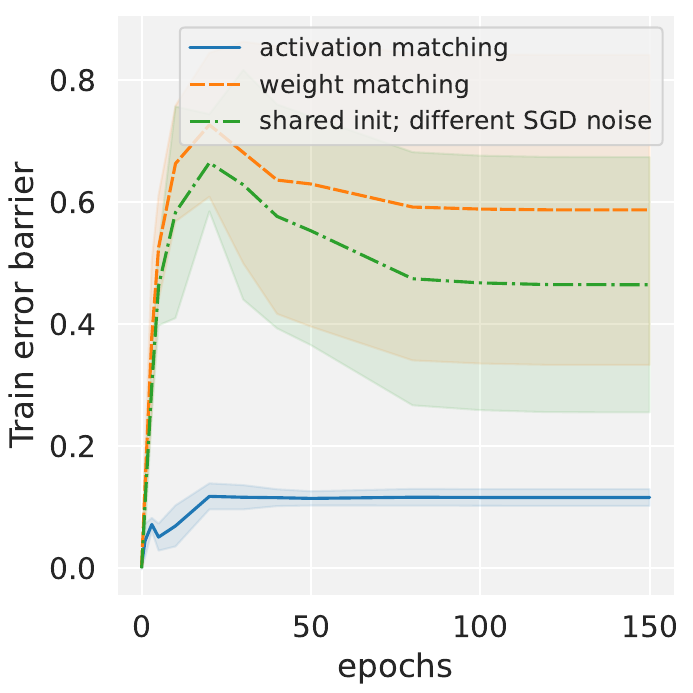}
\caption{
 Stability at initialization is neither sufficient \textbf{(left)} nor necessary \textbf{(right)} for finding permutation symmetries.
}
\label{fig:sufficient-necessary}
\end{center}
\end{figure}

\section{Conclusion}

In this work, we disambiguate different notions of linear connectivity modulo permutation. We argue that empirical evidence available in the literature only supports \defn{weak linear connectivity modulo permutation} (\cref{def:weak-lcmodp}).

We show that a permutation aligning two networks up to linear connectivity also simultaneously aligns a large number of related networks. We call this notion of linear connectivity \defn{simultaneous weak linear connectivity} (\cref{def:simultaneous-lcmodp}).
Specifically, we show that the permutation aligning two fully trained networks also aligns the corresponding partially trained networks at any point throughout the training process. 
We next consider sparse matching subnetworks pruned by iterative magnitude pruning, finding that it is possible to simultaneously align all such subnetworks by applying a permutation that aligns the corresponding dense networks. 
Furthermore, we show that using the dense-to-dense permutation, sparse masks learned from one network can be reused effectively on another network.

We also define a notion of \defn{strong linear connectivity modulo permutation} which moves beyond the earlier pairwise notion of linear connectivity. Most notably, strong linear connectivity implies that a large set of independently trained networks are all simultaneously and mutually linearly connected after suitable permutations. To this end we experiment with aligning and interpolating between three networks at a time, and demonstrate that strong linear connectivity may be possible with very wide networks.

Finally, our investigations shed light on the behavior of two principal classes of heuristics used in permutation-finding algorithms: weight matching and activation matching. Although neither heuristic is clearly dominant, we observe that activation matching is generally more robust over more settings than weight matching. We find that weight matching depends on large magnitude weights that appear later in training. In keeping with the common wisdom that the initial layers of neural networks are learned earlier, we also find that weight matching identifies the permutation of initial layers earlier.

\begin{credits}
\subsubsection{\ackname}
The authors would like to thank Tiffany Vlaar and Utku Evci for feedback on a draft and various ideas, as well as Udbhav Bamba for preliminary implementation work. DMR and DR are supported by Canada CIFAR AI Chairs and NSERC Discovery Grants. The authors also acknowledge material support from NVIDIA in the form of computational resources, and are grateful for technical support from the Mila IDT and Vector teams in maintaining the Mila and Vector Compute Clusters.
Resources used to prepare this research were provided, in part, by Mila (mila.quebec), the Vector Institute (vectorinstitute.ai), the Province of Ontario, the Government of Canada through CIFAR, and companies sponsoring the Vector Institute (\url{www.vectorinstitute.ai/partnerships/current-partners/}).

\subsubsection{\discintname}
The authors have no competing interests to declare that are
relevant to the content of this article.
\end{credits}
\bibliographystyle{splncs04}
\bibliography{citations}

\newpage
\appendix
\onecolumn

\section{Algorithms}
\label{app:algorithms}

\paragraph{Iterative magnitude pruning for lottery tickets.}
\label{app:algorithms-imp}
In the context of the lottery ticket hypothesis, iterative magnitude pruning is a procedure that identifies matching subnetworks (``winning tickets'') up to a high level of sparsity \cite{frankle2018lottery}. The procedure operates as follows:
\begin{enumerate}
    \item Train a dense network for $T$ iterations until convergence.
    \item Identify and mask (set to zero) the 20\% smallest unpruned weights according to their magnitude after convergence.
    \item Rewind the remaining unpruned weights to their values at some iteration $t$.
    \item Retrain the pruned network to convergence for $T - t$ iterations.
    \item Repeat steps 2--4 until the desired sparsity is reached.
\end{enumerate}
We refer to $L$ repetitions of steps 2--4 as sparsity level $L$, where each additional level prunes $20\%$ of the previous level's weights to give a total sparsity fraction of $0.8^L$.

IMP has been shown to identify non-trivial sparse subnetworks (``winning tickets'') that have matching accuracy when trained from iteration $t$ \cite{frankle2018lottery,frankle2020linear,renda2020comparing}. 
For IMP to succeed however, the rewinding epoch $t$ must be a point in training at which the masked (sparse) subnetworks are \defn{stable to SGD noise}, meaning that two sparse subnetworks independently trained from $t$ to convergence remain linearly connected despite random effects (e.g. different batch orderings) \cite{frankle2020linear}.
In our work, we refer to matching subnetworks obtained via IMP as \defn{IMP subnetworks}.

\paragraph{Weight matching.}
\label{app:algorithms-weight-matching}
We use the weight matching algorithm ``Greedy-SOBLAP'' to align weights \cite{ainsworth2022git}. The algorithm operates as follows.

Consider two networks $A, B$ with identical architectures.
Each hidden layer $k$ of the network admits an independent permutation of parameters.
Let $A^{(k)}, B^{(k)} \in \mathbb{R}^{p \times q}$ be a matrix consisting of the $k$\textsuperscript{th} set of permutable parameters from network $A$ and $B$ respectively, where $p$ is the common dimension to be permuted, and $q$ are the remaining dimensions of the parameters combined.
Note that some parameters (e.g. weight matrices) appear in more than one permutation, but will be permuted in different dimensions by each permutation.

Normalization layers typically include a learnable affine transform.
The parameters of this transform share the same permutation as the previous hidden layer.
In the case of residual networks, all skip connections that do not have a shortcut layer share the same permutation.

Let $P^{(k)} \in \mathcal{S}_{p}$ be a permutation of the $k$\textsuperscript{th} parameter set.
We can optimize any particular $P^{(k)}$ by first finding the pairwise similarity of the parameters over dimension $p$, and then solving the optimal transport problem to find a 1-to-1 mapping which maximizes the sum of pairwise similarities:
\begin{align*}
    P^{(k)} &= \underset{P \in \mathcal{S}_p}{\operatorname{argmax}}\{\similarity(A^{(k)}, P[B^{(k)}])\}
    \\
    &= \underset{P \in \mathcal{S}_p}{\operatorname{argmax}}\{\operatorname{trace} (G P^T) \},
    \quad \text{where}
    \, G = \langle A^{(k)}, B^{(k)} \rangle.
\end{align*}
To compute the Gram matrix $G$ we use the Euclidean inner product which minimizes the $L^2$ distance between the weights of $A^{(k)}$ and $B^{(k)}$.
However, in principle any inner product can be applied.

We then apply the following iterative procedure:
\begin{enumerate}
    \item
    Randomly shuffle the order in which sets of permutations are solved.
    \item
    For each parameter set $k$,
        \begin{enumerate}
        \item 
        Compute the Gram matrix (matrix of pairwise similarities) between $A^{(k)}$ and $B^{(k)}$ using a linear kernel: $G = A^{(k)} B^{(k)\top}$.
        \item 
        Use a linear sum assignment algorithm to find a permutation which maximizes pairwise similarity (e.g. scipy.optimize.linear\_sum\_assignment):
        \[
        P^\star = \text{argmax}_{P \in S_p} \{\text{trace} (GP^T)\}.
        \]
        \item
        Apply the new permutation $P^\star$ to $B^{(k)}$ and $P^{(k)}$.
    \end{enumerate}
    \item
    Repeat until either pairwise similarity does not improve in every permutation set, or to a maximum number of iterations.
\end{enumerate}

Note that the weight matching algorithm does not directly optimize for the loss barrier to be small. 
In previous work \cite{jordan2022repair,ainsworth2022git,benzing2022random} and our experiments, however, we see that the loss barrier is highly reduced after accounting for permutations as found via this algorithm.

Similar algorithms have also been proposed in other works for federated learning and model merging \cite{li2015convergent,wang2019federated,singh2020model}. However, the Greedy-SOBLAP algorithm includes two innovations that significantly reduce barriers between permuted networks \cite{ainsworth2022git}. These innovations are (1) randomizing the order in which layers are aligned, and (2) running multiple iterations of the alignment algorithm. The reason that these innovations are needed is because two permutations can act on the same weight matrix along different dimensions, which makes the problem non-convex.

\Cref{alg:sparse-weight-matching-rewind-accuracy} provides pseudocode following the above description.
We reuse and extend the implementation provided in git-rebasin \cite{ainsworth2022git}:
\hyperlink{https://github.com/samuela/git-re-basin}{https://github.com/samuela/git-re-basin}.

\begin{algorithm}[tb]
   \caption{Greedy-SOBLAP for weight matching.}
   \label{alg:sparse-weight-matching-rewind-accuracy}
\begin{algorithmic}
\STATE {\bfseries Require:} subroutines $\operatorname{LSA}$ for linear sum assignment, $\operatorname{PARAMS}$ for getting set of permuted parameters, and $\operatorname{PERMUTE}$ for applying permutation to parameters.
   \STATE {\bfseries Input:} trained dense parameters $A, B$.
   \FOR{all $k$ hidden layers}
   \STATE $A^{(k)} \leftarrow \operatorname{PARAMS}(A, k)$
   \STATE $P^{(k)} \leftarrow I$
   \ENDFOR
   \WHILE{$sim > lastsim$}
   \STATE $lastsim \leftarrow sim$
   \STATE $sim \leftarrow 0$
   \STATE Randomize order of permutations $order \leftarrow (x_1, x_2, \dots, x_k)$.
   \FOR{all $k$ in $order$}
   \STATE $B^{(k)} \leftarrow \operatorname{PARAMS}(B, k)$
   \STATE $G \leftarrow A^{(k)} (B^{(k)})^T$
   \STATE $P \leftarrow \operatorname{LSA}(G)$
   \STATE $sim \leftarrow sim + \operatorname{trace}(G P^T)$
   \STATE $B \leftarrow \operatorname{PERMUTE}(B, P, k)$
   \STATE $P^{(k)} \leftarrow P P^{(k)} $
   \ENDFOR
   \ENDWHILE
   \STATE {\bfseries Return:} all $P^{(k)}$
\end{algorithmic}
\end{algorithm}

\paragraph{Activation matching.}
For activation matching, we compute the Gram matrix over the combined activations of all data used to train the network (test data is omitted).
Activations are drawn from the output of the last operation that occurs in a given permutation.
For example, if a permutation contains a convolution layer followed by a normalization and ReLU non-linearity, we use the output of the ReLU non-linearity as the activations for that permutation.

As in weight alignment, we use a linear kernel to minimize the $L^2$ norm (Euclidean distance) between activations.
However, since each intermediate activation is affected only by a single permutation, the problem is convex and we do not need to optimize layers in random order and over multiple iterations.
We therefore run a single iteration of Greedy-SOBLAP where the permutations are found in order of network depth.
This procedure is equivalent to the methods used prior to the development of Greedy-SOBLAP \cite{wang2019federated,singh2020model}.

\clearpage
\section{Comparison of definitions with prior work}
\label{app:definitions}

In \cref{sec:lc-mod-p}, we define several notions of linear connectivity modulo permutation, namely weak and strong linear connectivity modulo permutation and their simultaneous variants. 
The motivation behind these definitions was to formalize and differentiate among the mathematical structures that empirical and theoretical work have provided evidence for so far. 
What may surprise readers is that our notions are all distinct from the notion underlying the oft cited conjecture by 
Entezari et al.~\cite{entezari2021role}, which we reproduce here in full:

\begin{quote}\em
{\normalfont \textbf{Conjecture 1 (Entezari et al.~\cite{entezari2021role})}}
Let $f(\theta)$ be the function representing a feedforward network with parameters $\theta \in \mathbb{R}^k$, $\Perms$ be the set of all valid permutations for the network, $P : \mathbb{R}^k \times \Perms \to \mathbb{R}^k$ be the function that applies a given permutation to parameters and returns the permuted version, and $B(\cdot,\cdot)$ be the function that returns the barrier value between two solutions as defined in \cref{equation-barrier}. 
Then, there exists a width $h > 0$ such that for any network $f(\theta)$ of width at least $h$ the following holds: there exists a set of solutions $\cF \subseteq \mathbb{R}^k$ and a function $Q : \cF \to \Perms$ such that for any $\theta_1, \theta_2 \in \cF$, $B(P(\theta_1, Q(\theta_1)), \theta_2) \approx 0$, and with high probability over an SGD solution $\theta$, we have $\theta \in \cF$.
\end{quote}

In order to distinguish the different notions of linear connectivity modulo permutation, it is useful to extract the exact property relied upon in the the conjecture above. 

\begin{definition}%
\label{def:app-conjecture-entezari}
Let $\cF\subseteq \mathbb{R}^k$  be a class of networks.
Say that $\cF$ \emph{has the Entezari property} if there exists a map $Q : \cF \to \Perms$ such that, for all $\theta_1, \theta_2 \in \cF$, $B(P(\theta_1, Q(\theta_1)), \theta_2) \approx 0$.
\end{definition}

In terms of this definition, the conjecture of Entezari et al.\ can be stated concisely: \emph{for sufficient wide networks, there is a high probability set of SGD solutions that has the Entezari property.}

Note that according to \cref{def:app-conjecture-entezari}, a class $\cF$ has the Entezari property if, for all $\theta_1 \in \cF$, there is a single permutation $Q(\theta_1)$ that removes the loss barrier between $\theta_1$ and \textbf{all} $\theta_2 \in \cF$. Notably, this permutation does not depend on the second network $\theta_2$. 

To facilitate comparison, our basic definitions from \cref{sec:lc-mod-p} can be expressed in the same language:

\begin{definition}[\WLCmP
]
\label{def:app-weak-lmcmp}
Let $\cF\subseteq \mathbb{R}^k$  be a set of networks.  We say that weak linear connectivity modulo permutation holds on $\cF$ if there exists $Q : \cF \times \cF \to \Perms$ such that for all $\theta_1, \theta_2 \in \cF$, $B(P(\theta_1, Q(\theta_1, \theta_2)), \theta_2) \approx 0$.
\end{definition}

\begin{definition}[\SLCmP]
\label{def:app-strong-lmcmp}
Let $\cF\subseteq \mathbb{R}^k$  be a set of networks.  We say that strong linear connectivity modulo permutation holds on $\cF$ if there exists $Q : \cF \to \mathcal{S}$ such that for all $\theta_i,\theta_j \in \cF$, we have $B(P(\theta_i, Q(\theta_i)), P(\theta_j, Q(\theta_j))) \approx 0$.
\end{definition}

In \cref{prop:separation}, we prove that, on arbitrary classes, the Entezari property implies \SLCmP. In the other direction, we show that there are classes for which they are distinct notions.
We conjecture they are not equivalent notions also on high probability subsets of SGD solutions. 
In particular, a curious consequence of the Entezari property on a class of networks is that it implies the networks are already piece-wise linearly connected \textbf{before} permutation, which does
not appear to be a consequence of \WLCmP or \SLCmP. (See \cref{cor:piecewise-lc}.)

There is no direct empirical or theoretical evidence for \cref{def:app-conjecture-entezari} at present.
It is worth noting that Entezari et al.~\cite{entezari2021role} provide an informal statement of their conjecture before the formal one provided by \cref{def:app-conjecture-entezari}. This informal statement is:
\begin{quote}\em
Most SGD solutions belong to a set $\cF$ whose elements can be permuted in such a way that
there is no barrier on the linear interpolation between any two permuted elements in $\cF$.
\end{quote}
Depending on how one interprets the order of quantification, we believe that this statement can be naturally interpreted as either \SLCmP or \WLCmP. 
On the other hand, it seems difficult to find an interpretation leading to \cref{def:app-conjecture-entezari}. 

As further evidence that our notions of \WLCmP and \SLCmP better capture the intended structure, we note that, after their informal conjecture statement, Entezari et al.\ refer to their own Figure 1, which we would interpret as depicting \SLCmP. On the other hand, the main theorem presented by Entezari et al.\ can be seen to actually establish \WLCmP \cite[Theorem 3.1]{entezari2021role}.\footnote{Specifically, the theorem establishes linear connectivity modulo permutation at initialization between fully-connected, SGD-trained networks with a single hidden layer. The order of quantification ``for all $\theta_1=(U,v)$, for all $\theta_1=(U',v')$ there exists a permutation [\dots]'' in the theorem indicates that the permutation may depend on both networks \cite{entezari2021role}. Further inspection of the proof makes it is clear that this is the case.} Note that one can convert \cref{def:app-conjecture-entezari} to \WLCmP by modifying the function $Q$ in \cref{def:app-conjecture-entezari} to depend on both $\theta_1$ and $\theta_2$. 

Other work has also give empirical evidence for \cref{def:app-weak-lmcmp} on SGD trained networks \cite{ainsworth2022git,benzing2022random,pena2023re}.
However, to our knowledge, no formal statement of \SLCmP (\cref{def:app-strong-lmcmp}) exists in the literature, although there are closely related illustrations \cite[figure 1 left and middle]{entezari2021role} and proxy models of linear mode connectivity modulo permutation \cite[Section 3.5]{entezari2021role}.

Here, we should note that linear mode connectivity is sometimes assumed to imply convexity in the loss landscape, which is untrue, both in general \cite{vlaar2022can} and in our case.
In particular, our definitions provide necessary but not sufficient conditions for convexity modulo permutation.
\WLCmP (\cref{def:app-weak-lmcmp}) is obviously not convex for sets containing at least 3 networks, since a given permutation $Q(\theta_1, \theta_2)$ is not guaranteed to make more than 2 networks linearly connected.
In the case of \SLCmP, \cref{def:app-strong-lmcmp} only requires linear connectivity between networks from a particular set, but the line segments connecting those networks may themselves be excluded from that set (and thus, points from different lines may not be linearly connected). Therefore, a set satisfying \SLCmP is not necessarily convex in the loss landscape, and additional requirements are required for convexity to hold, such as the set being dense and convex in the parameter space.

\subsection{Proofs of implication between definitions}
\label{app:proof-entezari-equivalence}

What is the structure of classes with the Entezari property? The following result highlights that the property has very profound consequences:

\begin{corollary}[Piecewise linear connectivity]
\label{cor:piecewise-lc}
    Let $\cF \subseteq \mathbb{R}^k$ be a subset of neural networks that satisfies \cref{def:app-conjecture-entezari}. Then, 
    for all $\theta_1, \theta_2 \in \cF$, there exists $\theta_3 \in \cF$ such that
    $B(\theta_1, \theta_3)\approx 0$ and $B(\theta_2, \theta_3)\approx 0$. 
\end{corollary}
\begin{proof}
    Take $P$ from \cref{def:app-conjecture-entezari}.
    And take $\theta_3 = P(\theta_1, Q(\theta_1))$. 
    Then, $B(P(\theta_1, Q(\theta_1)), \theta_2)\approx 0$ and $B(P(\theta_1, Q(\theta_1)), \theta_1)\approx 0$.
\end{proof}

The following result teases apart \WLCmP, \SLCmP, and the Entezari property. It will be useful to consider a class $\cF_0$ consisting of two networks that are permutations of one another but that are not linearly connected. (In standard neural network architectures and training regimens, such classes are easily constructed by running SGD to convergence and permuting the solution uniformly at random.)

\begin{proposition}\label{prop:separation}
For all classes $\cF \subseteq \mathbb{R}^k$, 
if $\cF$ has the Entezari property, then 
\SLCmP holds on $\cF$ (and thus \WLCmP holds also).
In the other direction, \SLCmP holds on $\cF_0$ but $\cF_0$ does not have the Entezari property.
\end{proposition}

\begin{proof} Suppose \cref{def:app-conjecture-entezari} holds. By assumption, $P(\theta_i, Q(\theta_i)) \in \cF$. From \cref{def:app-conjecture-entezari} we can deduce $B(P(\theta_i, Q(\theta_i)), P(\theta_j, Q(\theta_j)))\approx 0$ for all $\theta_i,\theta_j \in \cF$, thus implying \SLCmP, which implies \WLCmP.

In the other direction, let $\cF_0 = \{\theta_1, \theta_2\}$ such that $B(\theta_1, \theta_2) > 0$ and take $P(\theta_1, Q(\theta_1)) = \theta_1$ and $P(\theta_2, Q(\theta_2)) = \theta_1$. It is straightforward to verify that the map $Q$ witnesses \SLCmP. On the other hand, since $\theta_1$ and $\theta_2$ are not linearly connected, $P(\theta_2, Q(\theta_2))$ is disconnected from either $\theta_1$ or $\theta_2$ and so the Entezari property does not hold.
\end{proof}

It is an interesting open problem to separate \SLCmP and \WLCmP from the Entezari property on richer classes that $\cF_0$. We conjecture that they are distinct on high probability classes of SGD solutions, although it is straightforward to see that all three properties hold of a class for which the loss surface, restricted to the class, is a convex function, and so we will not be able to distinguish these notions in the ``NTK limit''. 

\clearpage
\section{Linear mode connectivity modulo permutation as a transitive relation}
\label{app:transitivity}

\label{sec:transitivity}

Our results speak to a broader property of transitivity between linearly connected networks modulo permutation.
If such a property were to exist, we might expect that if network $A$ is connected to $B$ and $B$ is connected to $C$, then $A$ should be connected to $C$.
It will be useful to express weak linear connectivity modulo permutation (\cref{def:weak-lcmodp}) as a binary relation. To that end,
for any pair of networks $A$ and $B$,
write $A \sim_p B$ if there exists some permutation of $B$ such that $P[B]$ is linearly connected to $A$.
Clearly $\sim_p$ is reflexive (i.e., $A \sim_p A$ for all $A$) and symmetric.
Thus, for $\sim_p$ to qualify as an equivalence relation we require it to satisfy some notion of transitivity.

\begin{definition}[Weak transitivity]
    Let $\mathcal C$ be a set of networks. We say that \emph{weak transitivity holds on $\mathcal{C}$} if, for all $A,B \in \mathcal C$, there exists a permutation $P \in \mathcal P$ such that $A$ and $\permuteop{P}{B}$ are linearly connected.
\end{definition}
If $\sim_p$ is transitive on a class $\mathcal C$, it is an equivalence relation on the class $\mathcal C$. 

Note that, on the surface, transitivity does \emph{not} imply that more than two networks can be simultaneously aligned and made to be linearly connected. For example, if $\permuteop{P}{A}$ is linearly connected to $C$ and $\permuteop{Q}{B}$ is linearly connected to $C$, it is not necessarily the case that $\permuteop{P}{A}$ and $\permuteop{Q}{B}$ are linearly connected to one another. 
Transitivity (and \cref{def:weak-lcmodp}) merely implies that \emph{some} permutation exists that aligns $P[A]$ with $Q[B]$, which does not have to be the identity permutation.

Analogous to strong linear connectivity modulo permutation (\cref{def:strong-lcmodp}), we can define a strong transitivity relation.

\begin{definition}[Strong transitivity]
\label{definition:strong-transitivity}
Let $C$ be a set of networks. We say that
 \emph{$k$-wise strong transitivity holds on $C$} if (i) weak transitivity holds on $C$ and (ii) for any set of networks $A_1,\dots,A_k \in C$ that are equivalent under $\sim_p$,
there exist permutations 
$P_1,P_2,\ldots,P_k$ such that
$\permuteop{P_i}{A_i}$ is linearly connected to $\permuteop{P_j}{A_j}$ for every $i,j$. 
\end{definition}

Intuitively, strong transitivity means networks equivalent under $\sim_p$ can be simultaneously aligned into a mutually linearly connected family of networks. 
If strong transitivity were to hold, there would be several pleasant consequences. It would imply that a very large number of networks can be linearly connected given a suitable permutation.
It would also confirm the notion of a ``linearly connected mode,'' a term that has already been used extensively in the literature but which actually does not make sense without transitivity.
Namely, if linearly connected modes exist, then the property of ``lying in the same mode'' should be an equivalence relation, and hence transitive.
\clearpage
\section{Experimental details}
\label{app:expdetails}

\paragraph{Models and datasets.}
In all experiments, we use convolutional neural networks and residual neural networks trained on CIFAR-10 \cite{Krizhevsky2009}.

For convolutional models, we train standard VGG-16 convolutional networks \TBD{cite} with a standard width of 64 hidden units after the input layer.
For residual models, we train ResNet-20 with a width of 64 hidden units ($4\times$ width), as \cite{ainsworth2022git} finds that standard width ResNet-20 models have much larger error barrier than equivalent VGG networks.

We replace batch normalization with layer normalization \cite{ba2016layer} for comparability with \cite{ainsworth2022git}, and also to avoid the variance reduction issue raised by \cite{jordan2022repair}.

Networks for trajectory matching and sparse network matching experiments are trained with the SGD optimizer with weight decay, applying a $0.1$ learning rate after a 5 epoch warmup schedule. After 5 epochs of training we apply the cosine annealing learning rate decay schedule. 

Networks studying instability of SGD  (\cref{sec:lmc}) are trained using SGD optimizer without weight decay, applying a $0.1$ learning rate after only 1 epoch of warm-up schedule. This allows us to obtain networks that are unstable at initialization. 

Remaining networks are trained with the SGD optimizer, applying a $0.1$ learning rate after a 1 epoch warmup schedule. Full training and pruning details are found in \cite{frankle2020linear}.

To estimate loss and error barrier, we test 25 equally spaced interpolation values $\alpha$ between $0$ and $1$ (inclusive) on 10k training/test examples.
All results are averaged over 5 replicates.

\paragraph{Software.} 
We used \hyperlink{https://github.com/facebookresearch/open_lth}{OpenLTH} and  \hyperlink{https://github.com/samuela/git-re-basin/}{git-re-basin} software packages for our experiments \cite{ainsworth2022git,frankle2018lottery}. GNU Parallel \cite{Tange2011a} is used to manage experiment tasks.

\clearpage
\section{Comprehensive training trajectory results}
\label{app:all-trajectory-results}

Here we provide additional results showing linear mode connectivity modulo permutation for training trajectories for a variety of models, datasets, and permutations.

\paragraph{VGG-16 train and test loss or error trajectories.}
\cref{fig:app-trajectoryalignment-vgg-activation} provides a more detailed view of the trajectories of individual permutations computed with activation matching from various times during training. 

\begin{figure*}[ht]
\begin{center}
\includegraphics[width=.24\textwidth]{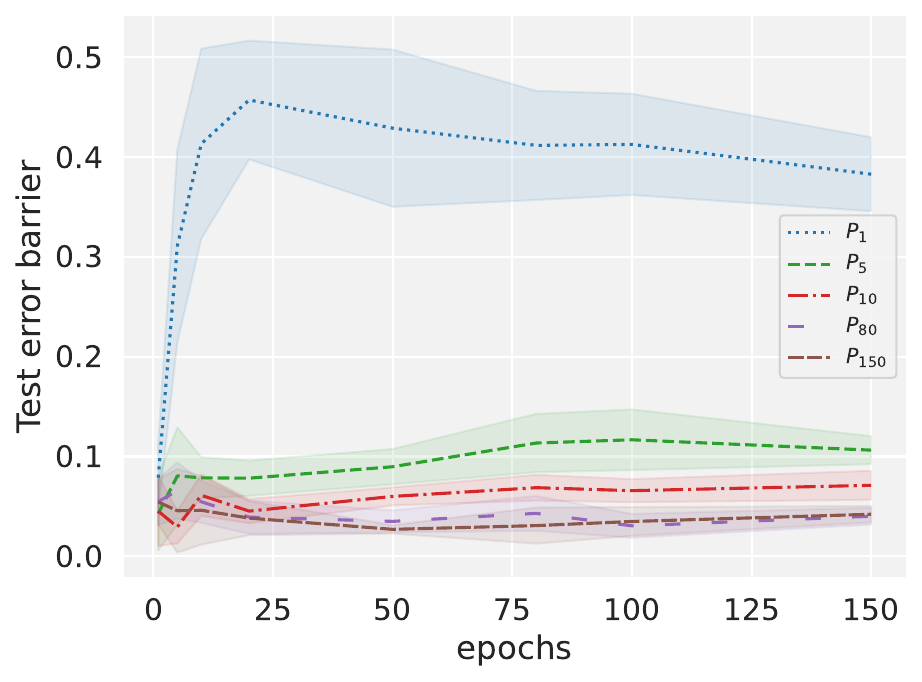}
\includegraphics[width=.24\textwidth]{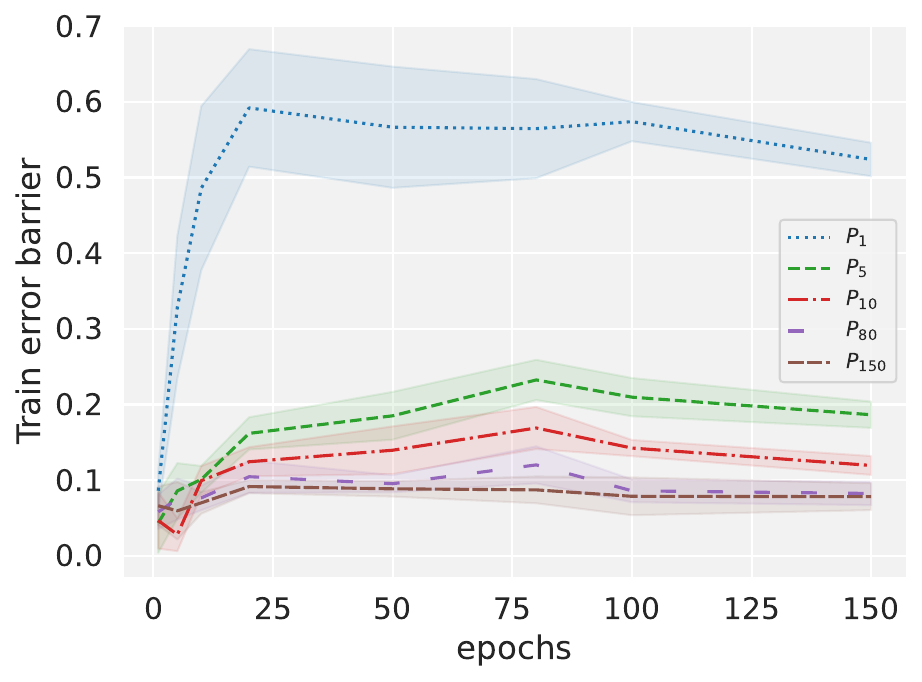}
\includegraphics[width=.24\textwidth]{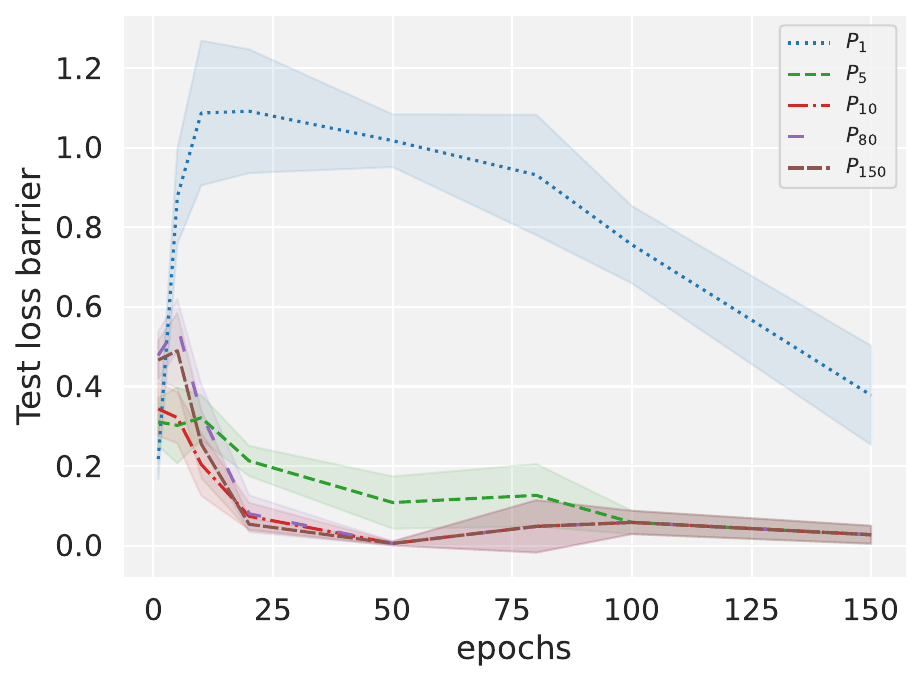}
\includegraphics[width=.24\textwidth]{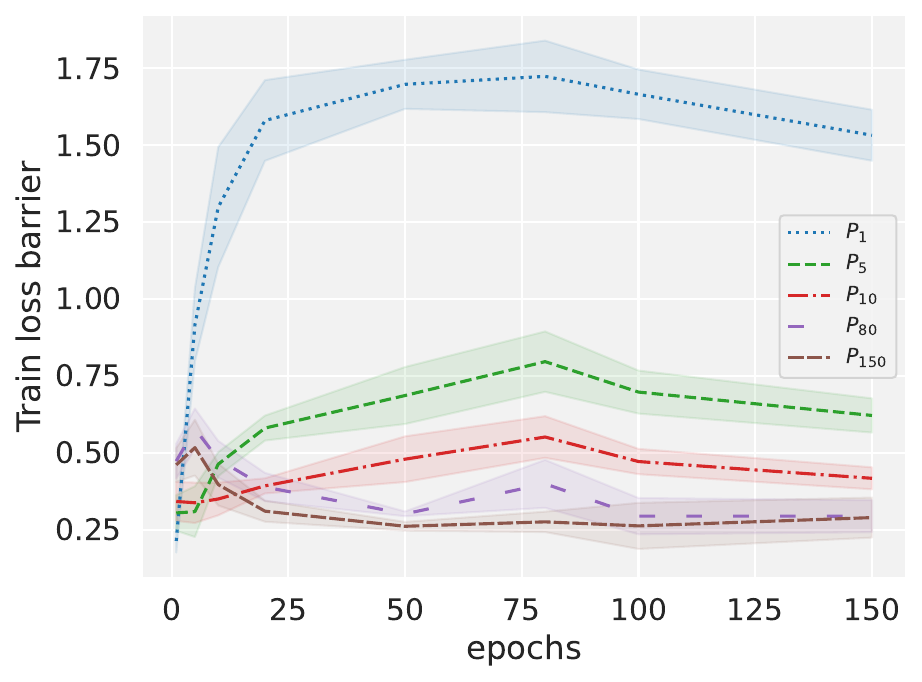}
\includegraphics[width=.24\textwidth]{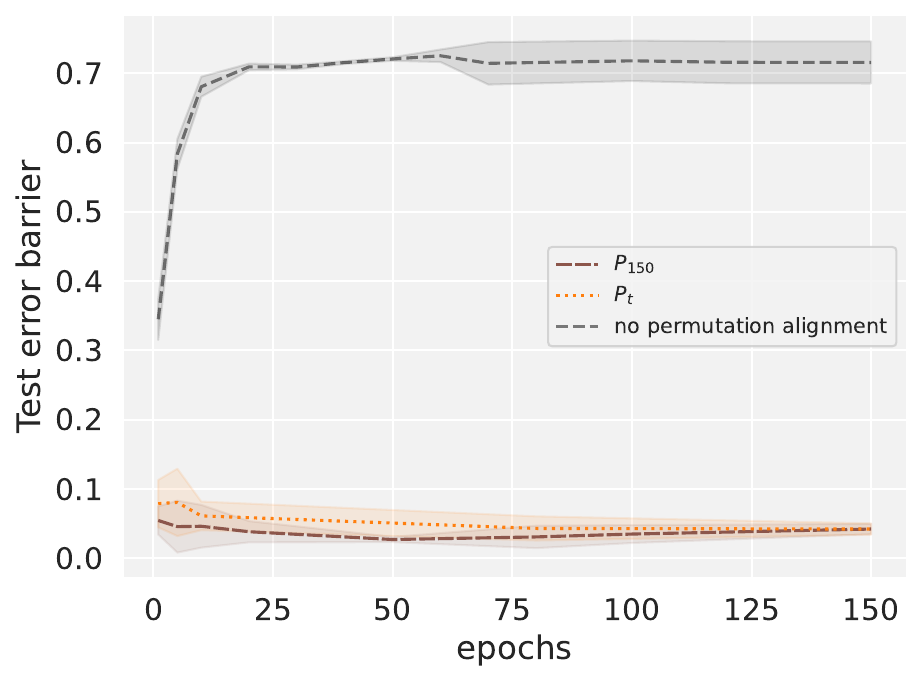}
\includegraphics[width=.24\textwidth]{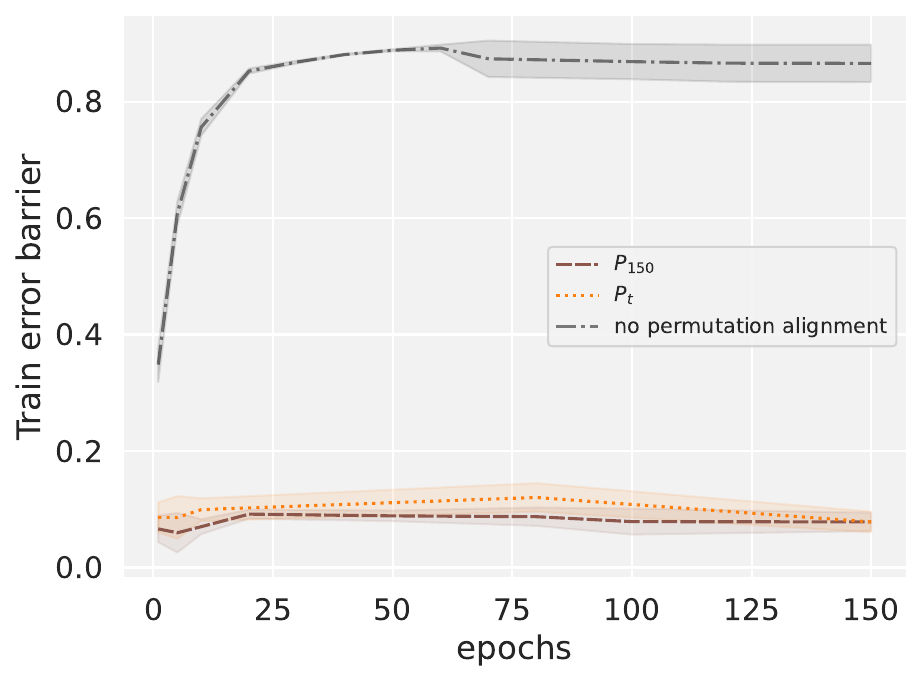}
\includegraphics[width=.24\textwidth]{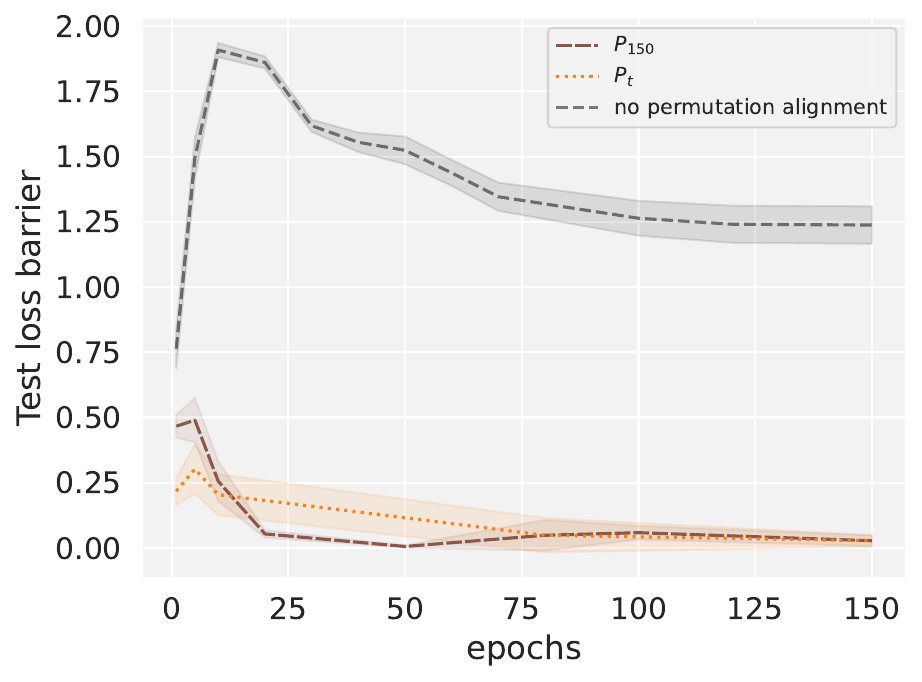}
\includegraphics[width=.24\textwidth]{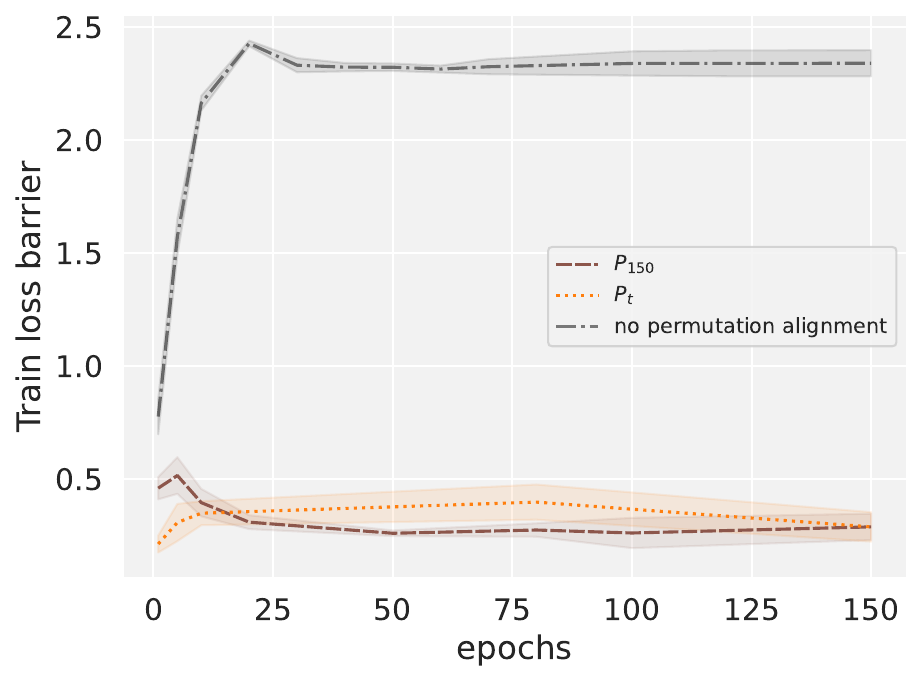}
\caption{ 
 (\textbf{Activation matching}) evolution of the error and loss barrier between a pair of networks throughout training under a fixed permutation $P_t$ computed using activation alignment at epoch $t$, for VGG-16 trained on CIFAR-10.
}
\label{fig:app-trajectoryalignment-vgg-activation}
\end{center}
\end{figure*}

\cref{fig:trajectoryalignmentall} plots weight matching trajectories on both train and test data for cross-entropy loss (used in training) and 0--1 loss (accuracy, used for model evaluation) functions, as well as trajectories for child networks trained from the same initial checkpoint with different minibatch orders. In every case, we are able to find a single permutation (typically computed at the end of training) that leads to simultaneous linear mode connectivity of the entire training trajectory.
Comparing to \cref{fig:app-trajectoryalignment-vgg-activation}, notice that activation matching algorithm allows one can find such a permutation earlier in training.

\begin{figure}[H]
\begin{center}
\includegraphics[height=3.5cm]{figures/test_error_trajectory_alignment}
\includegraphics[height=3.5cm]{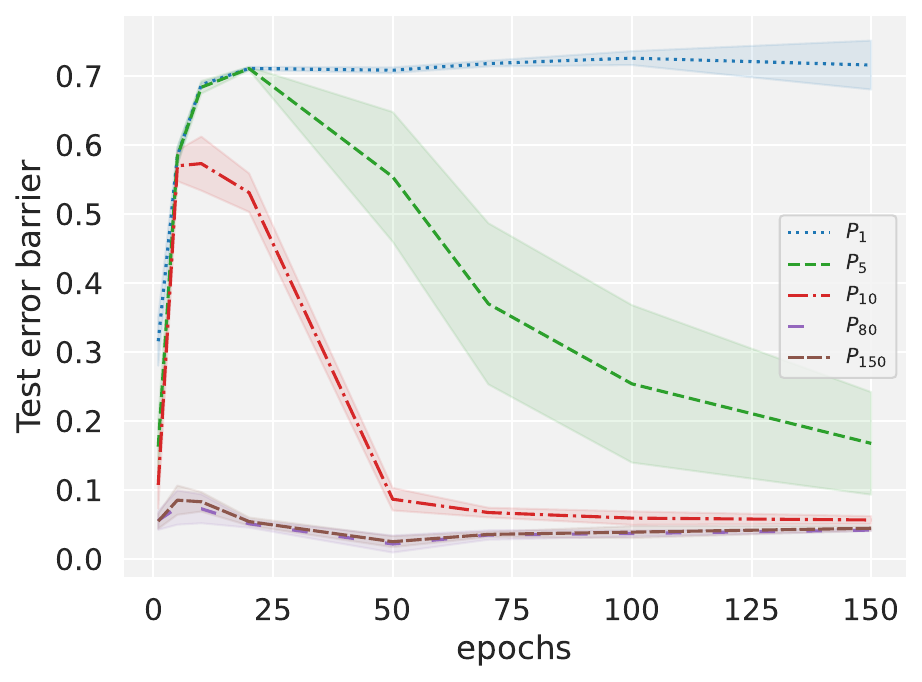}
\includegraphics[height=3.5cm]{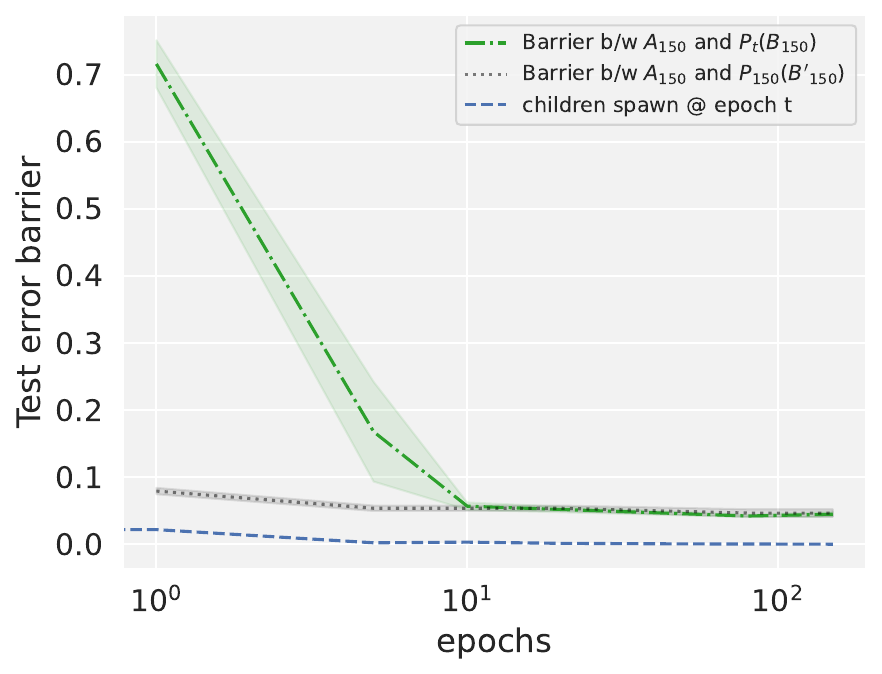}
\\
\includegraphics[height=3.5cm]{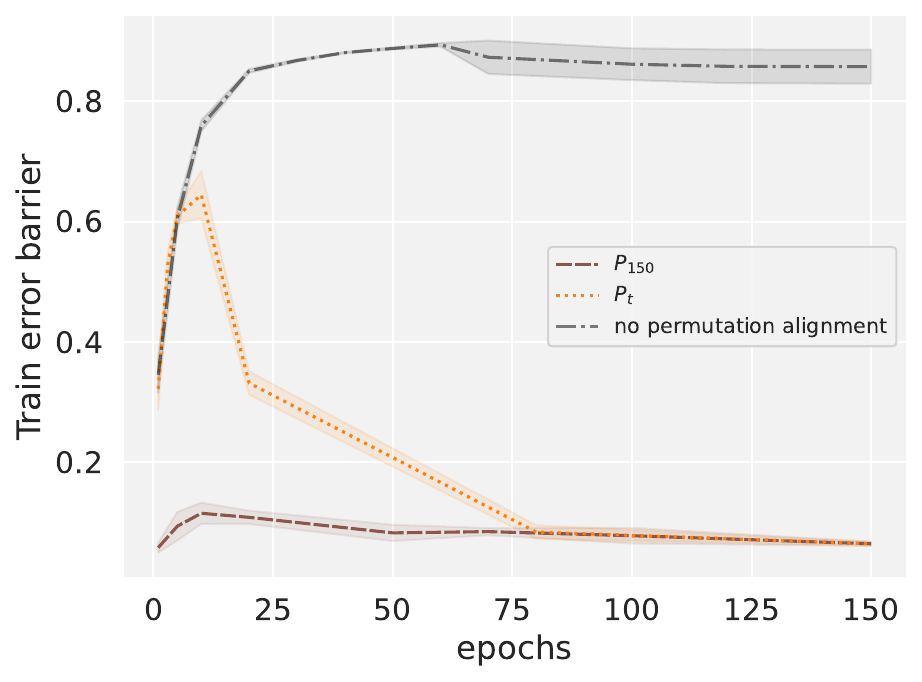}
\includegraphics[height=3.5cm]{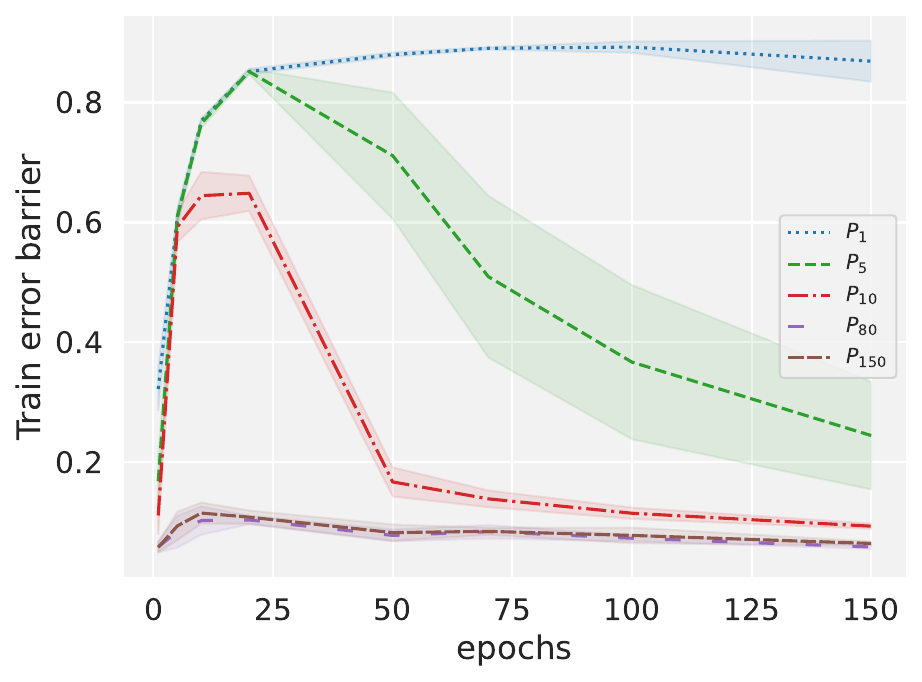}
\includegraphics[height=3.5cm]{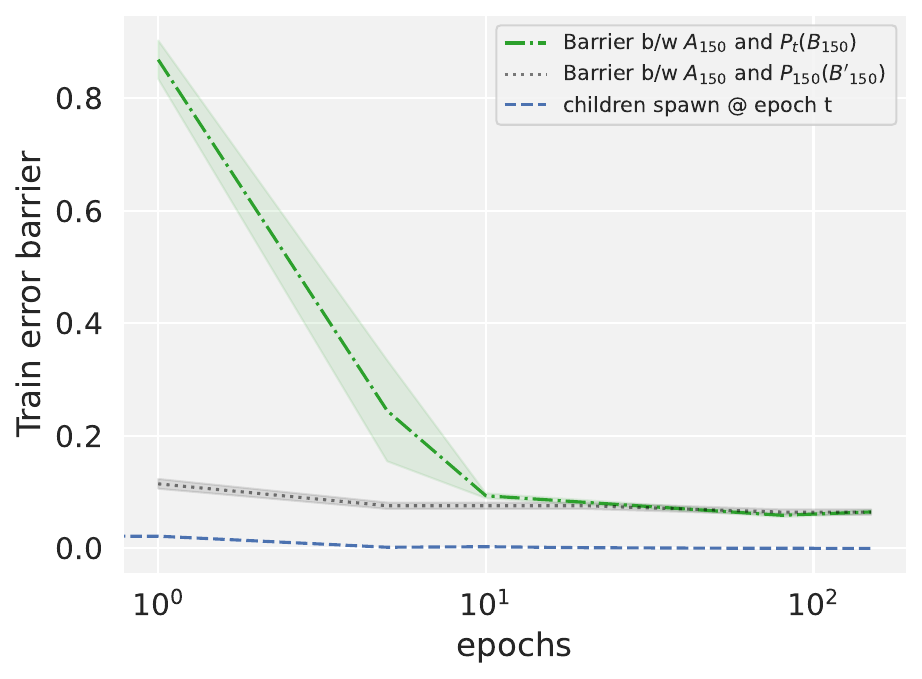}
\\
\includegraphics[height=3.5cm]{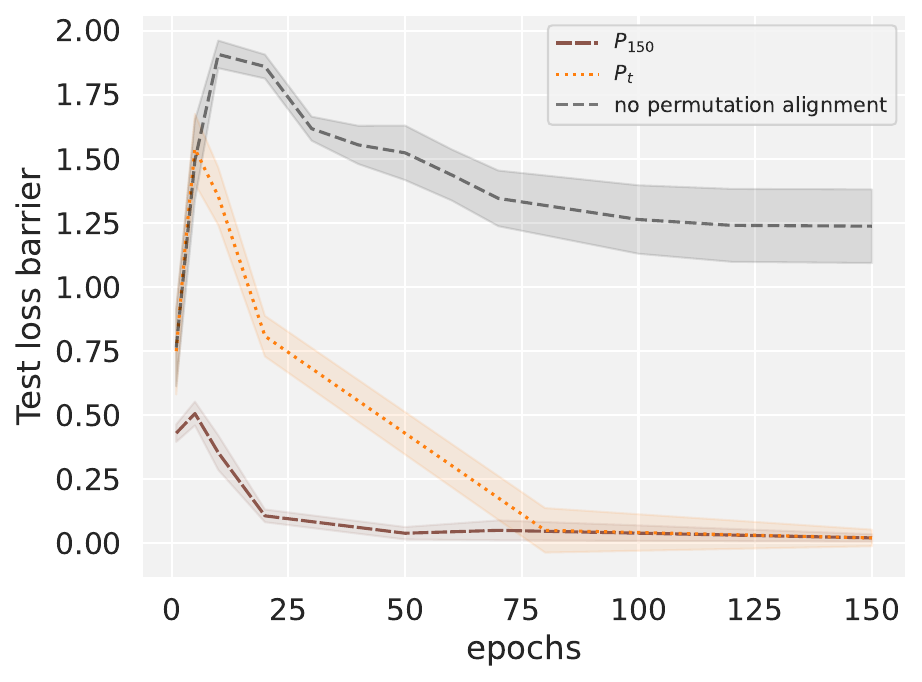}
\includegraphics[height=3.5cm]{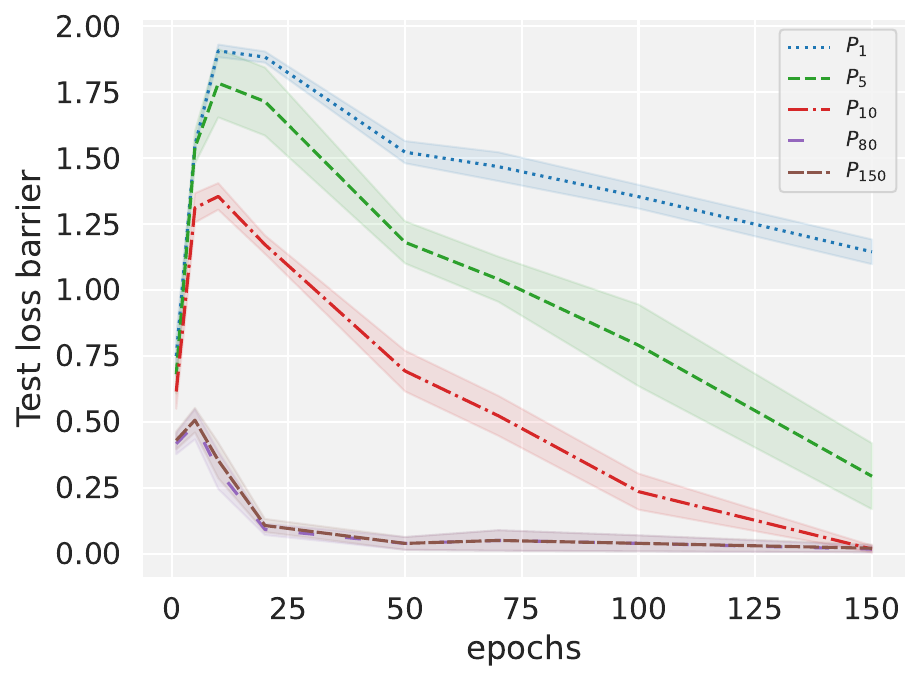}
\includegraphics[height=3.5cm]{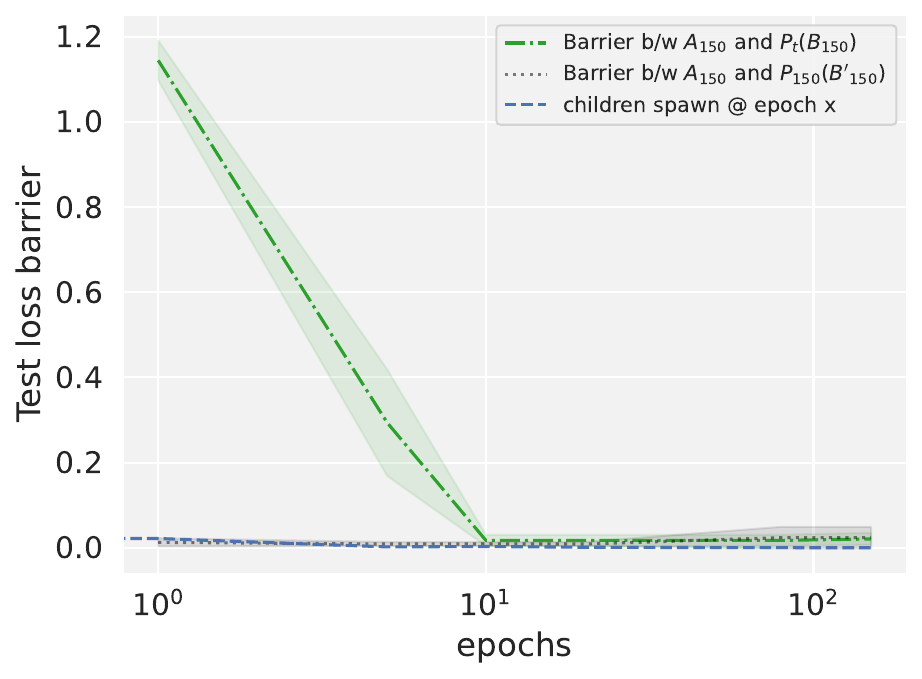}

\includegraphics[height=3.5cm]{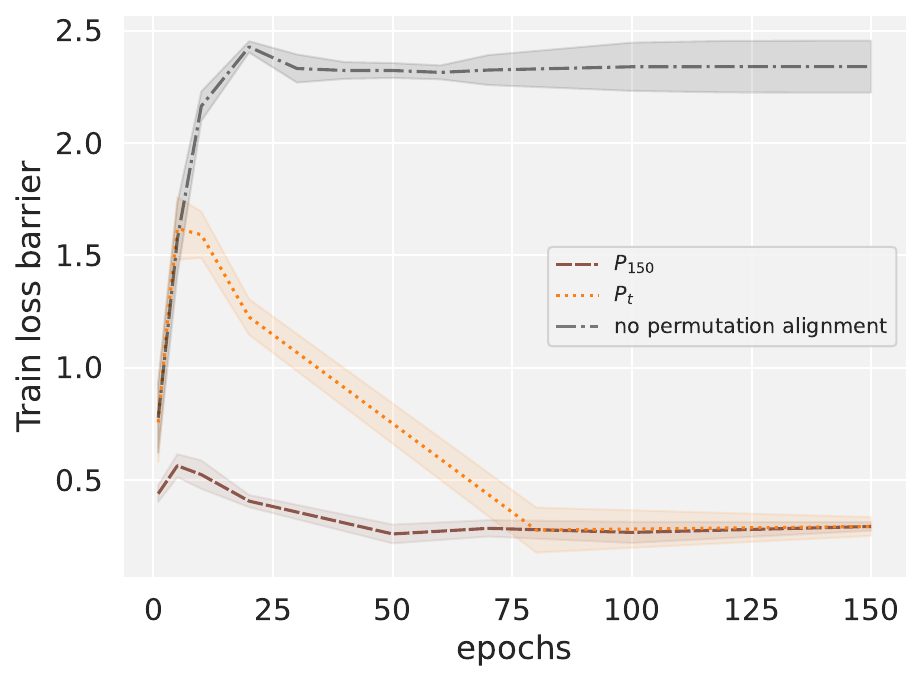}
\includegraphics[height=3.5cm]{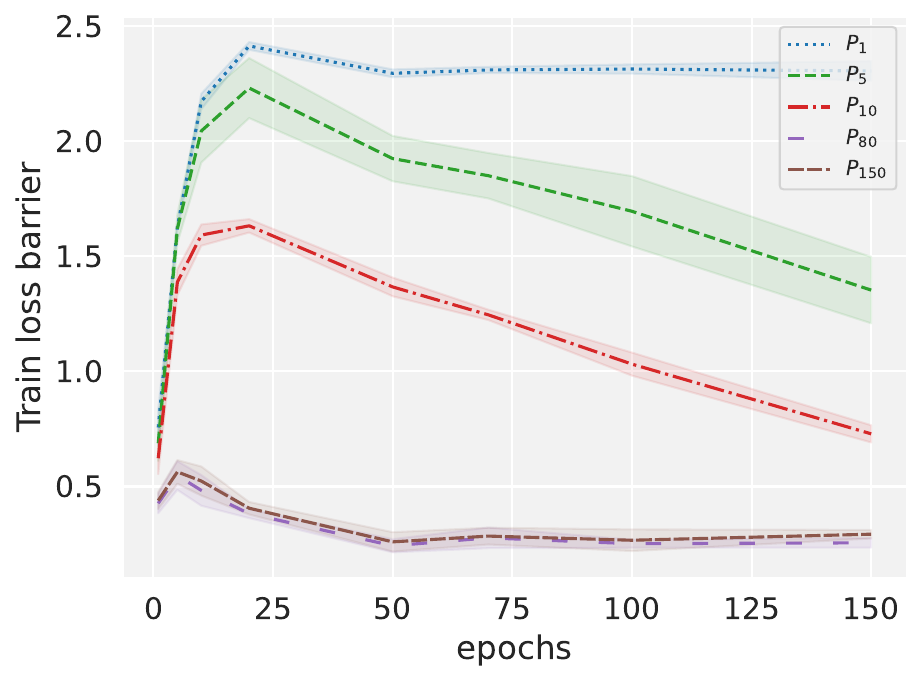}
\includegraphics[height=3.5cm]{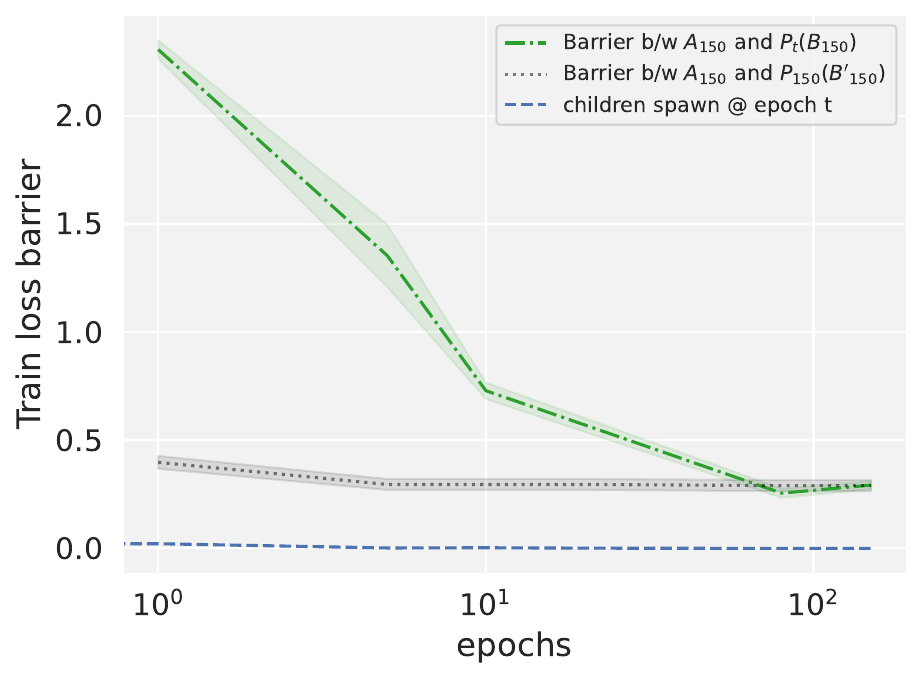}
\caption{ (\textbf{VGG16 architecture}) evolution of loss or error barriers (y-axis) over training trajectories.
\textbf{(Column 1)} error barrier between networks $A_t$ and permuted $B_t$ at different training checkpoints $t$ (x-axis). The blue line corresponds to applying a permutation $P_{\mathrm{end}}$ learned at the end of training. The orange line corresponds to a permutation $P_t$ computed and applied at time $t$.
\textbf{(Column 2)} loss barrier between a pair of networks throughout training under a fixed permutation $P_t$ computed at epoch $t$ (different lines correspond to different values of $t$ as indicated in the legend). The pair of networks are trained from different initialization and with different minibatch orders.
\textbf{(Column 3)} error barrier at the end of training for a pair of networks after applying a permutation $P_t$ (dot-dashed green line). The time $t$ when the permutation is computed is indicated on the $x$-axis (log scale).
This error barrier is compared against error barrier between  A) ``child'' networks spawned from a single ``parent'' network at time $t$ (dashed blue line); B) model A and ``child'' network of model B but after applying the permutation learned model A and ``parent'' model B (dotted grey line). Each child network is trained with a different minibatch order starting from the parent weights at time $t$.
}
\label{fig:trajectoryalignmentall}
\end{center}
\end{figure}

\clearpage
\paragraph{ResNet-20 train and test loss or error barrier trajectories.} Next we examine the evolution of loss and error barrier for ResNet20 model architecture. Our results show that we have simultaneous linear mode connectivity for ResNet20 model architecture. 

\begin{figure*}[ht]
\begin{center}
\includegraphics[height=3.5cm]{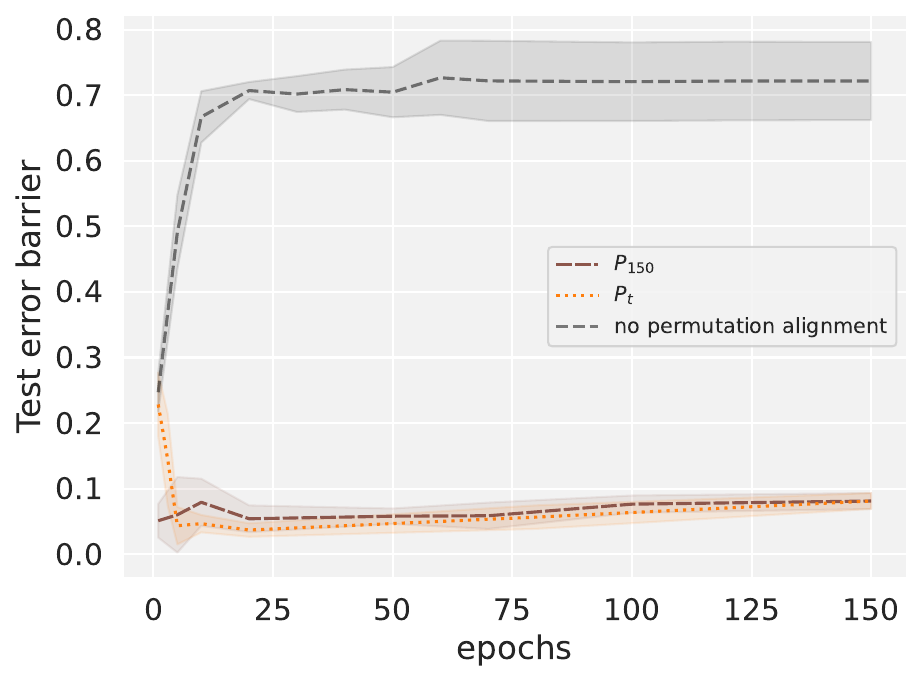}
\includegraphics[height=3.5cm]{figures/resnet20_test_error_barrier_evolution}
\includegraphics[height=3.5cm]{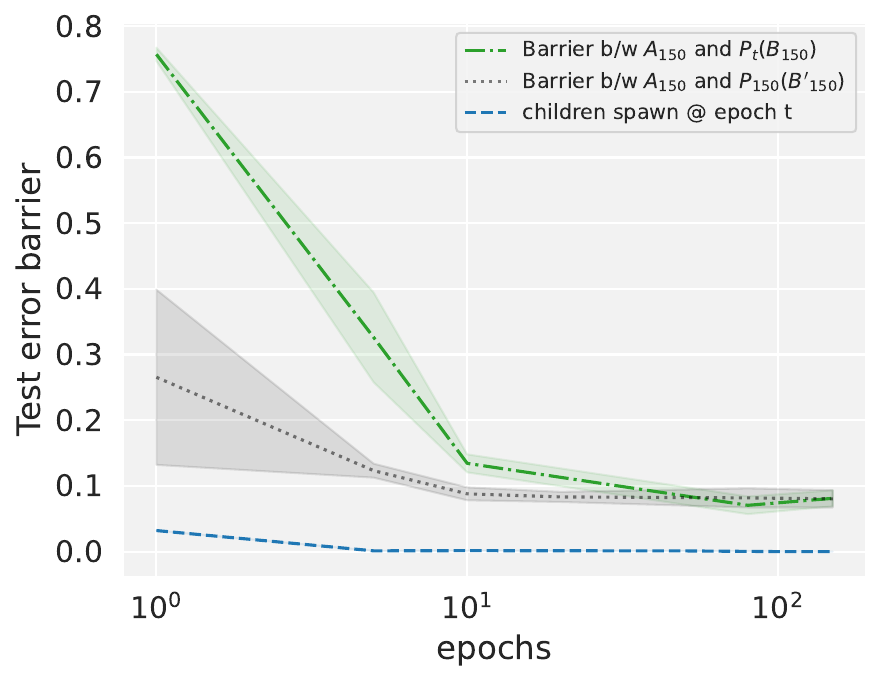}
\\
\includegraphics[height=3.5cm]{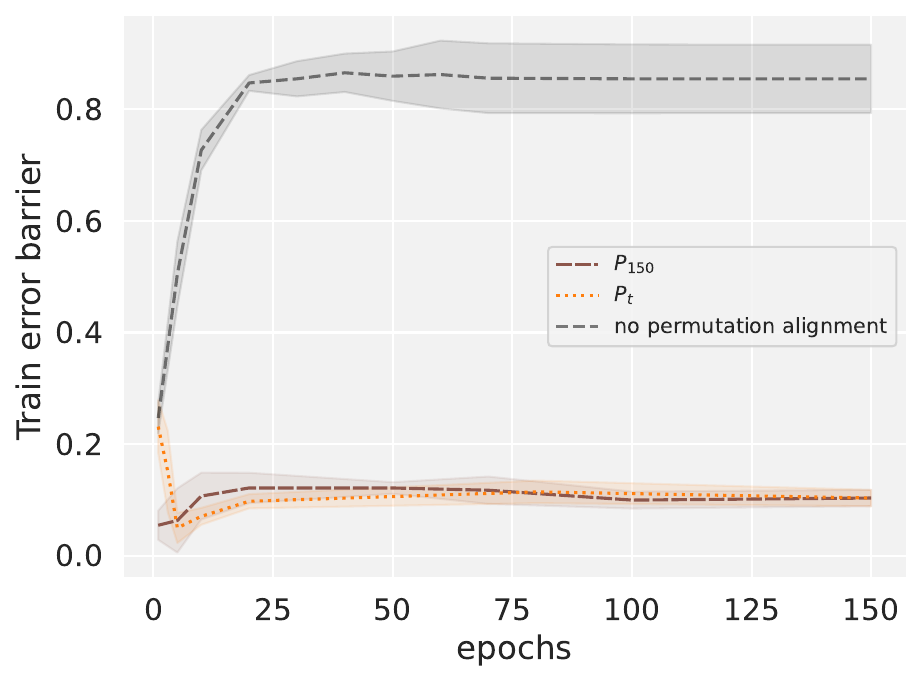}
\includegraphics[height=3.5cm]{figures/resnet20_train_error_barrier_evolution}
\includegraphics[height=3.5cm]{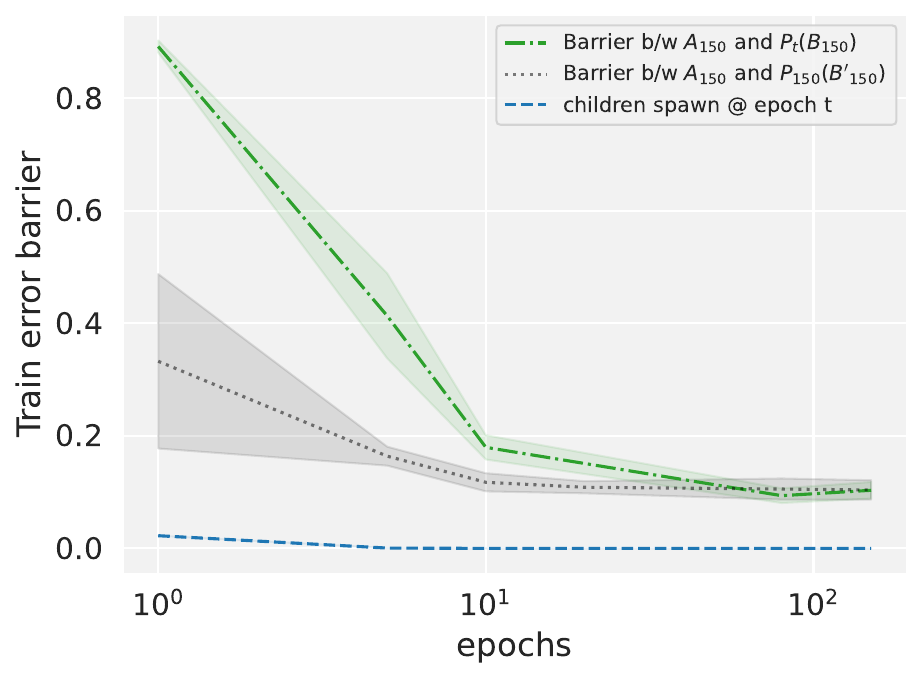}
\\
\includegraphics[height=3.5cm]{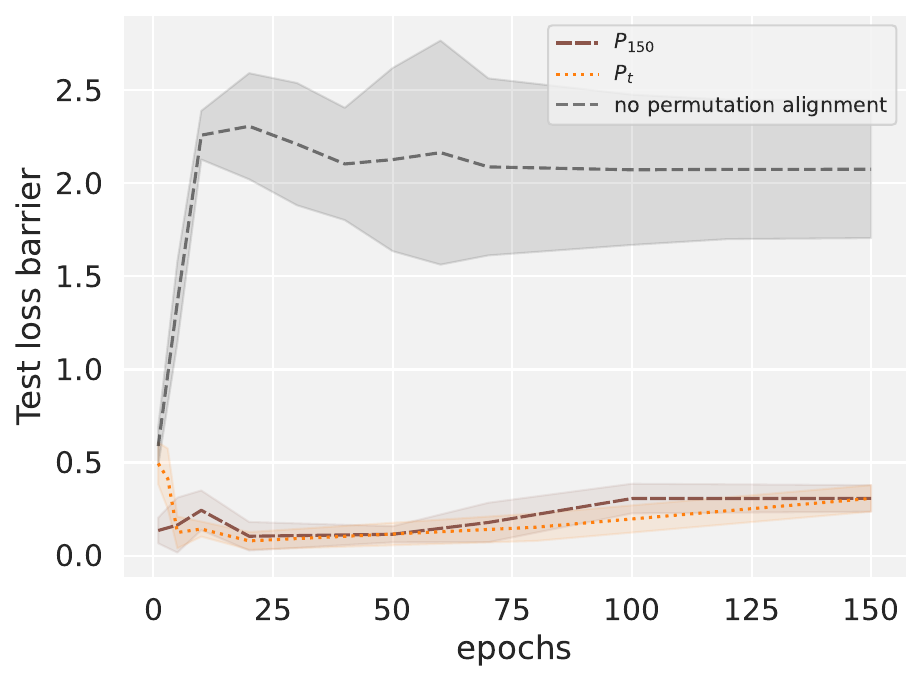}
\includegraphics[height=3.5cm]{figures/resnet20_test_loss_barrier_evolution}
\includegraphics[height=3.5cm]{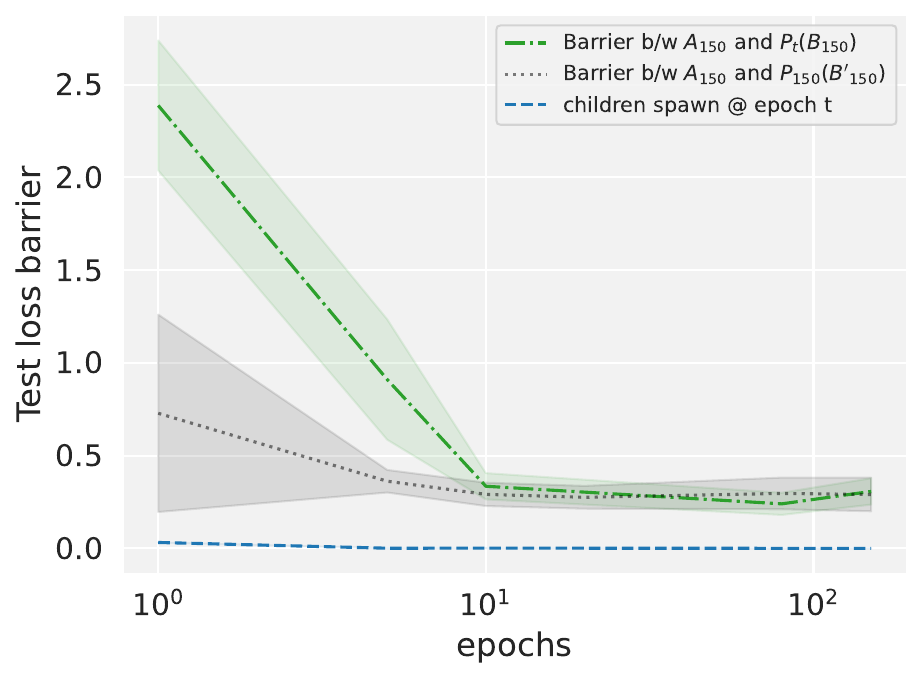}

\includegraphics[height=3.5cm]{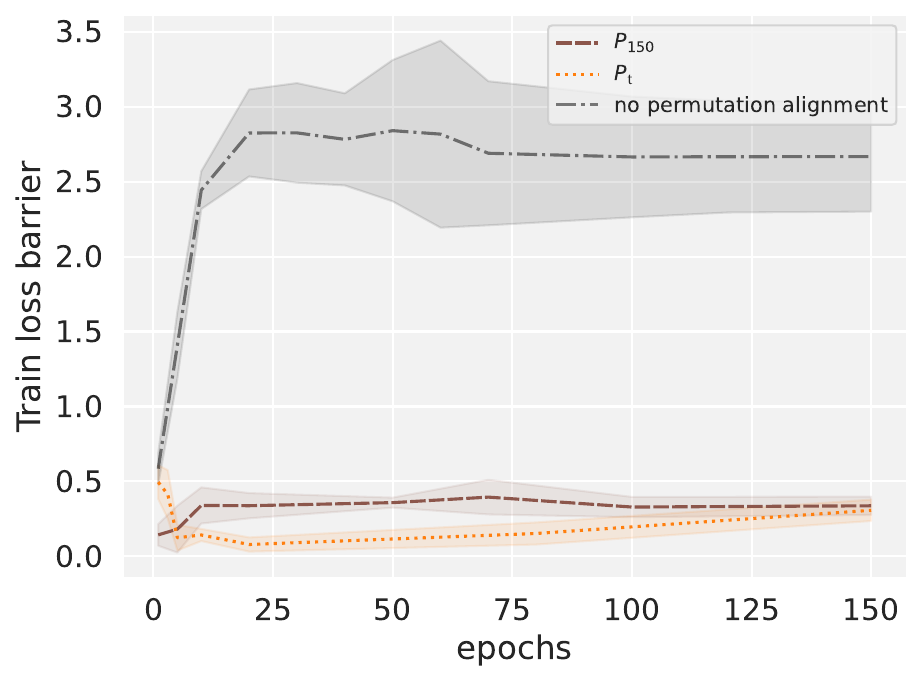}
\includegraphics[height=3.5cm]{figures/resnet20_train_loss_barrier_evolution}
\includegraphics[height=3.5cm]{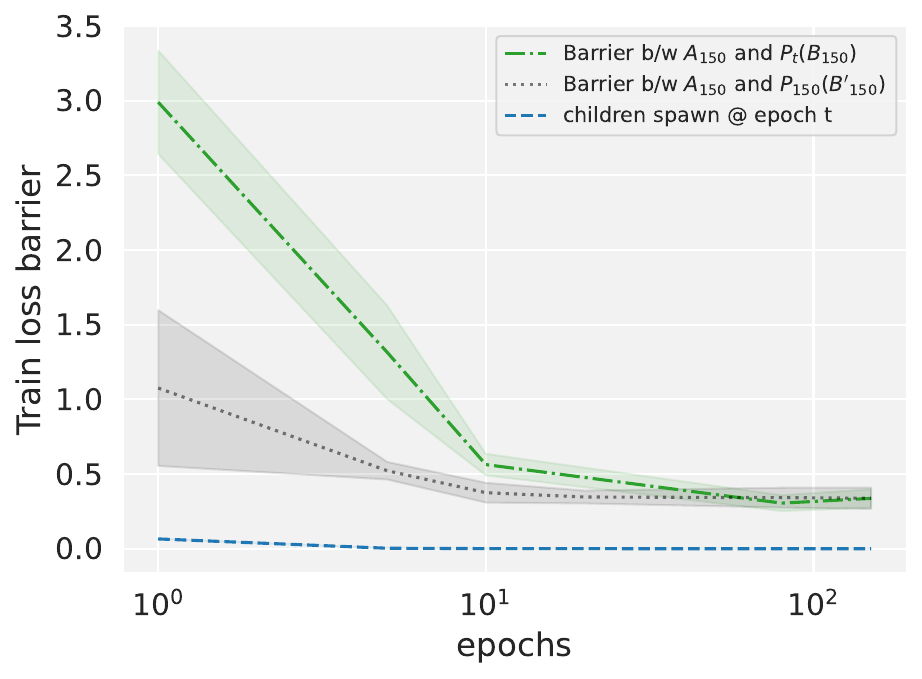}
\caption{(\textbf{ResNet-20 architecture}) Identical plots to \cref{fig:trajectoryalignmentall}, but for ResNet-20 networks.
}
\label{fig:trajectoryalignmentallresnet}
\end{center}
\end{figure*}

\paragraph{Trajectories for other datasets.} Next, we plot the evolution of loss barrier for a variety of image datasets and model architectures. The results show that our observations generalize across different datasets. 

\begin{figure*}[ht]
\begin{center}
\includegraphics[width=.24\textwidth]{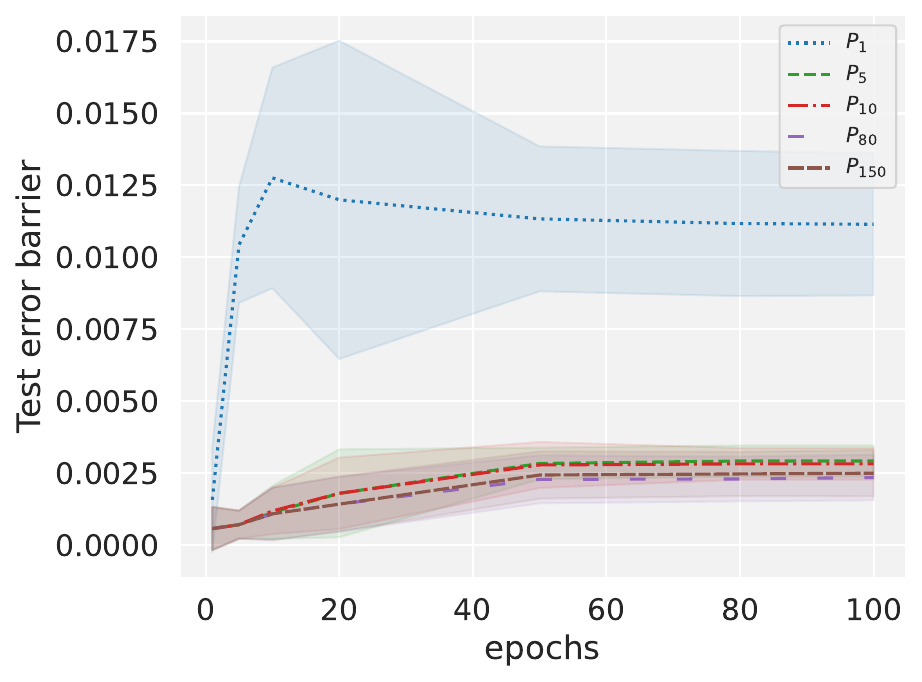}
\includegraphics[width=.24\textwidth]{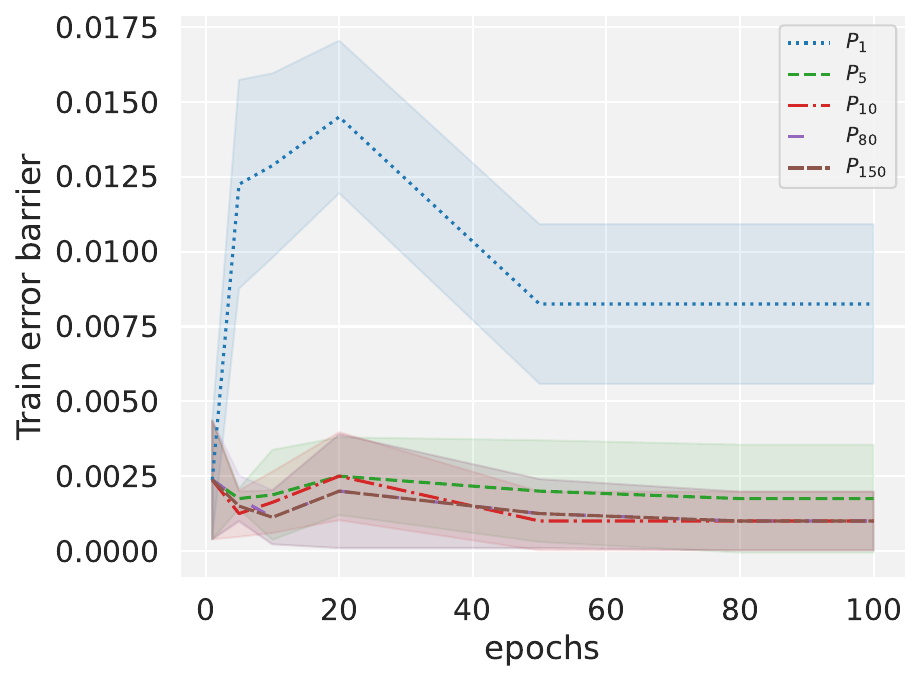}
\includegraphics[width=.24\textwidth]{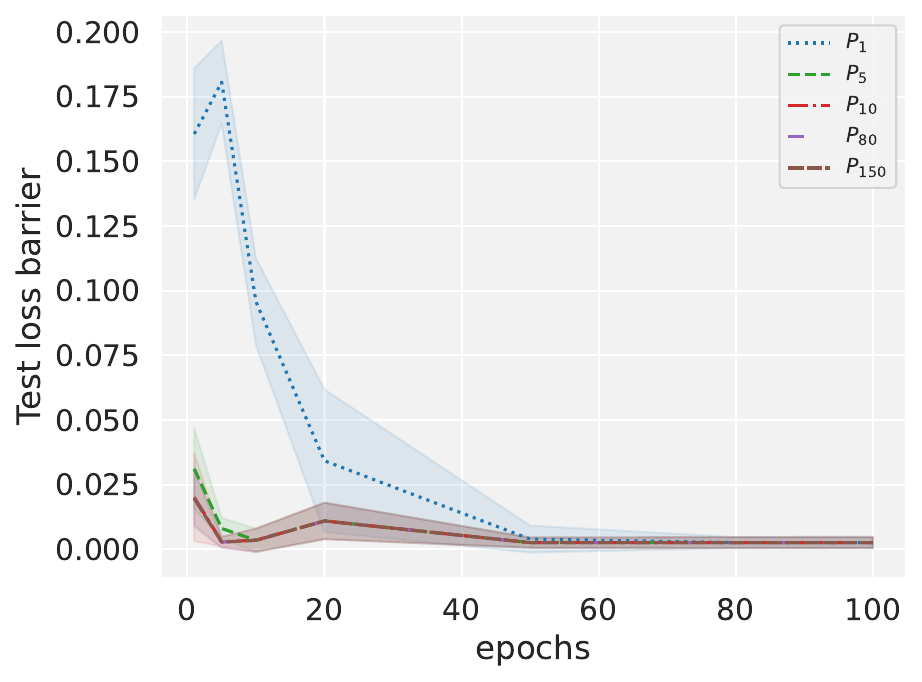}
\includegraphics[width=.24\textwidth]{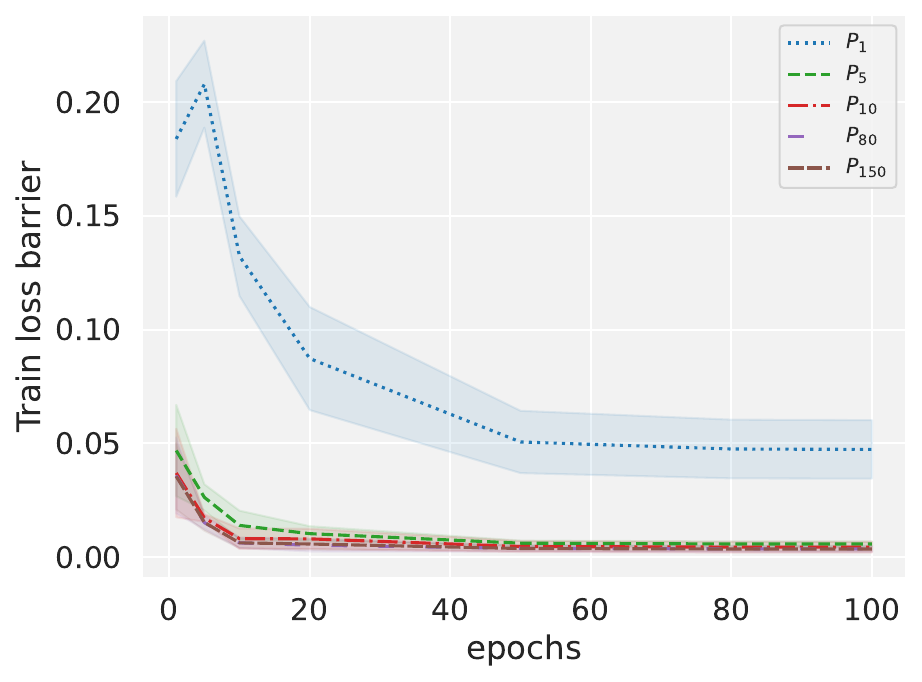}
\\
\includegraphics[width=.24\textwidth]{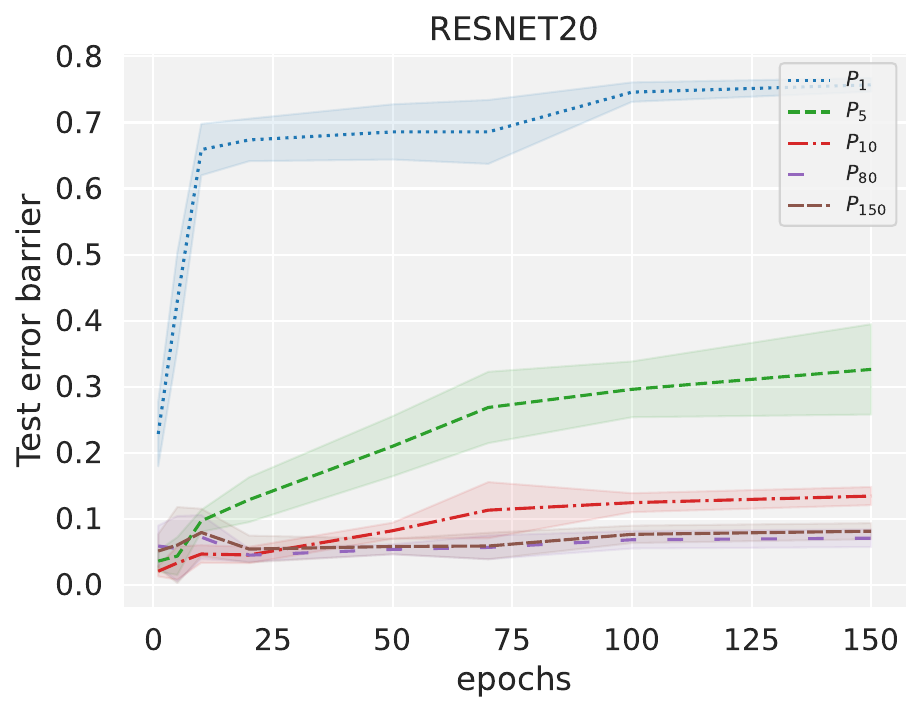}
\includegraphics[width=.24\textwidth]{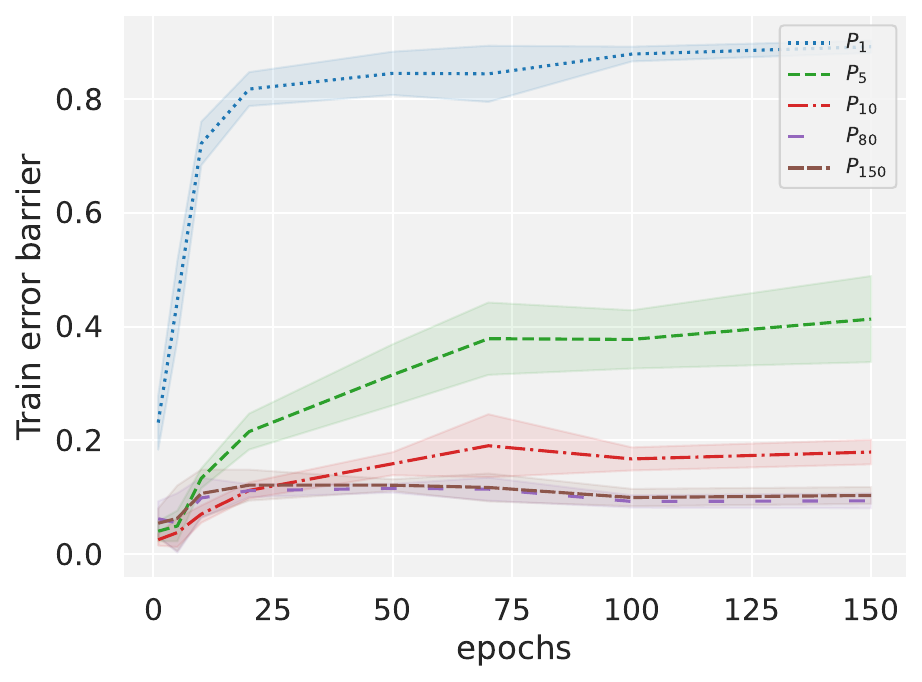}
\includegraphics[width=.24\textwidth]{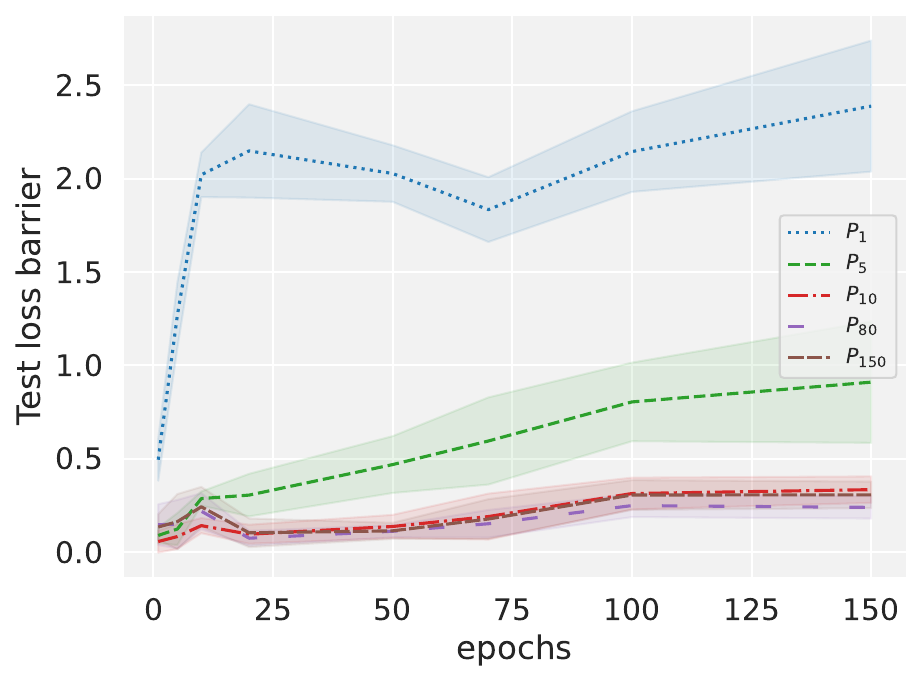}
\includegraphics[width=.24\textwidth]{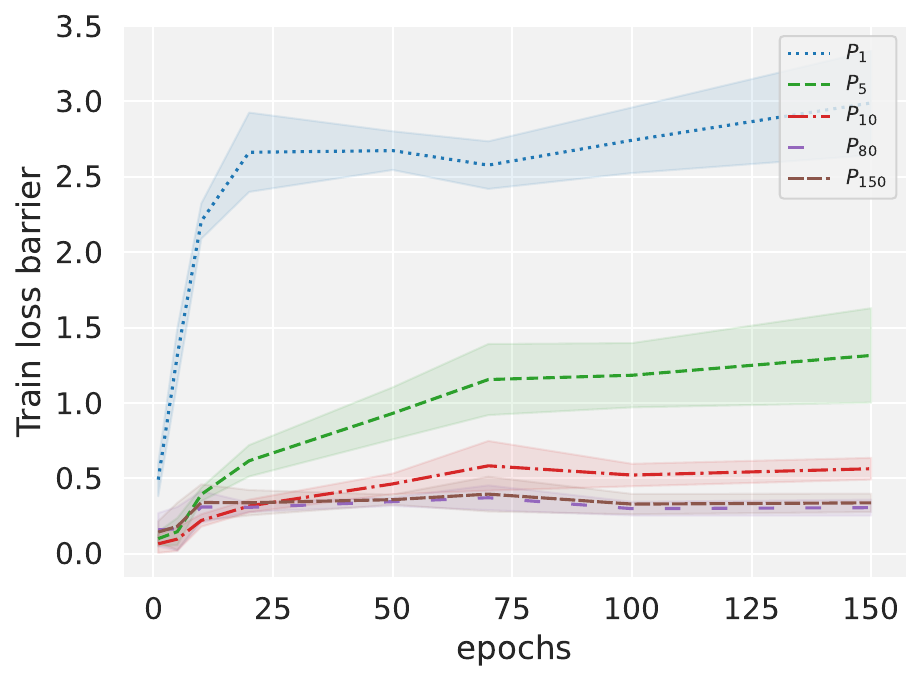}
\\
\includegraphics[width=.24\textwidth]{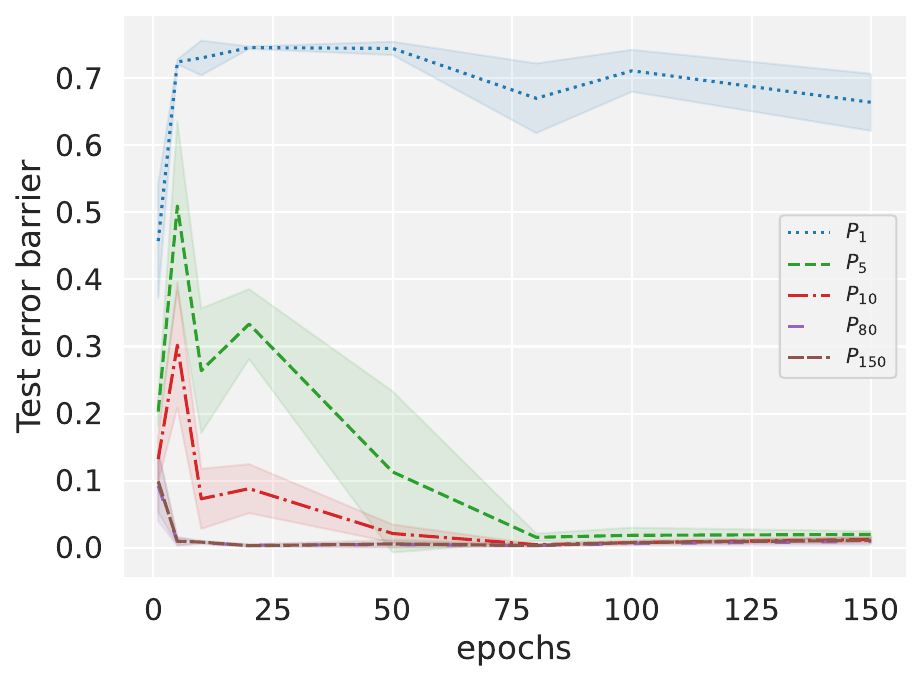}
\includegraphics[width=.24\textwidth]{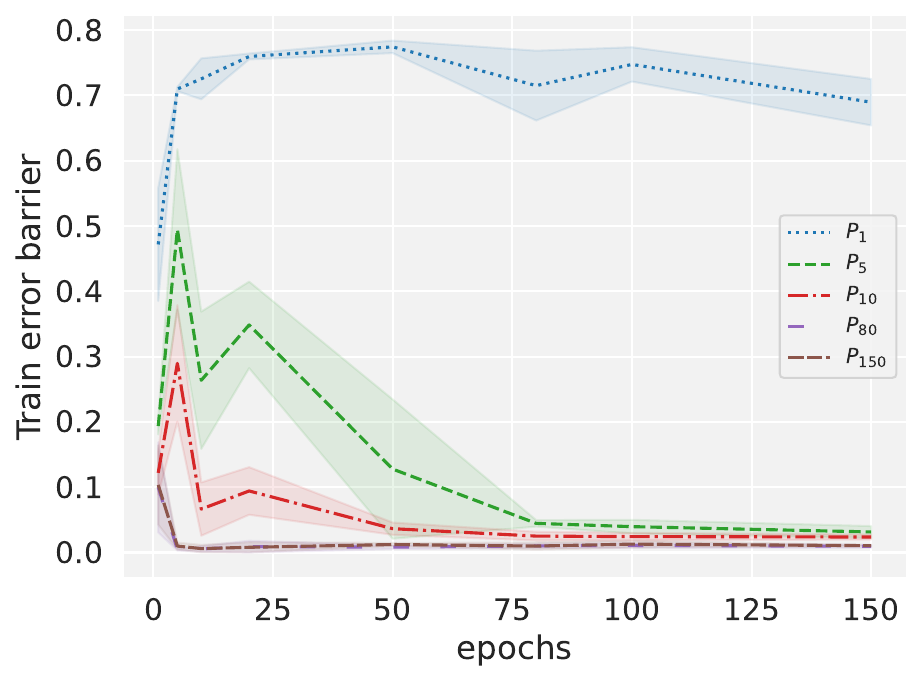}
\includegraphics[width=.24\textwidth]{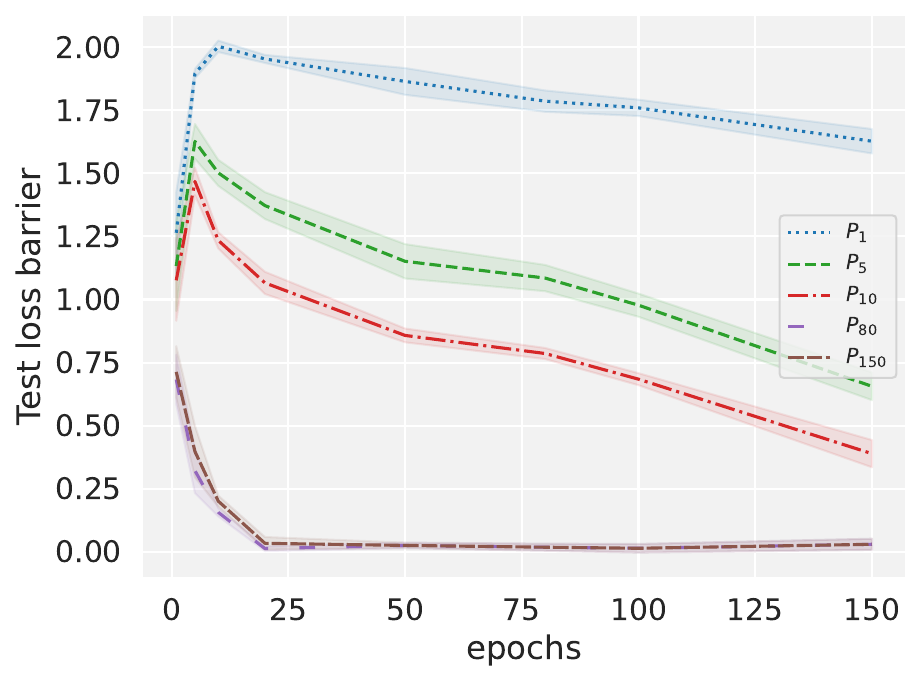}
\includegraphics[width=.24\textwidth]{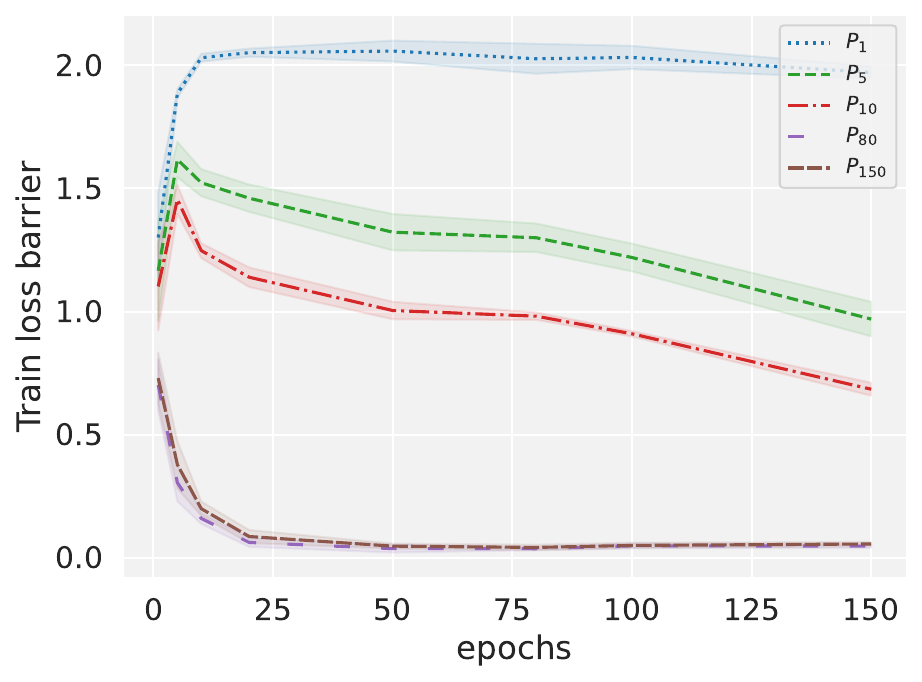}
\\
\includegraphics[width=.24\textwidth]{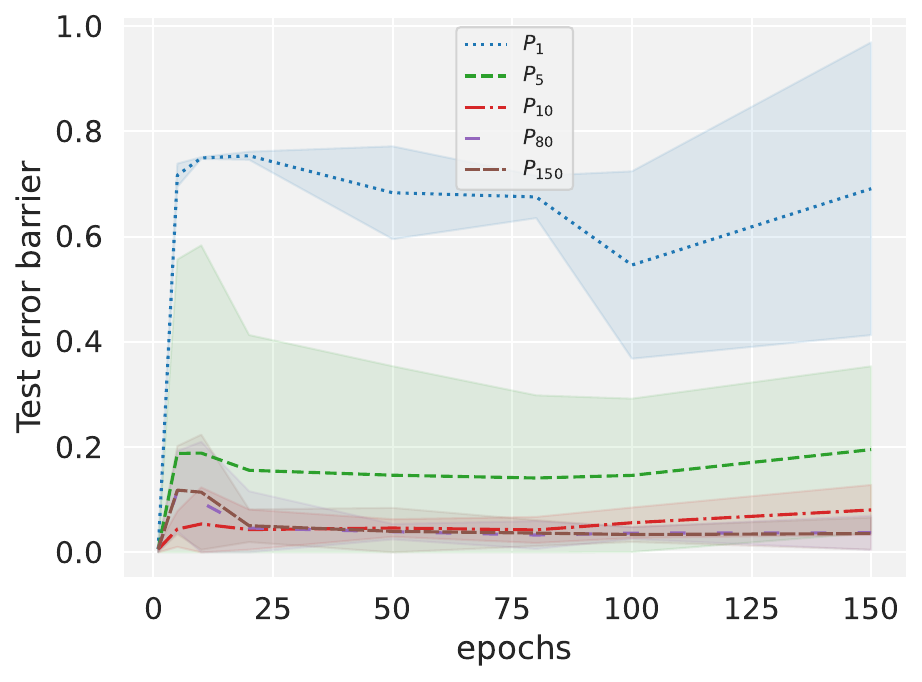}
\includegraphics[width=.24\textwidth]{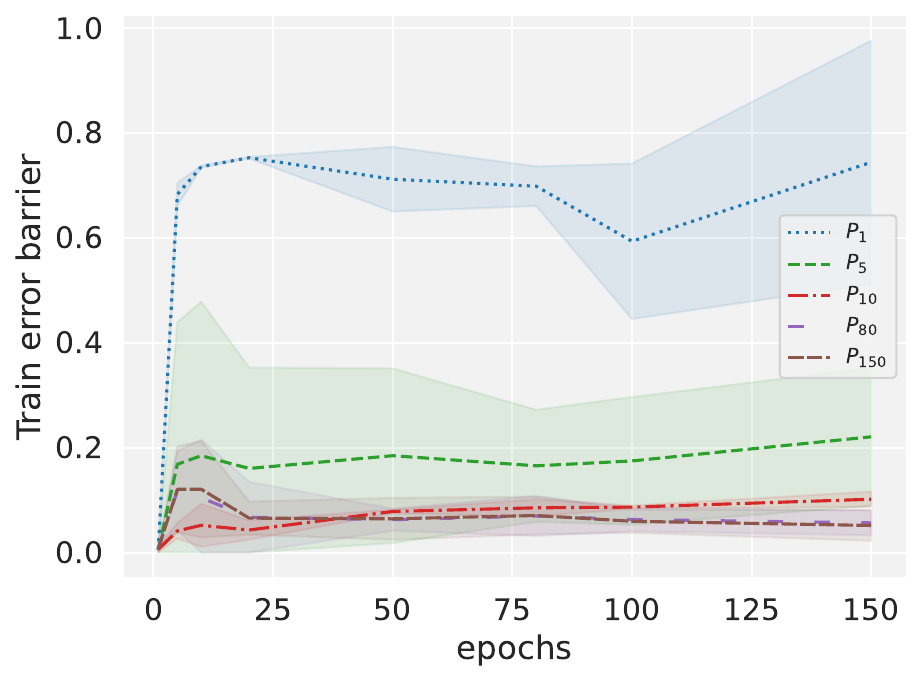}
\includegraphics[width=.24\textwidth]{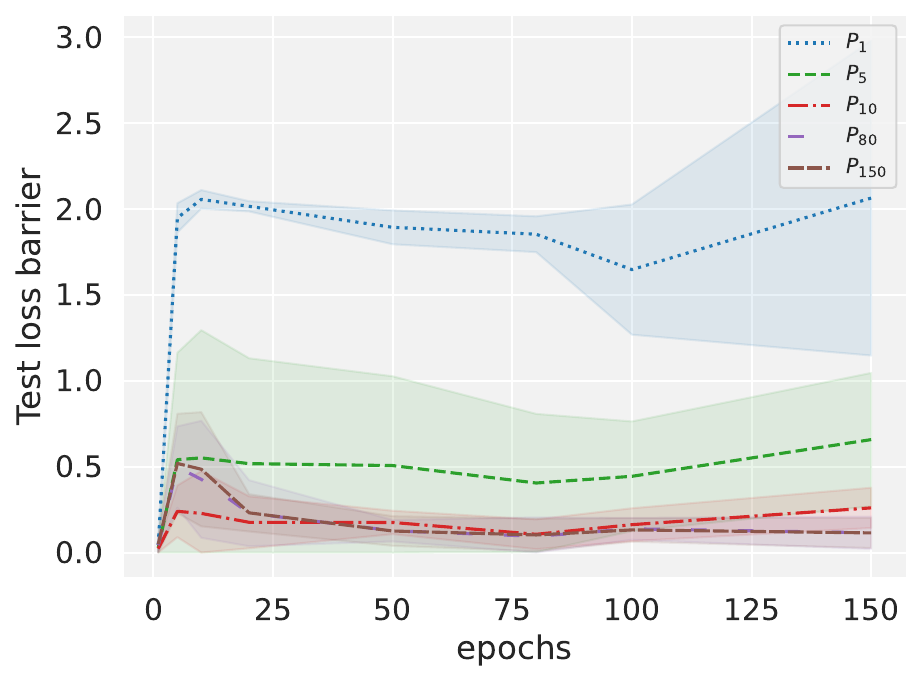}
\includegraphics[width=.24\textwidth]{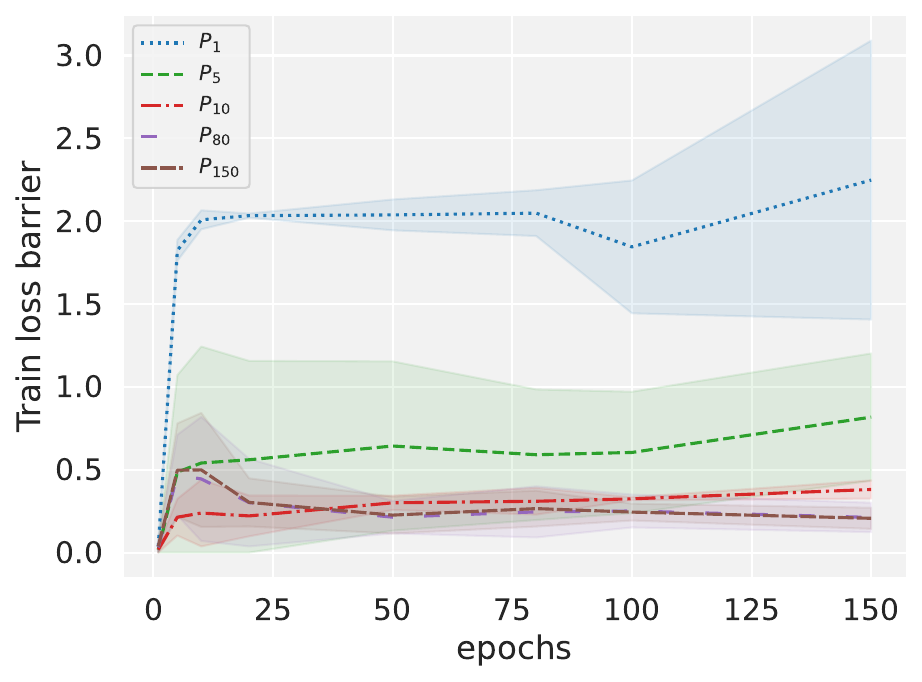}
\\
\includegraphics[width=.24\textwidth]{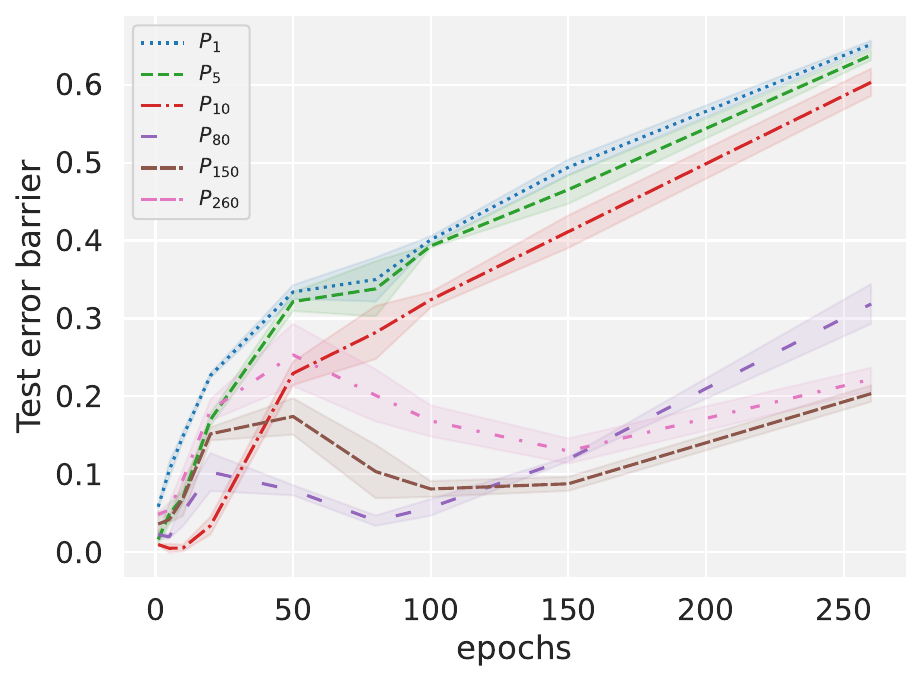}
\includegraphics[width=.24\textwidth]{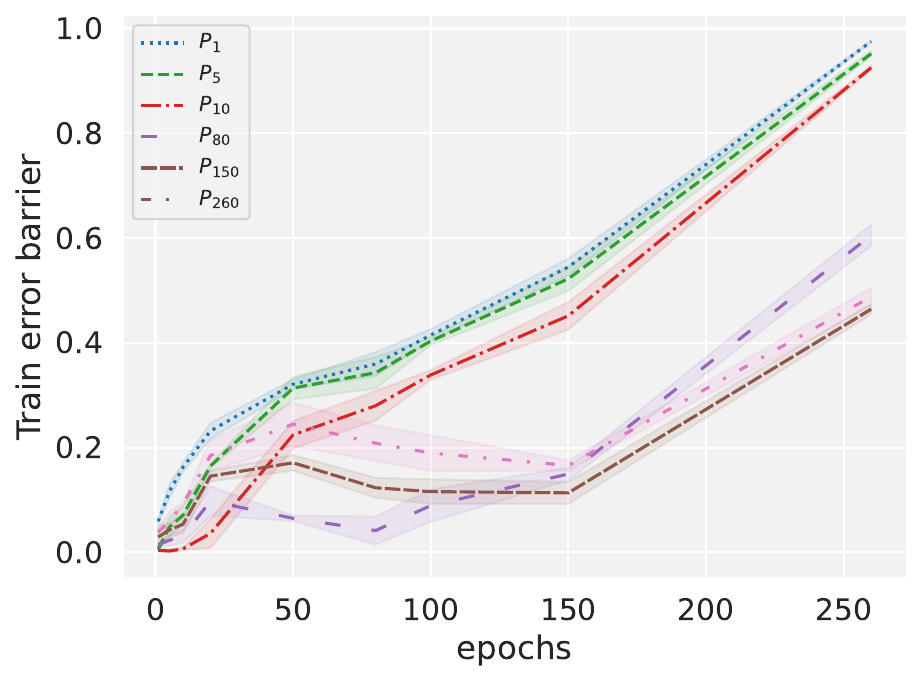}
\includegraphics[width=.24\textwidth]{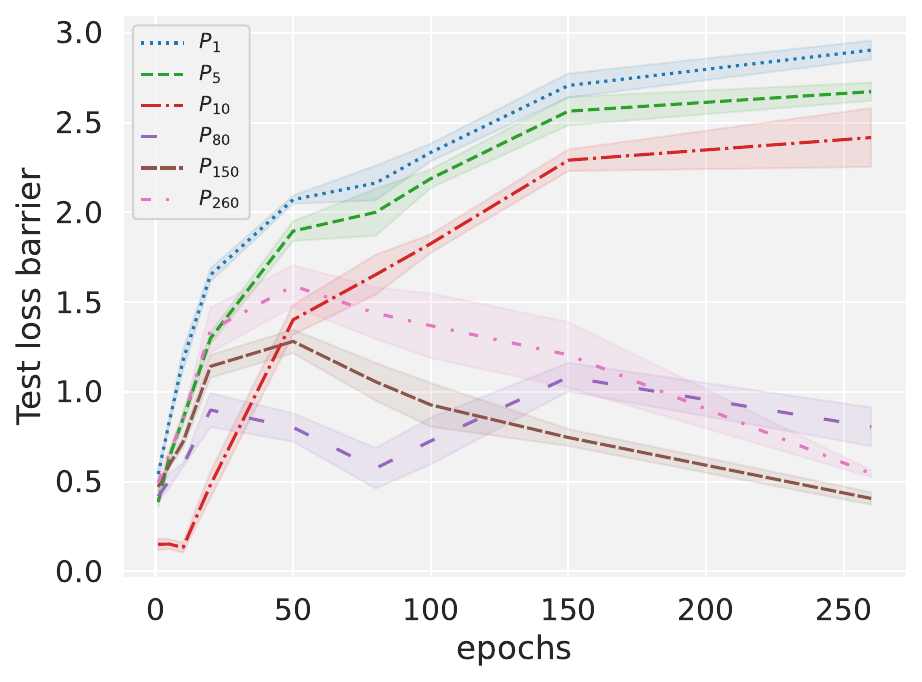}
\includegraphics[width=.24\textwidth]{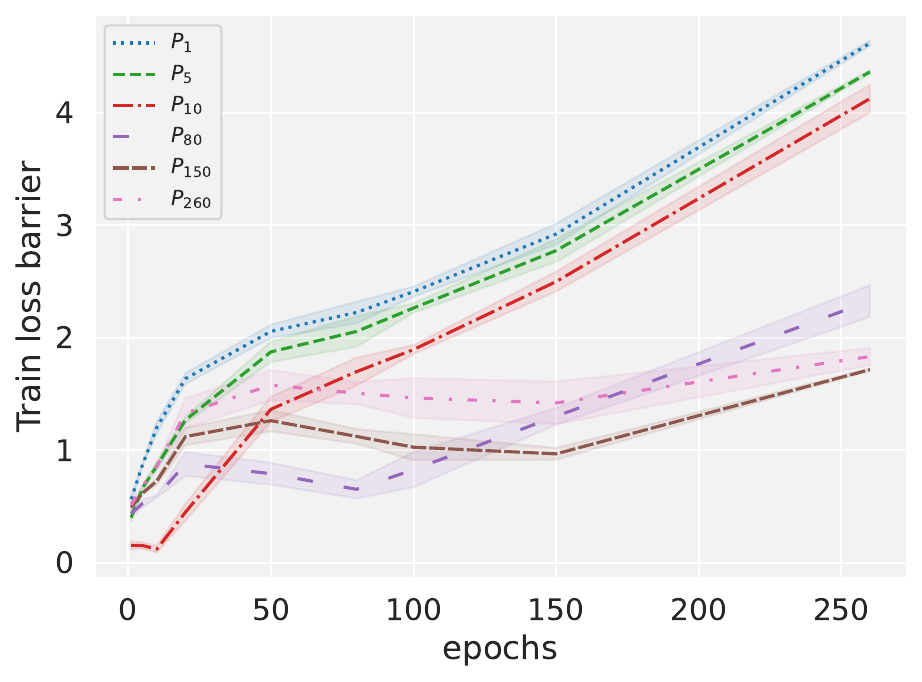}
\\
\includegraphics[width=.24\textwidth]{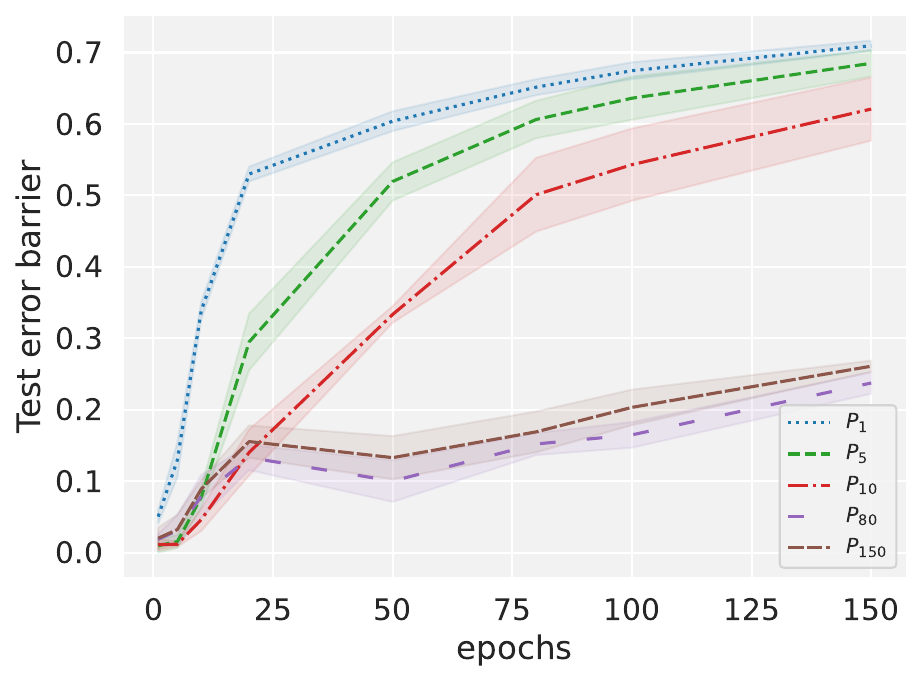}
\includegraphics[width=.24\textwidth]{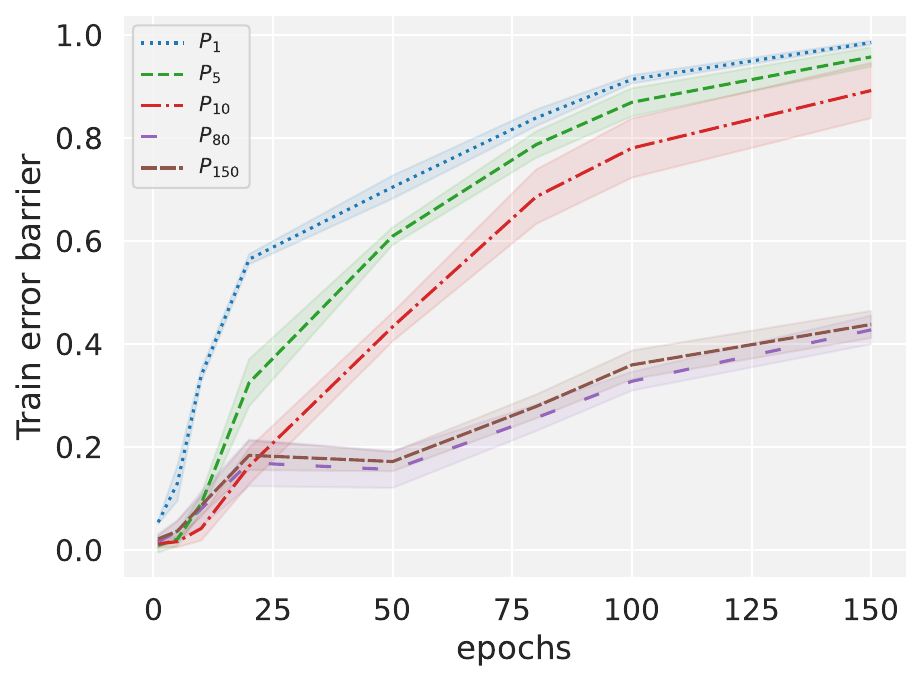}
\includegraphics[width=.24\textwidth]{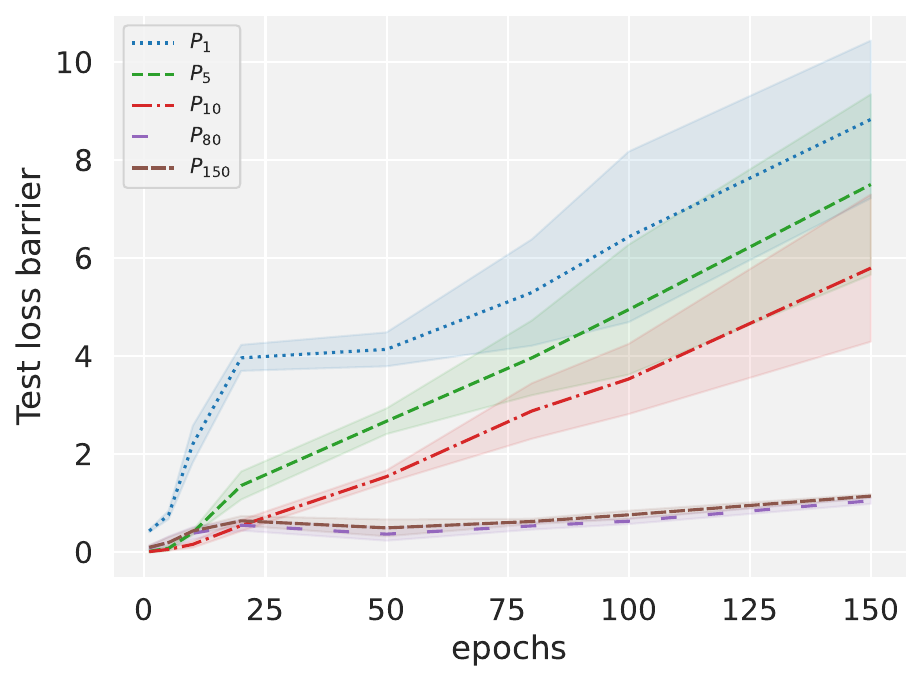}
\includegraphics[width=.24\textwidth]{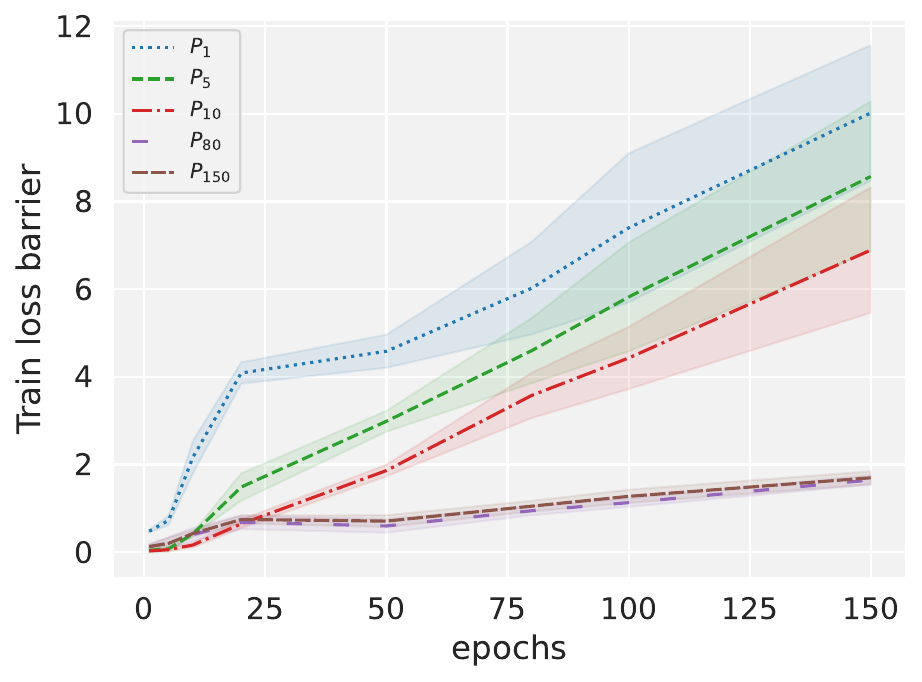}
\\
\caption{
Training trajectories 
Evolution of the error and loss barrier over training for different dataset and model architecture settings.
Test error \textbf{(column 1)}, train error \textbf{(column 2)}, test loss \textbf{(column 3)}, and train loss \textbf{(column 4)} between pairs networks throughout training under a fixed permutation $P_t$ computed at epoch $t$ using weight matching.
Different lines indicate different $t$.
he datasets and architectures are respectively:
\textbf{(row 1)} MNIST and MLP,
\textbf{(row 2)} CIFAR-10 and ResNet-20,
\textbf{(row 3)} SVHN and VGG-16,
\textbf{(row 4)} SVHN and ResNet-20,
\textbf{(row 5)} CIFAR-100 and VGG-16,
\textbf{(row 6)} CIFAR-100 and ResNet-20.
}
\label{fig:trajectoryalignmentdatasetsarch}
\end{center}
\end{figure*}

\clearpage
\section{Comprehensive sparse subnetwork results}
\label{app:sparsesubnetworks}

\paragraph{Permutations between sparse subnetworks.}

For our experiments involving sparse networks generated via iterative magnitude pruning, we use the following procedure:

\begin{enumerate}
    \item
    Train two networks $A$ and $B$ for $k$ iterations.
    \item
    Apply iterative magnitude pruning (IMP) on $A$ by pruning $20\%$ of weights, rewinding to epoch $k'$, and retraining for $k - k'$ epochs.
    Repeat IMP for $i$ iterations to get a pruning mask $M_A^{(i)}$ and sparse network $A' = A \otimes M_A^{(i)}$.
    Similarly, apply IMP on $B$ to get $B' = B \otimes M_B^{(i)}$.
    \item
    Use weight matching to find the dense-dense permutation $P: B \to A$, and permute $B$ to match $A$ as $PB$.
    Similarly, find the sparse-sparse permutation $P^{(i)}: B' \to A'$.
    \item
    Apply the mask from $A$ on $B$ by applying the dense-dense permutation $P$ to get a sparse network $B_{A'} = PB \otimes M_A^{(i)}$.
    Similarly, apply the mask from $A$ on $B$ using the sparse-sparse permutation $QB \otimes M_A$.
    \item
    \begin{enumerate}
        \item 
        \emph{Linear mode connectivity}: evaluate loss barriers between $A'$ and $B_{A'}$.
        Control: $barrier(A', B')$
        Baseline: $barrier(A', B \otimes M_{\text{one-shot}})$, where $M$ is the one-shot pruning mask generated from $B$ with equivalent sparsity to $M_A^{(i)}$.
        Previous work: $barrier(A', QB \otimes M_A)$.
        Gold standard: $barrier(A', PB')$.
        \item
        \emph{Lottery ticket hypothesis}: rewind weights of $B_{A'}$ to epoch $k' < k$ and retrain for $k - k'$ epochs.
        Compute test accuracy of retrained network.
        Control: rewind and retrain $B \otimes M_A$.
        Baseline: rewind and retrain $B \otimes M_{\text{one-shot}}$.
        Previous work: rewind weights of $P^{(i)} B \otimes M_A$.
        Gold standard: $B'$ (already rewound and retrained during IMP).
    \end{enumerate}
\end{enumerate}

We train to epoch $k=150$ and rewind to epoch $k' = 5$.

\paragraph{Additional sparse subnetwork figures.}

In addition to the loss barrier results reported in \cref{fig:perm-sparse-error-barrier} (left), 
we also look at the error barrier between a pair of sparse IMP subnetworks at different sparsity levels and under different permutations.
\cref{fig:sparsityplotsall} confirms that the same observation holds across networks and different training and test loss functions. In particular, the plot in \cref{fig:sparsityplotsall} show that applying the procedure in \cref{app:algorithms}, we are able to get linear mode connectivity at significantly higher sparsity levels as compared to naively applying the weight matching algorithm to sparse subnetworks. 

\begin{figure}[H]
\vskip 0.2in
\begin{center}
\centerline{\includegraphics[width=0.8\columnwidth]{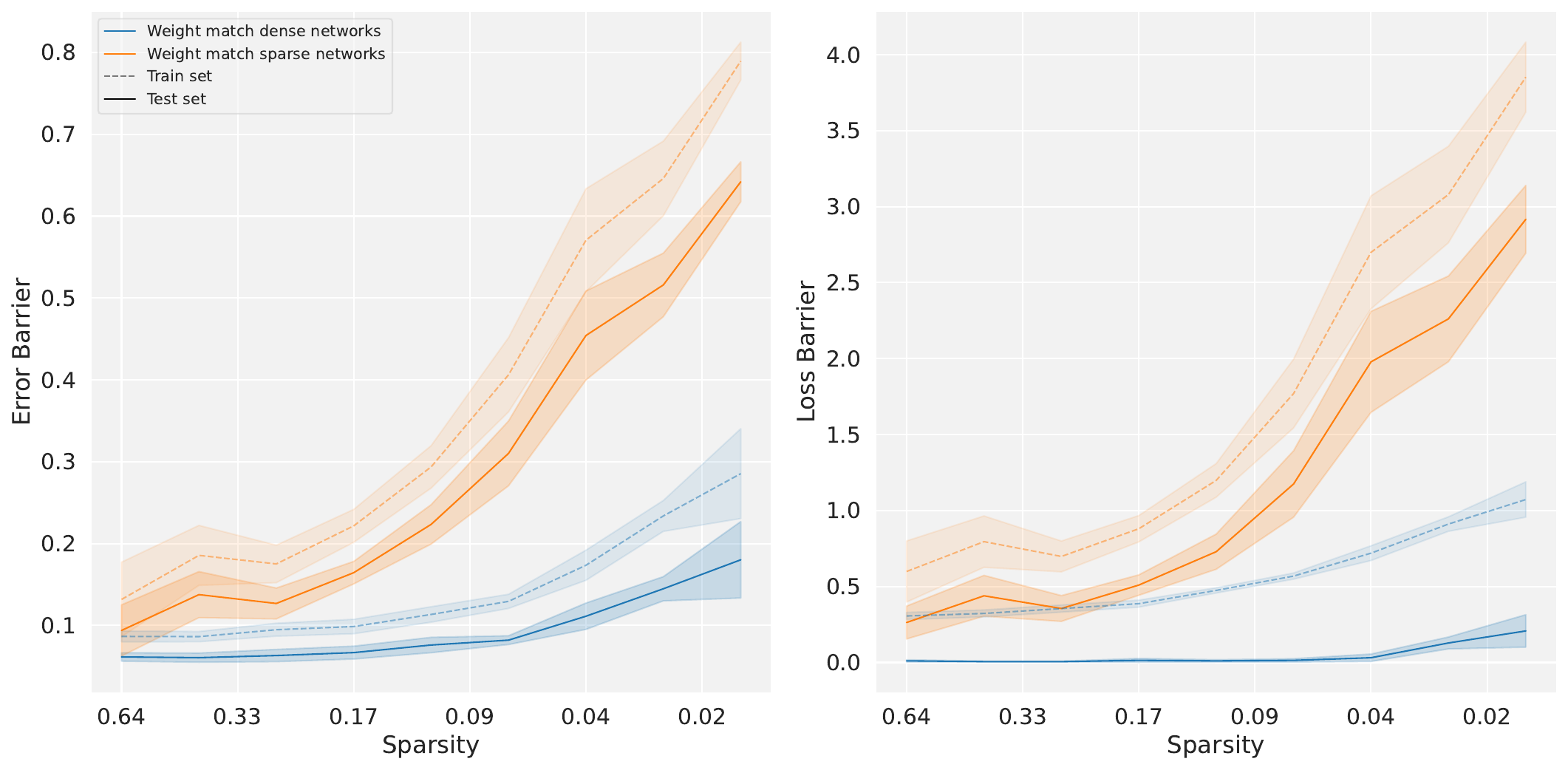}}
\centerline{\includegraphics[width=0.8\columnwidth]{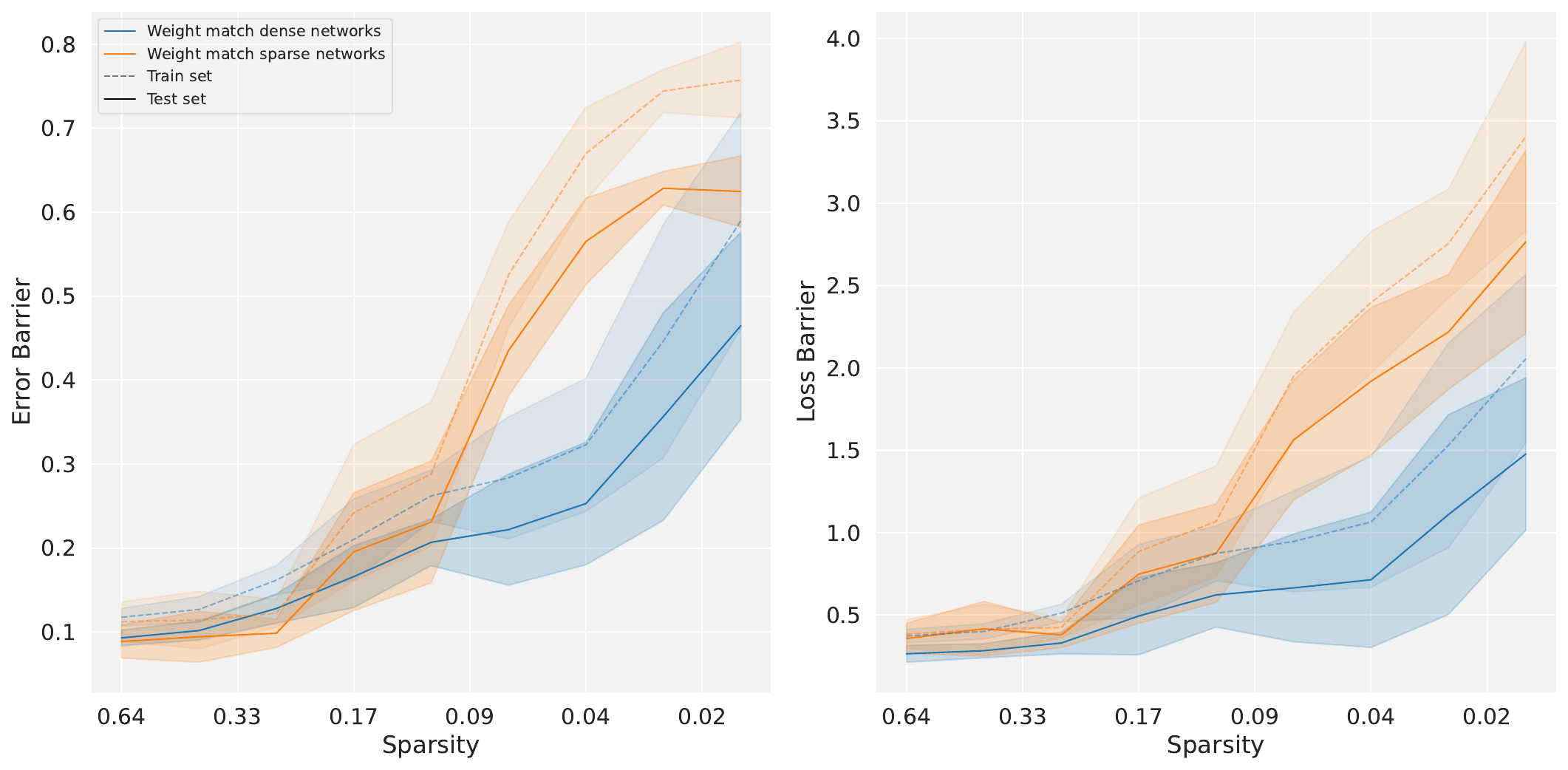}}
\caption{
Additional results on permuting sparse subnetworks as in \cref{fig:perm-sparse-error-barrier} (left).  \textbf{(top row)} VGG16 networks; \textbf{(bottom row)} ResNet-20 networks; \textbf{(first column)} loss barrier; \textbf{(second column)} error barrier. 
}
\label{fig:sparsityplotsall}
\end{center}
\vskip -0.2in
\end{figure}

\begin{figure}[H]
\vskip 0.2in
\begin{center}
\centerline{\includegraphics[width=0.8\columnwidth]{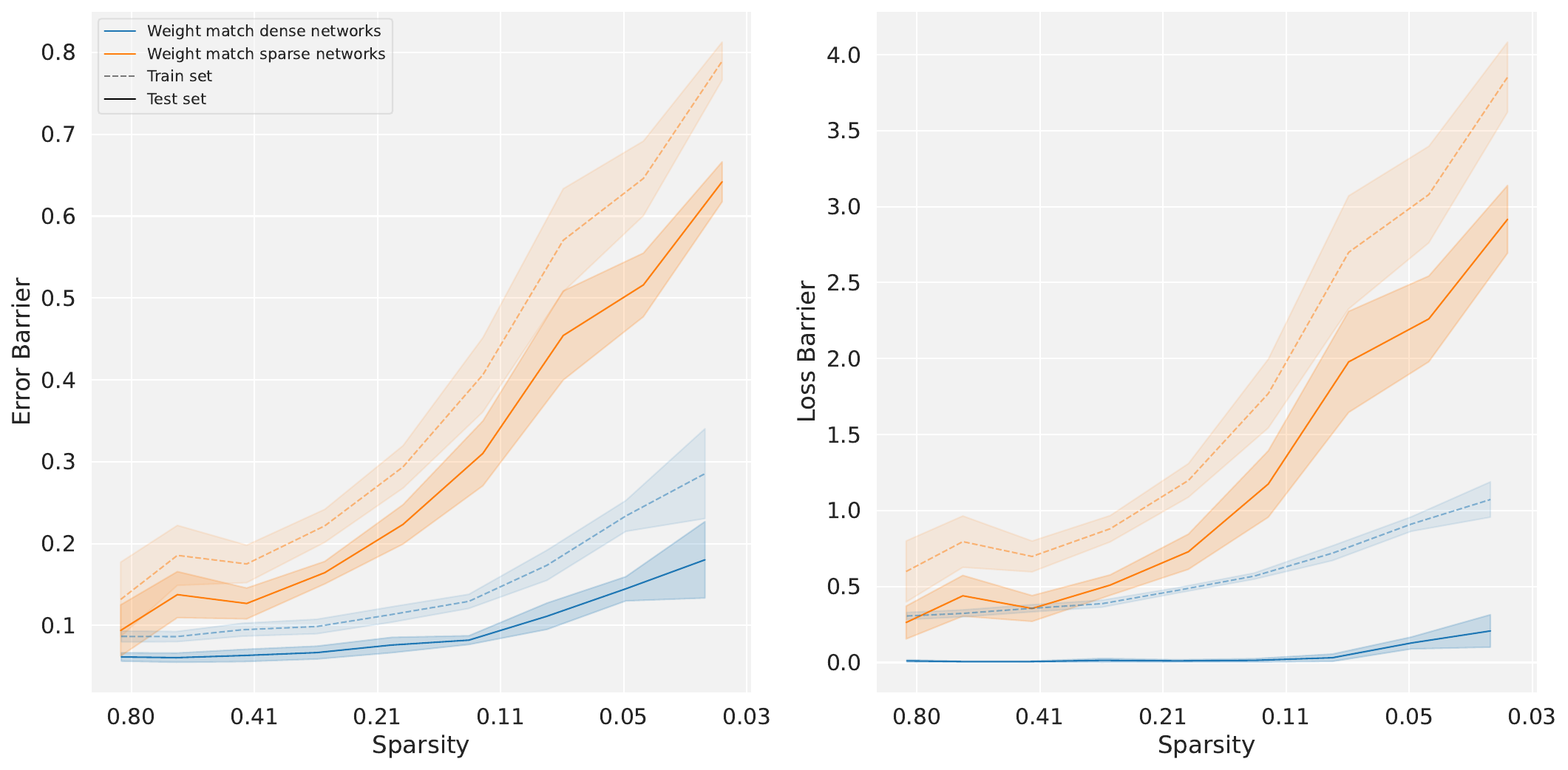}
}
\centerline{\includegraphics[width=0.8\columnwidth]{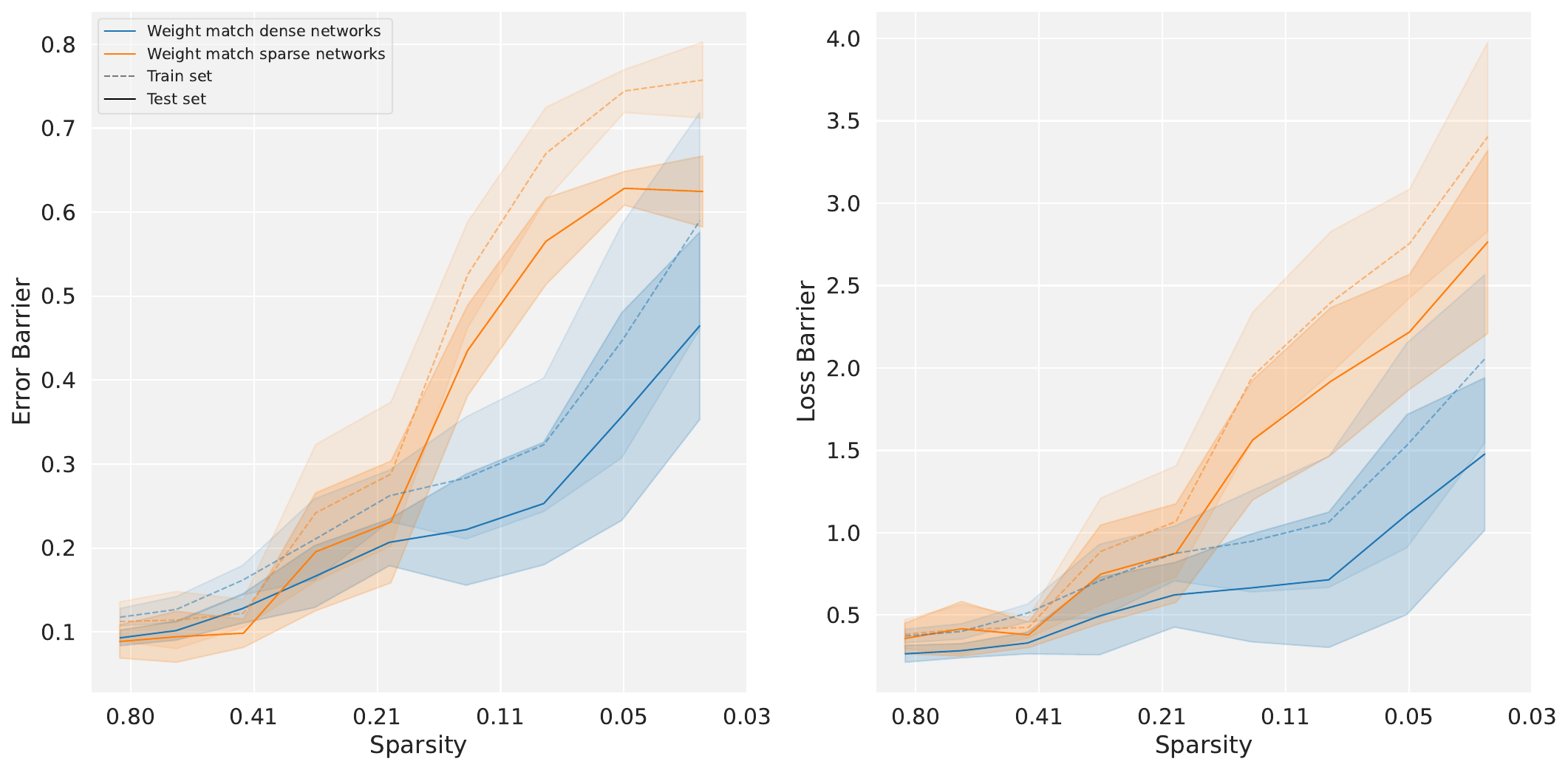}
}
\caption{
Error and loss barriers for sparse networks. Here x-axis corresponds to the sparsity of the interpolated model.
}
\label{fig:sparsityplots_xaxis}
\end{center}
\vskip -0.2in
\end{figure}

\paragraph{Additional results related to same permutation aligning multiple networks.}

In addition to the plots in \cref{fig:perm-sparse-error-barrier} (right), here we include further experiments looking at transferability of a mask computed on one network onto another network modulo permutation. In \cref{fig:mask-transitivity-error-additional-plot} we add two additional comparisons: 1) We look at permutations computed using activation matching and weight matching separately, and 2) we look at transferability of masks when sparse model is obtained using learning-rate rewinding instead of weight rewinding, i.e. we obtain sparse model $B$ by pruning $P[B]$ using the mask $m_A^{(i)}$, and retraining the pruned network using original learning rate schedule.

Our results show that permutation computed using weight matching allow to achieve matching accuracy for a higher sparsity level. Secondly, we learn that learning rate rewinding achieves matching accuracy for a lower sparsity level than original weight-rewinding scheme.

\begin{figure}[ht]
\vskip 0.2in
\begin{center}
\centerline{\includegraphics[width=.5\textwidth]{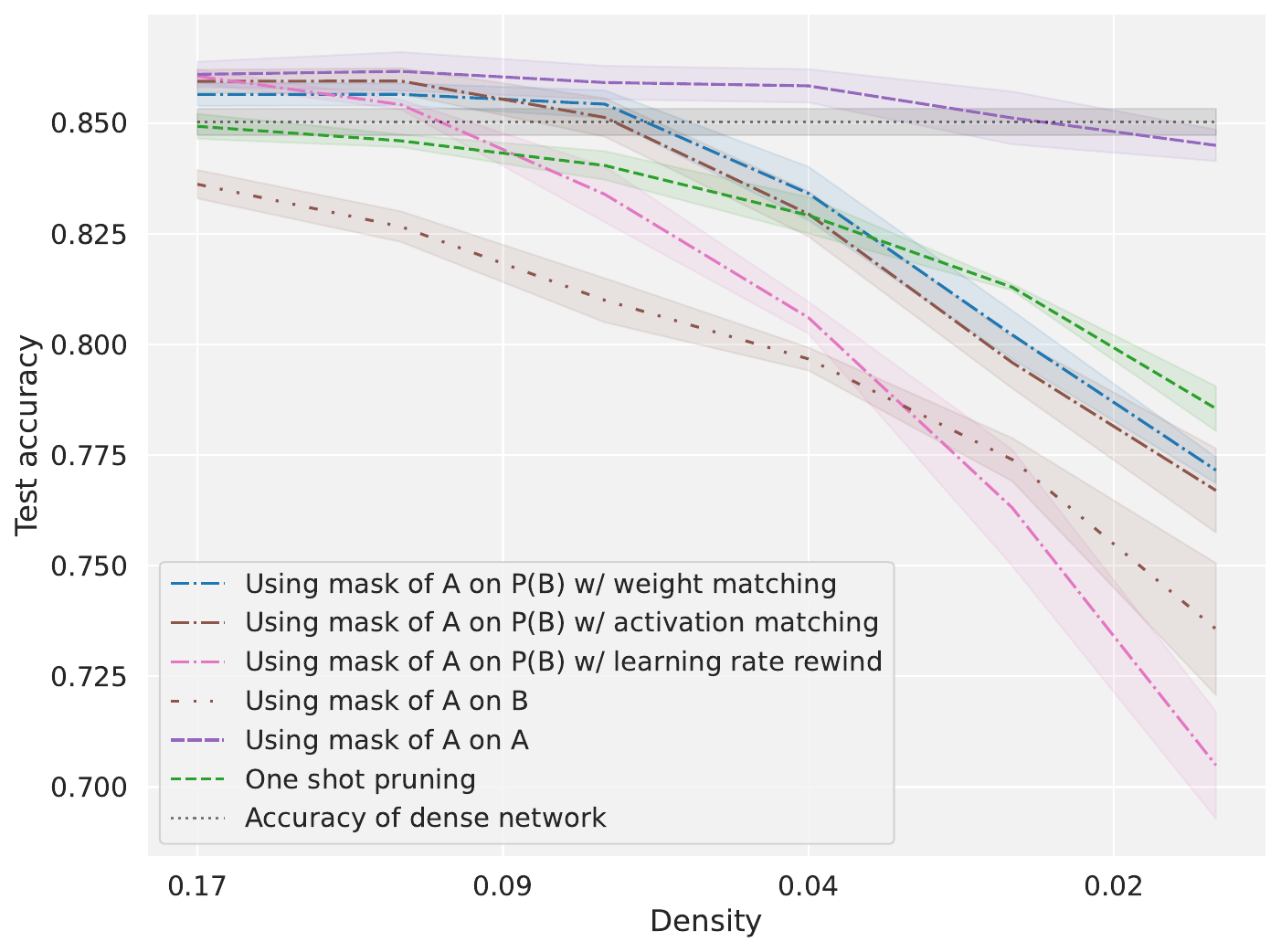}}
\caption{Test accuracy of different types of sparse networks at increasing levels of sparsity. The x-axis indicates the fraction of weights remaining.}
\label{fig:mask-transitivity-error-additional-plot}
\end{center}
\vskip -0.3in
\end{figure}

\section{Comprehensive results testing strong linear connectivity modulo permutation}
In \cref{fig:transitivity} we plot both loss and accuracy barriers associated with the train and test data for the experiment studied in \cref{sec:evidence-strong-lcmodp}. The plots show that we significantly reduce the loss/error barrier with indirect alignment for wide networks. 

\paragraph{Fixed points.}
\label{subsec:fixed-points}
We also compare the permutations found by indirect versus direct alignments using fixed points.

Fixed points are the elements on which two permutations agree, i.e. the channels mapped to the same index by both permutations.
Formally, the number of fixed points in layer $i$ is
\[
\operatorname{FP}_i = \operatorname{trace} \left( (P_t^{(i)})^{T} P_{\mathrm{end}}^{(i)} \right).
\]
The total fraction of fixed points over $k$ layers is
\[
\operatorname{FP} = \frac{\sum_{i=1}^k f_i}{\sum_{i=1}^k n_i},
\]
where $n_i$ is the size of each permutation $P_i$.

For each of the permutations used to compute barriers in \cref{fig:transitivity} (\textbf{top}), we plot the corresponding fixed points between direct and indirect alignment in \cref{fig:transitivity} (\textbf{bottom}). We find that direct and indirect alignment do not produce similar permutations, and that increasing width does not make the permutations more similar.
This indicates that the permutations found by weight and alignment matching vary widely depending on the pair of network weights being aligned.

\begin{figure}[t]
\begin{center}
    \includegraphics[height=4cm]{figures/transitivity-width-VGG-16_train_loss_barrier.pdf}
    \includegraphics[height=4cm]{figures/transitivity-width-ResNet-20_train_loss_barrier.pdf}
    \\
    \vskip -0.14in
    \includegraphics[height=4cm]{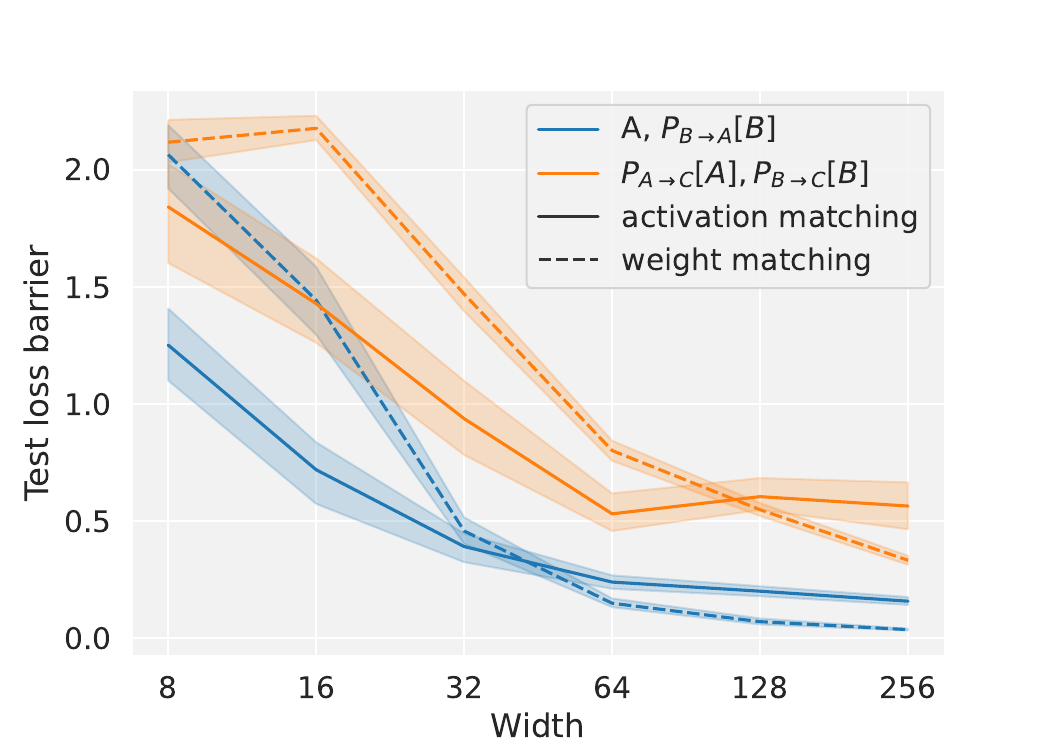}
    \includegraphics[height=4cm]{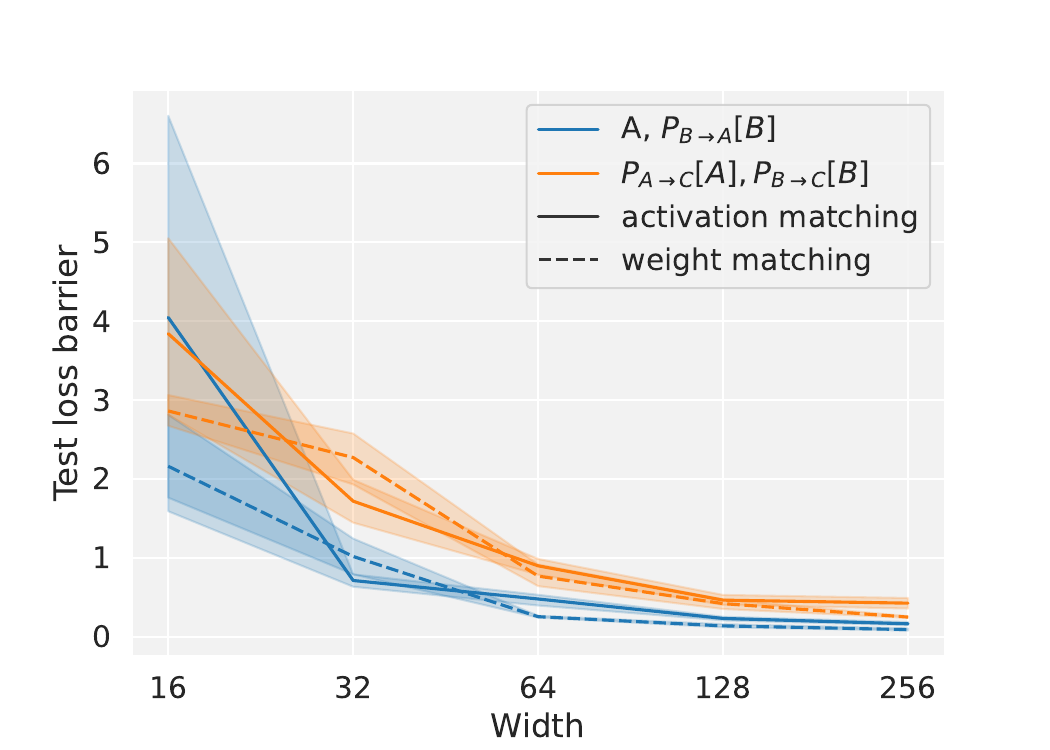}
    \\
    \vskip -0.14in
    \includegraphics[height=4cm]{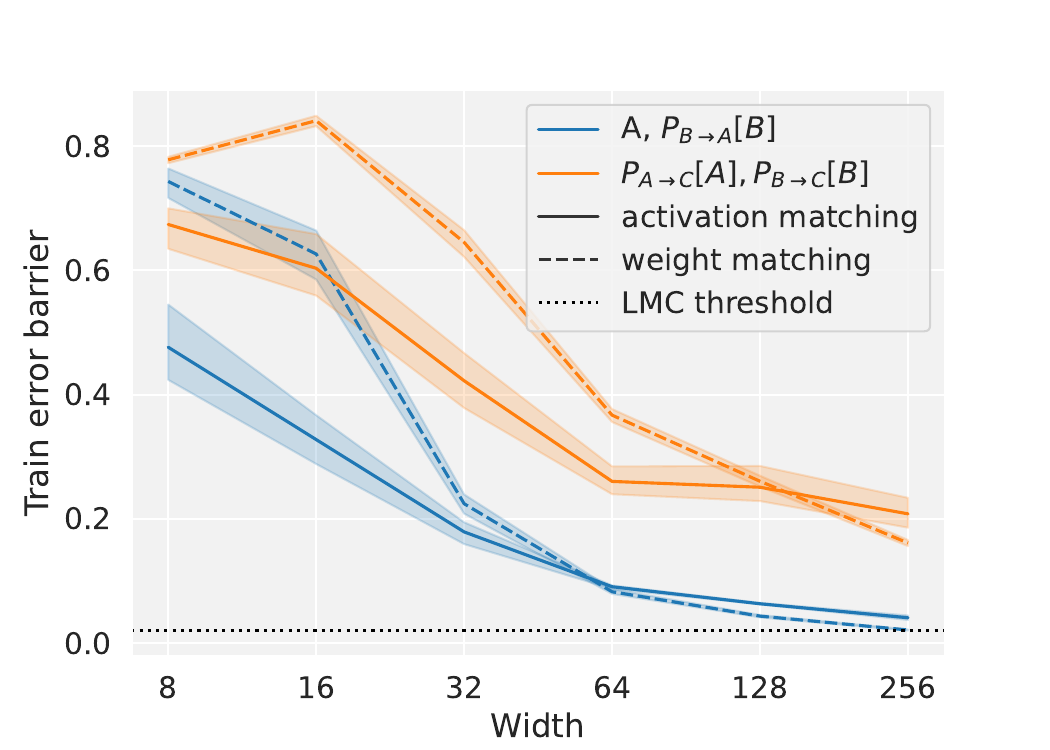}
    \includegraphics[height=4cm]{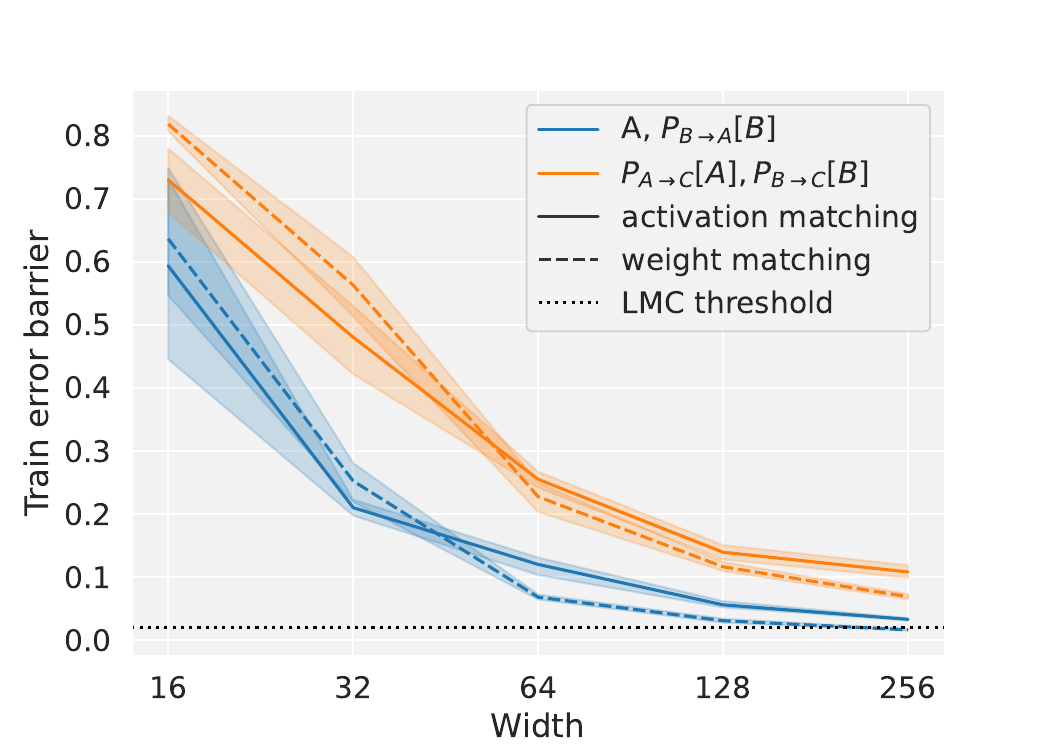}
    \\
    \vskip -0.14in
    \includegraphics[height=4cm]{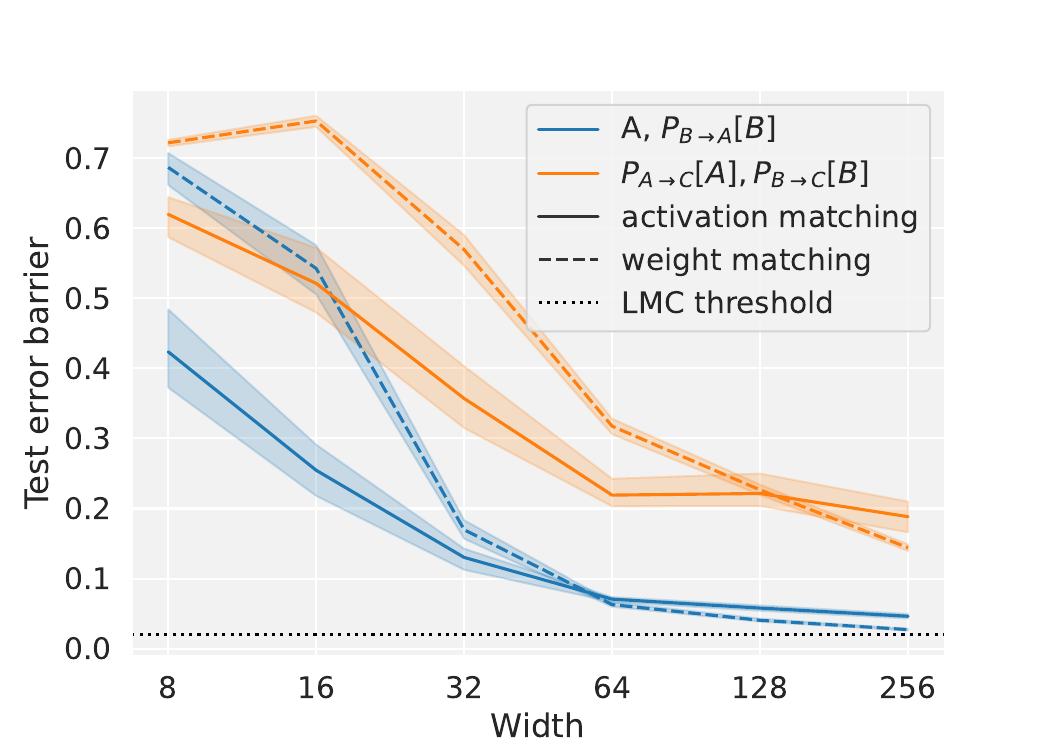}
    \includegraphics[height=4cm]{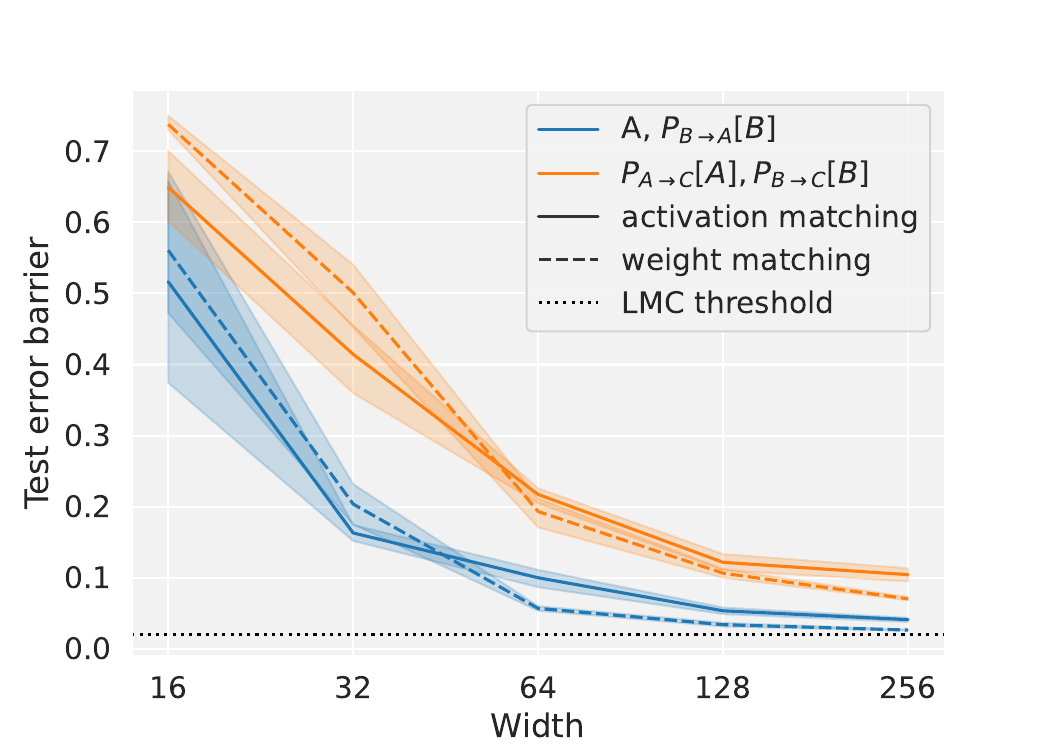}
    \\
    \vskip -0.14in
    \includegraphics[height=4cm]{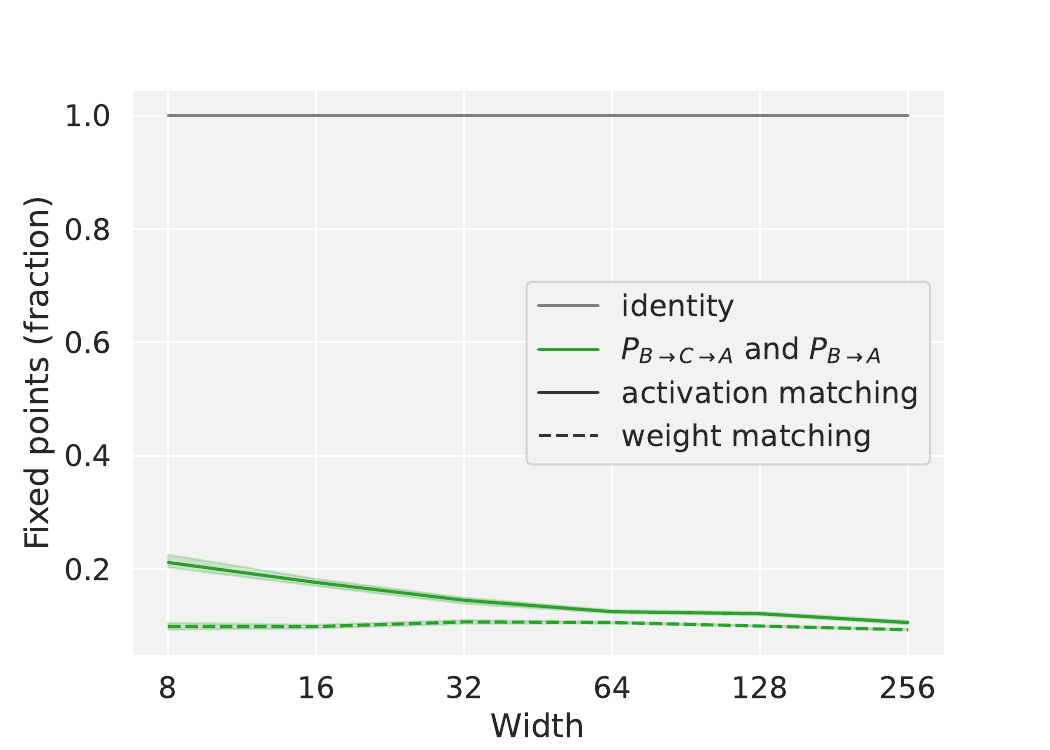}
    \includegraphics[height=4cm]{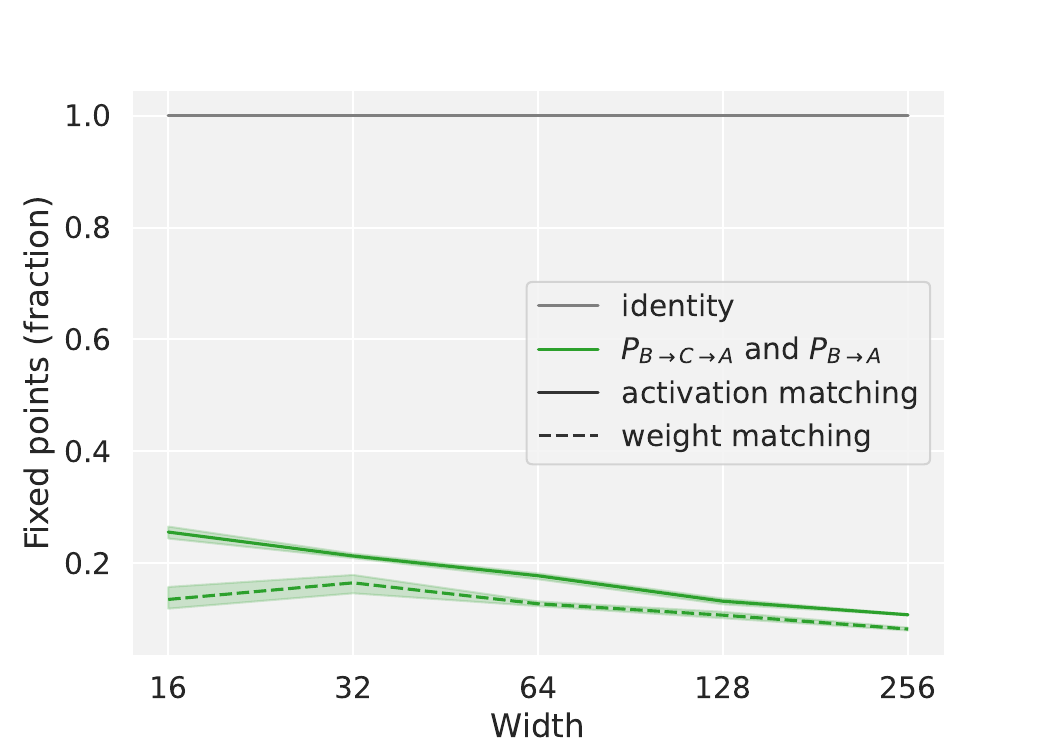}
    \vskip -0.1in
\caption{
    Test of strong linear connectivity modulo permutation, comparing barriers (y-axis) of networks aligned directly or relative to a reference network (colors) via activation and weight matching (line styles).
    \textbf{(Left)} VGG-16 models of increasing width (x-axis).
    \textbf{(Right)} ResNet-20 models.
    \textbf{(Rows 1-2)} loss barriers for train and test data.
    \textbf{(Rows 3-4)} error (accuracy or 0--1 loss) barriers.
    \textbf{(Row 5)} similarity of permutations found by direct versus indirect alignment, as a fraction of fixed points.
}
\label{fig:transitivity}
\end{center}
\end{figure}

\clearpage
\section{Limitations of weight matching}

\subsection{Weight matching fails early in training}
\label{sec:weightmatchingfails}

In \cref{appfig:trajectoryalignment}, we observe that the weight matching algorithm fails to minimize the loss barrier when applied early in training. 
This may be caused by a failure of the alignment algorithm itself (e.g. due to the noisiness or small scale of the weights early in training), or a lack of correlation between the alignment objective (which depends on the weights only) and the loss barrier (which depends on the data).
Notably, we find that $P_t$ more closely aligns weights in terms of $L^2$ distance than $P_{\mathrm{end}}$ early in training. %
This suggests a disconnect between the weight matching objective and the loss barrier early in training.

We investigate how the training loss barrier under a fixed $P_t$ evolves as training progresses.
\cref{appfig:trajectoryalignment} (middle) shows that when the permutation is computed early in training, the loss barrier between trajectories first increases and then decreases (red and green lines). 
This pattern suggests that at time $t$ when the permutation is computed, some key features have not yet been learned, and thus the weight values associated with these features are potentially too small and undifferentiated for the weight matching algorithm to succeed in finding the right alignment. See also \cref{app:landscapevis} for visualizations of the loss landscape under these permutations.

Interestingly, looking at \cref{appfig:trajectoryalignment} (right),
we see that the error barrier between a pair of trained networks under $P_t$ rapidly decreases with $t$, despite the large ``bump'' in the error barrier earlier in training, as was seen in \cref{appfig:trajectoryalignment} (middle). 
After around epoch 10, $P_t$ seems to be similar to $P_{\mathrm{end}}$ in terms of the final error barrier.
Note, however, that this bump seems to be architecture-specific (see \cref{fig:trajectoryalignmentallresnet}). 

\begin{figure*}[t]
\begin{center}
\includegraphics[height=4cm]{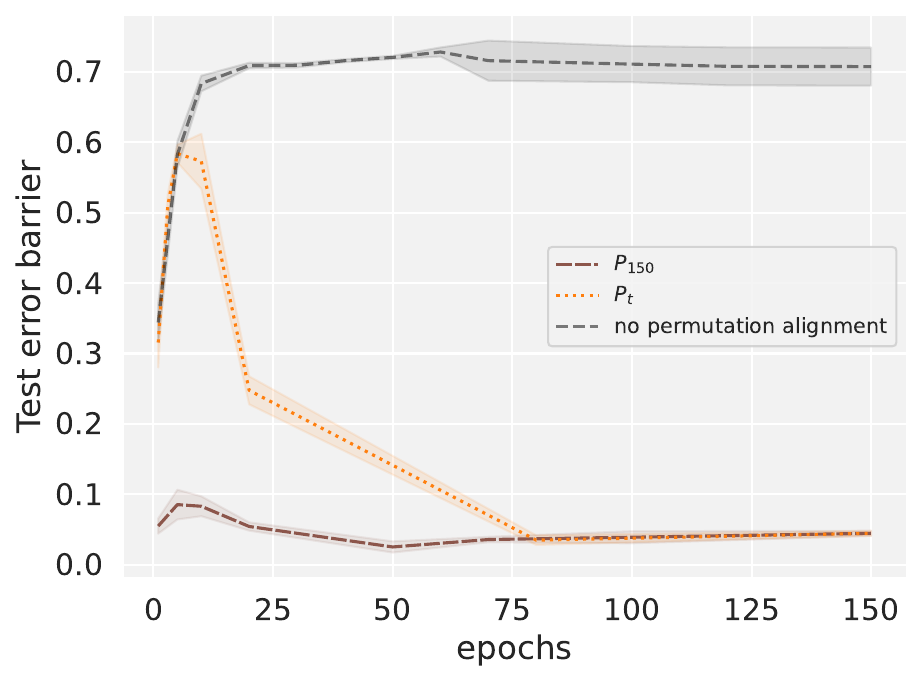}%
\includegraphics[height=4cm]{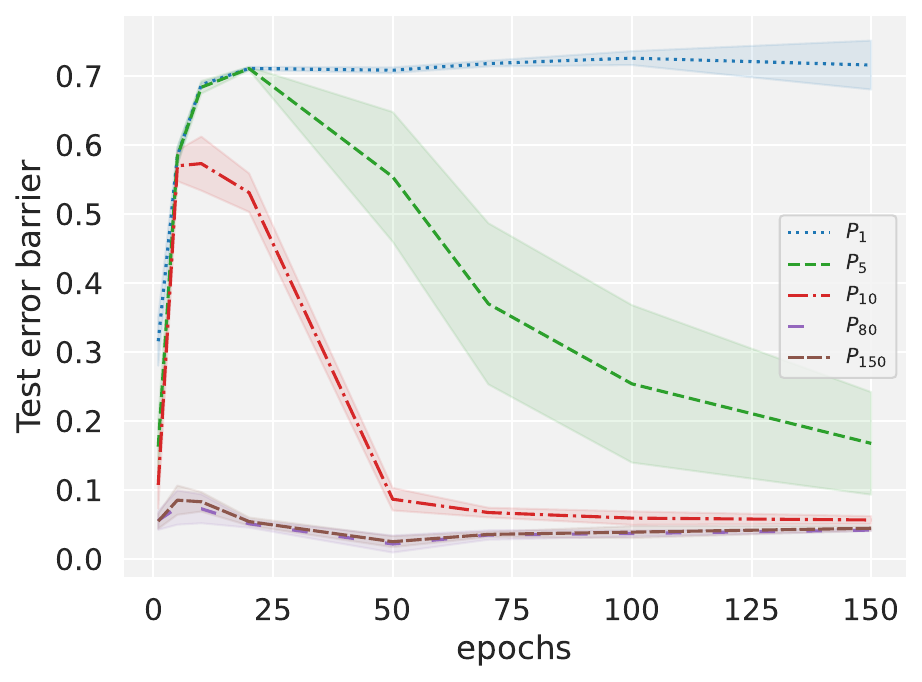}%
\includegraphics[height=4cm]{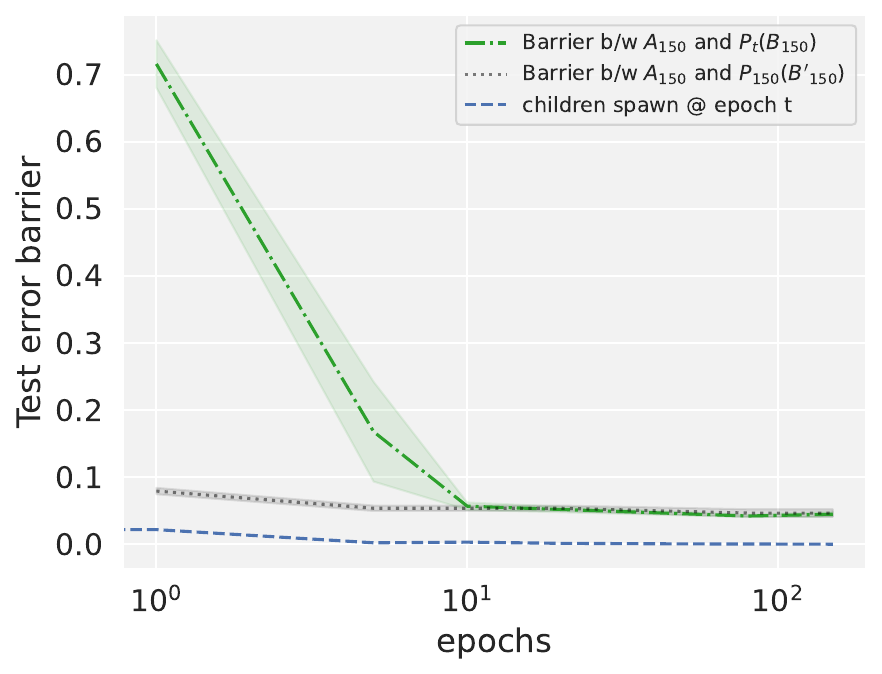}
\caption{
\textbf{(Left)} The error barrier (y-axis) between networks $A_t$ and permuted $B_t$ at different training checkpoints $t$ (x-axis). The brown line corresponds to applying a permutation $P_{\mathrm{end}}$ found at the end of training, which successfully eliminates the error barrier between the networks at nearly every epochs. The orange line corresponds to a permutation $P_t$ computed and applied at time $t$.
\textbf{(Middle)} The evolution of the error barrier between a pair of networks throughout training under a fixed permutation $P_t$ computed at epoch $t$. Each line corresponds to a different values of $t$ as indicated in the legend. Both (left) and (middle) compare networks trained with different initializations and minibatch orders.
\textbf{(Right)} 
The error barrier ($y$-axis) at the end of training for a pair of networks after applying a permutation $P_t$ (dot-dashed green line). The time $t$ when the permutation is computed is indicated on the $x$-axis (log scale).
This error barrier is compared against the error barrier between (1) ``child'' networks spawned from the same ``parent'' network at time $t$ (dashed blue line), and (2) model A and a ``child'' network of model B after applying the permutation $P_{\mathrm{end}}$ found between A and the ``parent'' network B at the end of training (dotted grey line). Each child is trained with a different minibatch order starting from the parent weights at time $t$.
}
\label{appfig:trajectoryalignment}
\end{center}
\end{figure*}

\subsection{Weight matching relies on large-magnitude weights}
\label{sec:weight-matching-magnitude}

Given that weight matching minimizes the $L^2$ (Euclidean) distance between networks in weight space, we hypothesize that weight matching is more sensitive to larger model weights.
To test this hypothesis, we conduct the following experiment:

\begin{enumerate}
    \item Take two independently trained checkpoints $A_t$ and $B_t$ from epoch $t$. Prune a fraction $p$ of the smallest weights in each model to get $A'_t, B'_t$ (pruning is conducted globally over the smallest magnitude weights from all layers).
    \item Use weight matching to align $A'_t$ and $B'_t$ via $P$, and compute the loss (error) barrier between $A'_t$ and $\permuteop{P}{B'_t}$.
    \item Compare barriers versus networks that are randomly pruned to the same sparsity before alignment.
\end{enumerate}

The results of this experiment are shown in \cref{fig:app-prune-then-align}.
As sparsity increases, random pruning (dashed lines) before alignment increases loss (error) barrier, whereas barriers generally do not increase when the smallest weights are pruned.
This relationship holds even when aligning checkpoints from various training epochs using weight alignment, although a larger fraction can be pruned from later checkpoints before affecting barriers.
We conclude that weight matching depends strongly on weights with larger magnitudes.
Additionally, the training process causes weight magnitude to become more correlated with network features over time.

\begin{figure}[t]
\begin{center}
    \includegraphics[width=.4\textwidth]{figures/prune-then-align-VGG-16-ep150-Loss_barrier-train.pdf}
    \includegraphics[width=.4\textwidth]{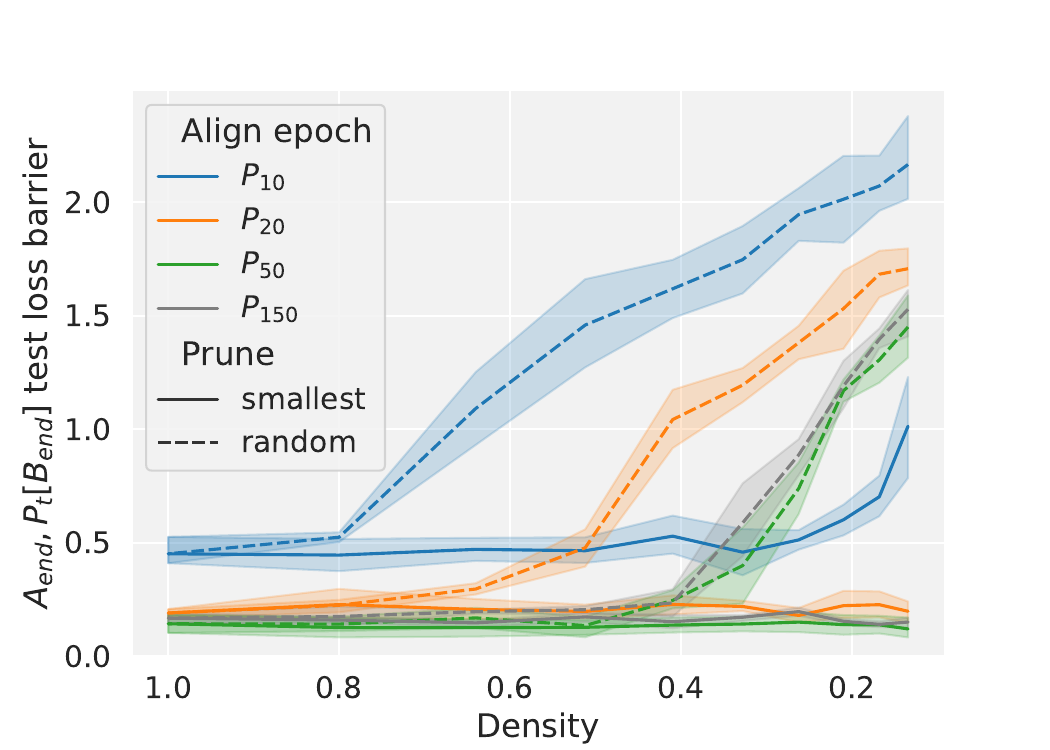}
    \includegraphics[width=.4\textwidth]{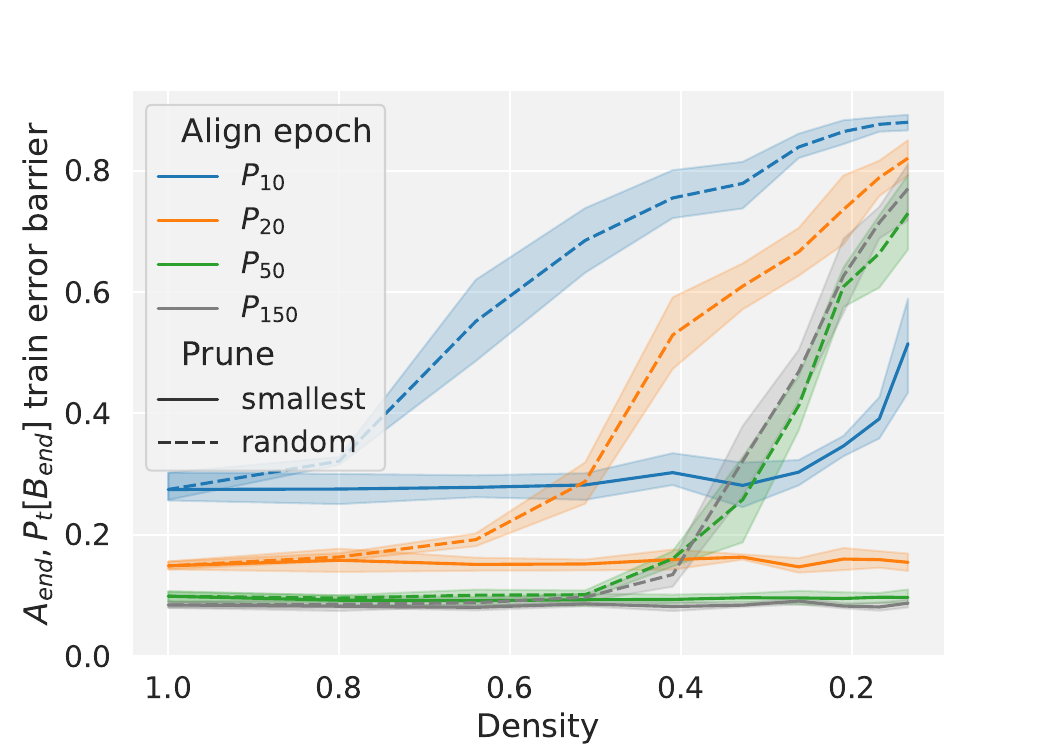}
    \includegraphics[width=.4\textwidth]{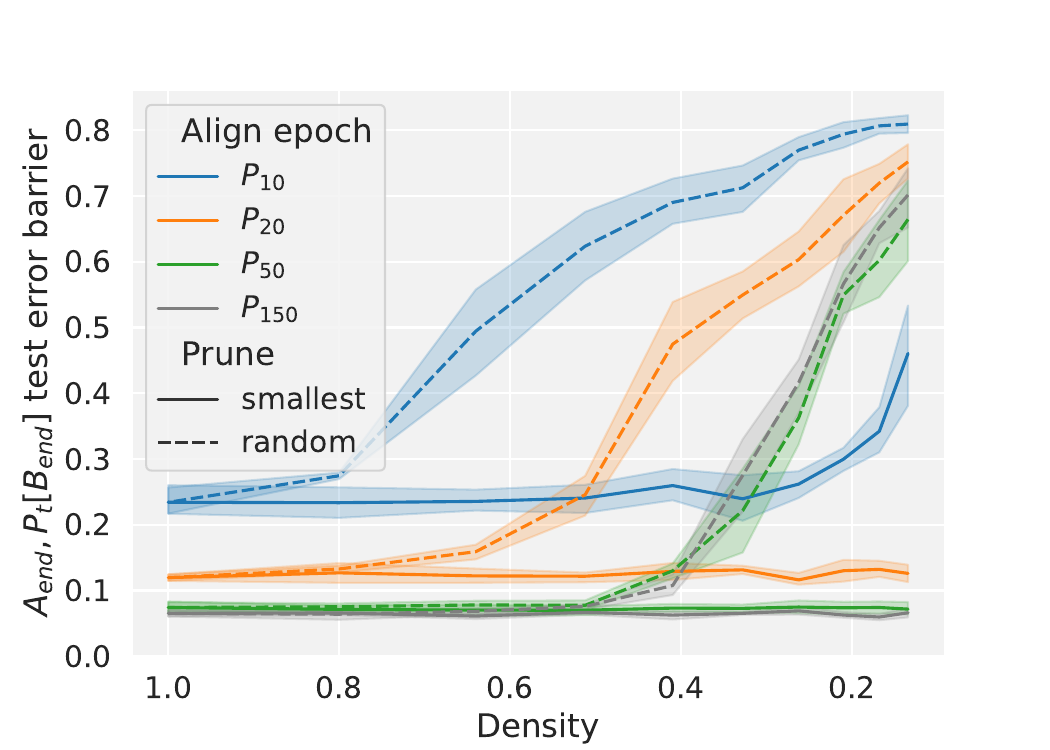}
    \includegraphics[width=.4\textwidth]{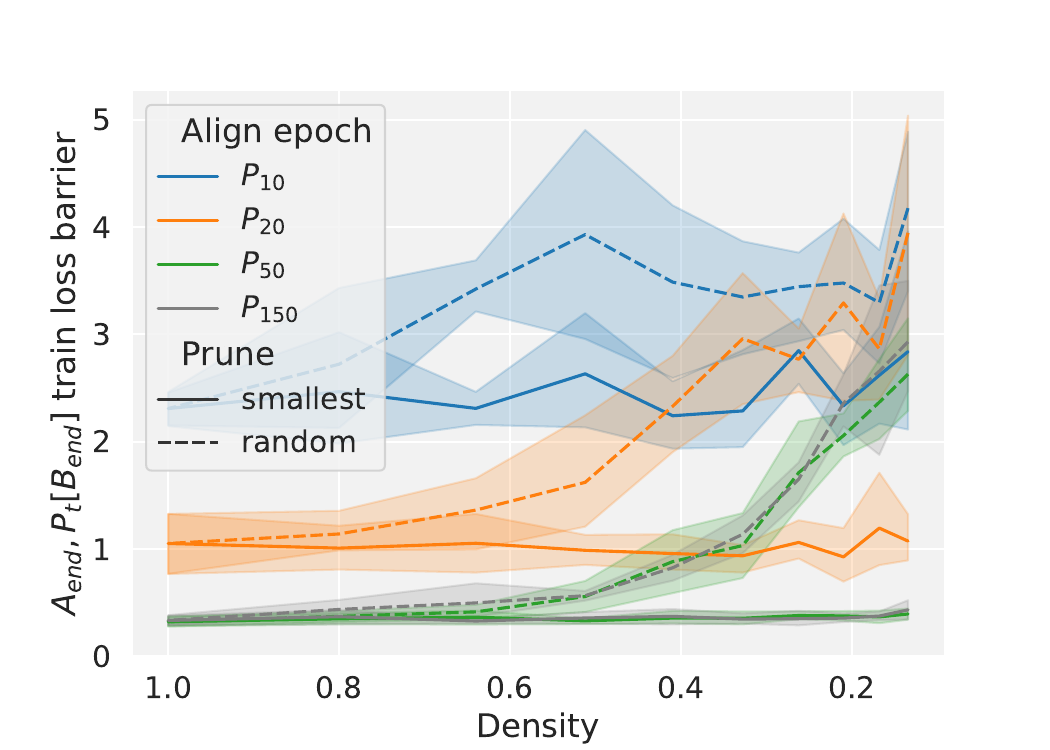}
    \includegraphics[width=.4\textwidth]{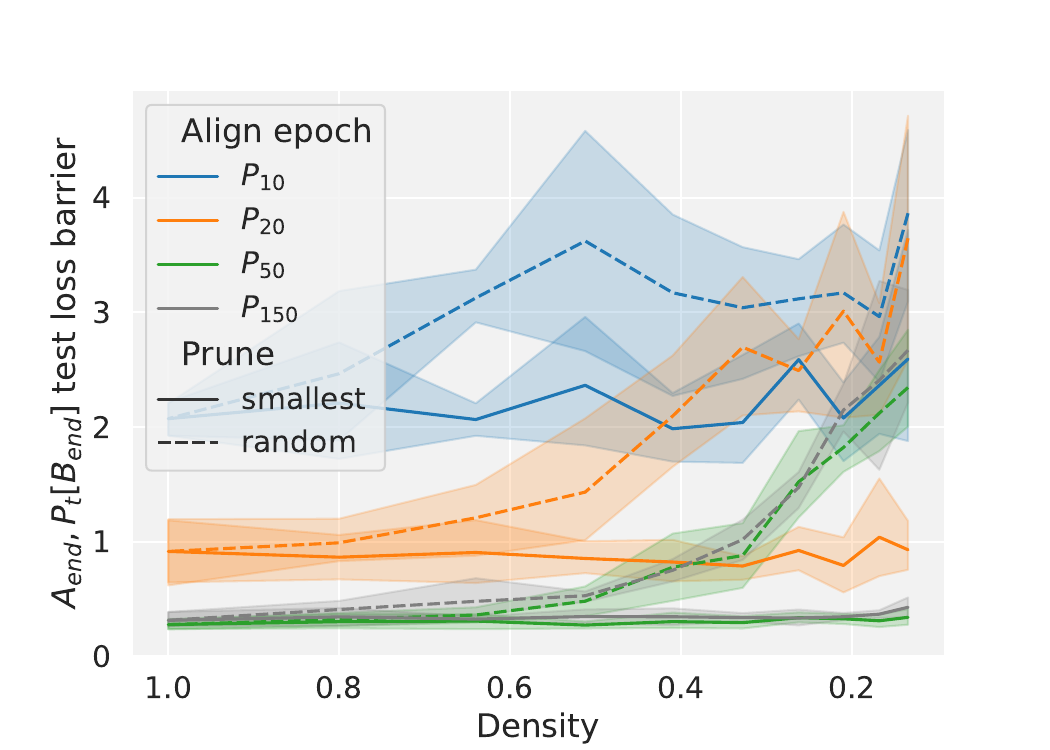}
    \includegraphics[width=.4\textwidth]{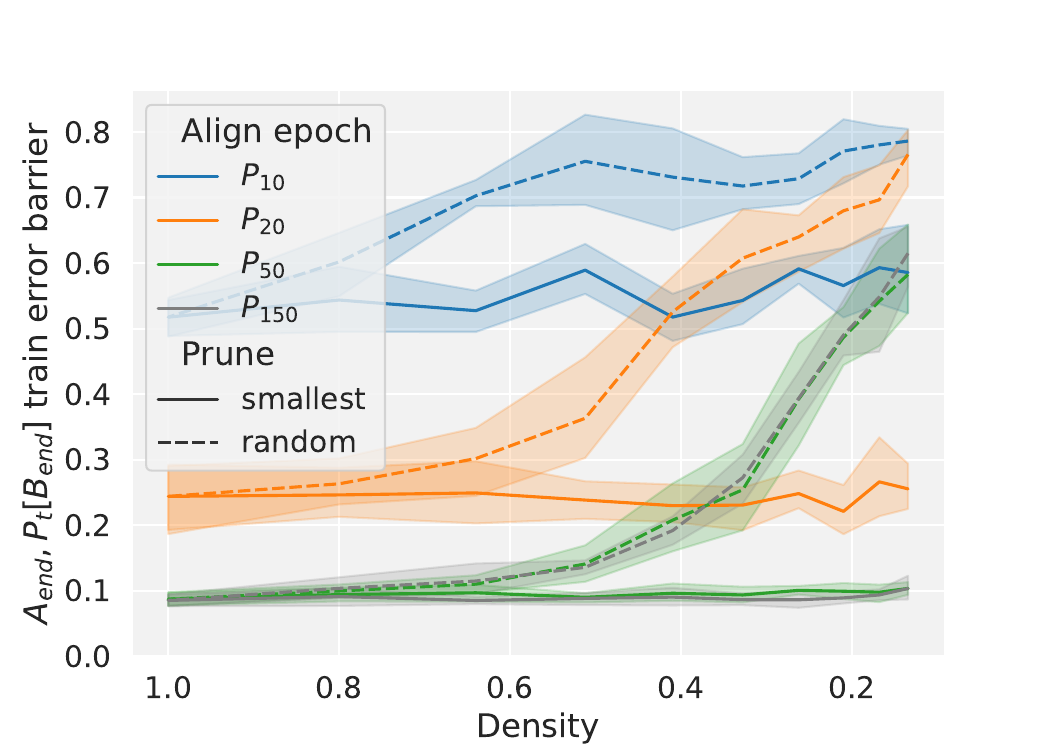}
    \includegraphics[width=.4\textwidth]{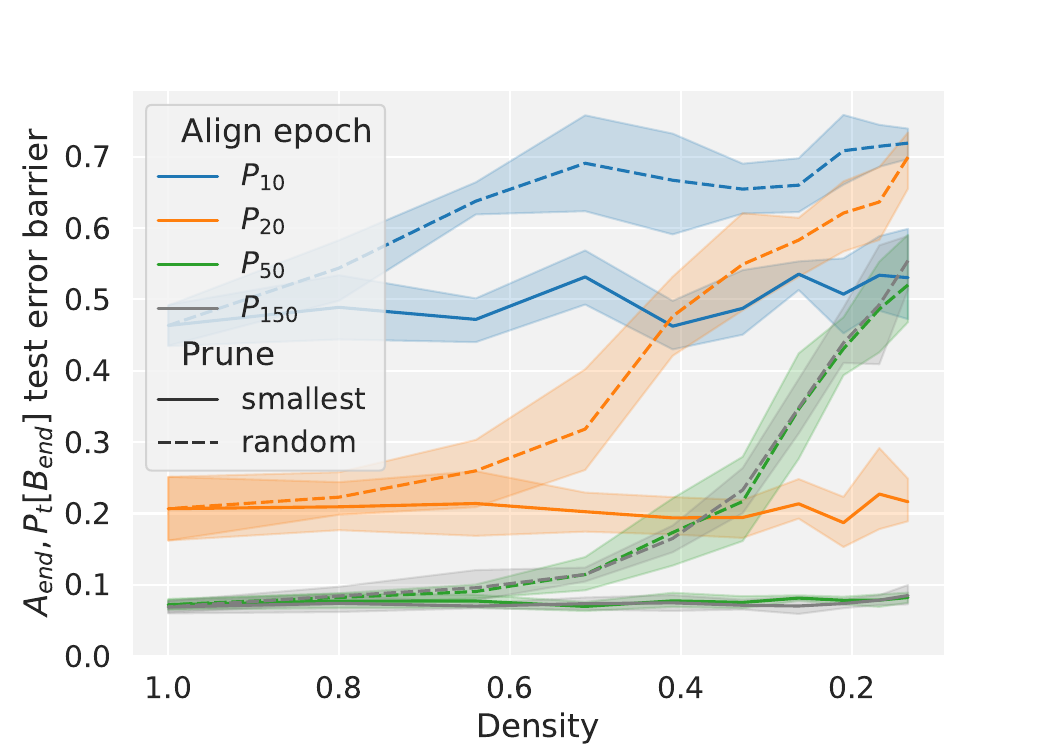}
\caption{
    Effect of magnitude pruning on performance of weight matching algorithm. Loss (error) barrier (y-axis) for networks after training which are aligned with a permutation found via weight matching. Permutations are computed on checkpoints at different training epochs (colors) which are first sparsified (x-axis) via random (dashed) or magnitude (solid) pruning.
}
\label{fig:app-prune-then-align}
\end{center}
\end{figure}

\clearpage
\subsection{Weight matching aligns initial layers earlier in training}
\label{sec:partialalign}

Our findings in \cref{fig:trajectoryalignment} show that there is a permutation early in training that eliminates the error barrier between a pair of networks.
However, \TBD{whereas activation alignment is able to find this permutation early in training, }weight alignment can only find this permutation near the end of training.
We try to determine \emph{where} the key differences in the network at early epochs that prevents weight matching from matching enough features to eliminate the error barrier.
Based on observations reported in prior work (e.g.~\cite{raghu2017svcca,morcos2018insights,alain2016understanding,baldock2021deep}), we hypothesize that the features in layers closer to the input are revealed first.

\paragraph{Layerwise weight matching experiment.}

To test when different layers are aligned by weight matching during training, we compute permutations based on \emph{partially-rewound networks}, in which we align a model using a combination of weights from two different training times.
In the \defn{bottom-up} version, we rewind the weights of the first $k$ layers (exclusive) of two networks to $A_t$ and $B_t$, keeping the remaining weights at their values from the end of training.
This results in a permutation that is analogous to permuting the first $k$ layers according to $P_t$, and the remainder according to $P_{end}$.
In the \defn{top-down} version, we rewind the opposite set of weights as the \defn{bottom-up} version, so that the first $k$ layers (inclusive) are from $A_{\mathrm{end}}$ or $B_{\mathrm{end}}$, and the rest are rewound to time $t$.
We also include a \defn{put-in} permutation that only rewinds the $k^{\mathrm{th}}$ layer, and \defn{leave-out} permutation that rewinds all but the $k^{\mathrm{th}}$ layer.
We apply these permutations at various layer thresholds $k$ and rewind points $t$ and compare their barriers at the end of training in \cref{fig:partialalignment,fig:partialalignment-error}.

We first note that at all rewind times, the error barriers of every type of partially rewound permutation is bounded between that of applying $P_t$ entirely (without concatenating with $P_{\mathrm{end}}$) and applying $P_{\mathrm{end}}$ entirely.
This suggests that when considering the alignment of any single layer or subset of layers, $P_{\mathrm{end}}$ is always strictly better than $P_t$ for reducing the error barrier at the end of training.

Analyzing the bottom-up permutations (\cref{fig:partialalignment,fig:partialalignment-error} left), we see that as the number $k$ of input layers to rewind increases, the error barriers at all rewind times increase.
Furthermore, the time at which each bottom-up permutation achieves the same error barrier as $P_{\mathrm{end}}$ also increases with $k$.
This indicates that the key features of the permutation are revealed in ascending layer order as hypothesized.

The evolution of weights in input layers affects error barriers more than that of output layers.
When splitting the network at the midpoint $k=8$, a comparison of the bottom-up and top-down permutations shows that rewinding the output layers (top-down) of the network does not increase error barrier as much as rewinding the input layers (bottom-up) of the network.
In particular, the first layers up to $k=4$ seem to have the largest influence on error barrier across all plots.

Finally, error barrier is more strongly coupled with the number and depth of rewound layers when only a few layers are rewound.
This can be seen in the bottom-up and put-in figures, which have less rewound layers than the top-down and leave-out figures, but are also more differentiated in error barrier with respect to different values of $k$.
In contrast, error barriers have almost no relation to $k$ in the top-down and leave-out figures (apart from $k=8$ in top-down and $k=2$ in leave-out), and instead depend mainly on the rewind epoch.
This suggests that permutation features are not isolated to particular layers, but rather encompass a large subset of layers or the entire network.

\begin{figure*}[ht]
\begin{center}
\includegraphics[width=\textwidth]{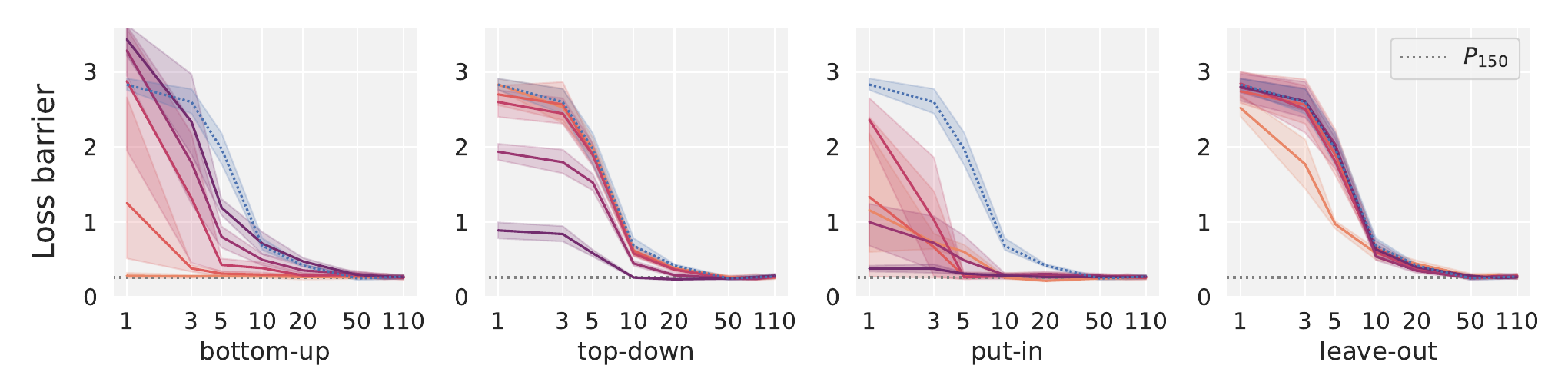}
\includegraphics[width=\textwidth]{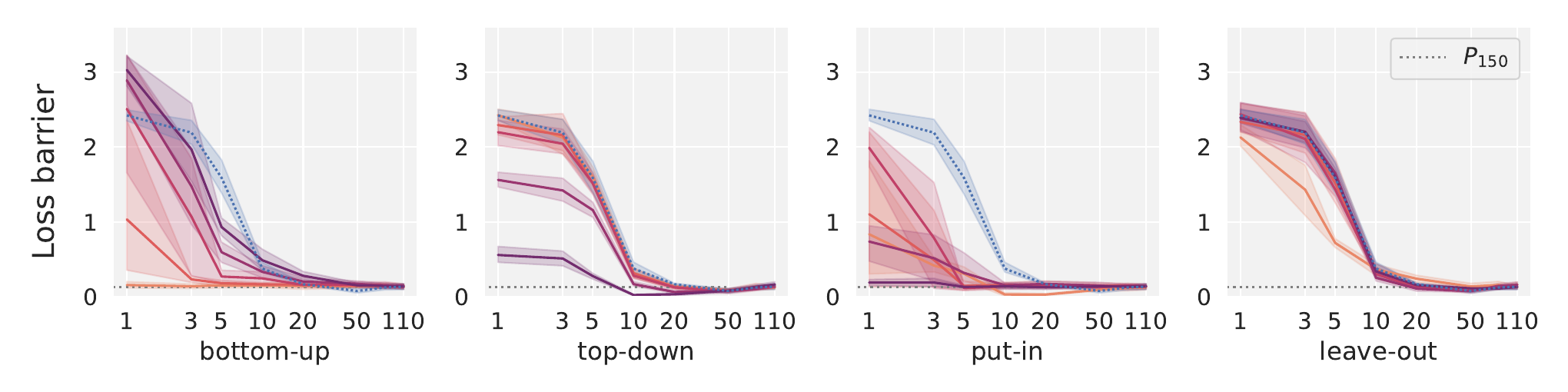}
\includegraphics[width=\textwidth]{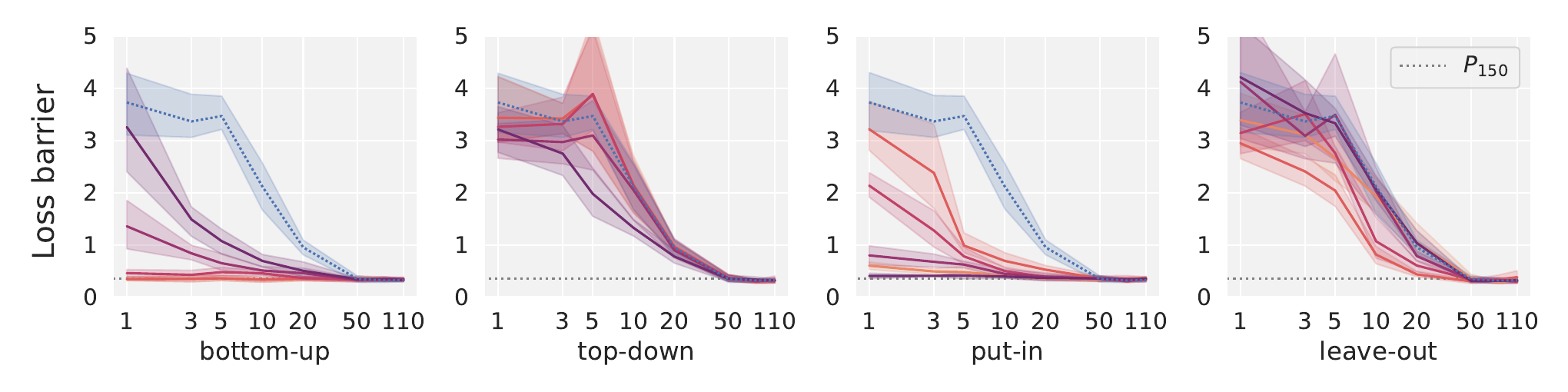}
\includegraphics[width=\textwidth]{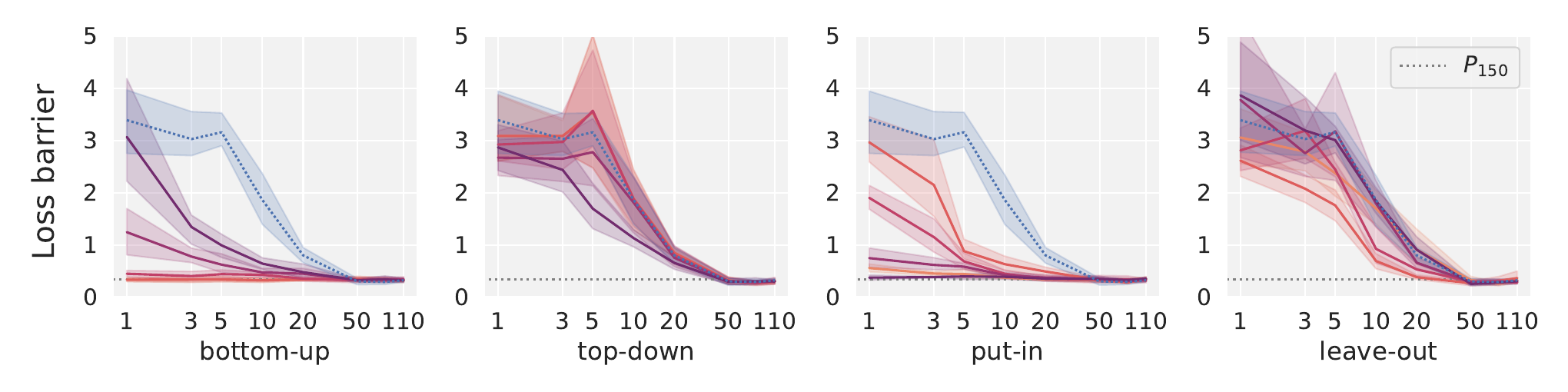}
\caption{
Loss barriers at the end of training between a pair of networks under various permutations.
From left to right, the permutations are: bottom-up, which concatenates the permutation $P_t$ from input up to layer $k$ indicated in the legend, and uses $P_{\mathrm{end}}$ for the rest; top-down, which is the opposite of bottom-up and uses $P_t$ for the layers $k$ and after only; put-in, which only uses the permutation $P_t$ for the $k^{\mathrm{th}}$ layer and $P_{\mathrm{end}}$ otherwise; and leave-out, which uses all of $P_t$ except for the $k^{\mathrm{th}}$ layer which uses $P_{\mathrm{end}}$ instead.
\textbf{Rows 1 and 2:} train and test barriers for VGG-16 trained on CIFAR-10.
\textbf{Rows 3 and 4:} train and test barriers for ResNet-20 with $4 \times$ width (64 hidden channels) trained on CIFAR-10.
Permutations $P_t$ and $P_{\mathrm{end}}$ computed at time $t$ (x-axis) are presented as baselines (dotted lines) in each plot.
}
\label{fig:partialalignment}
\end{center}
\vskip -0.2in
\end{figure*}

\begin{figure*}[ht]
\begin{center}
\includegraphics[width=\textwidth]{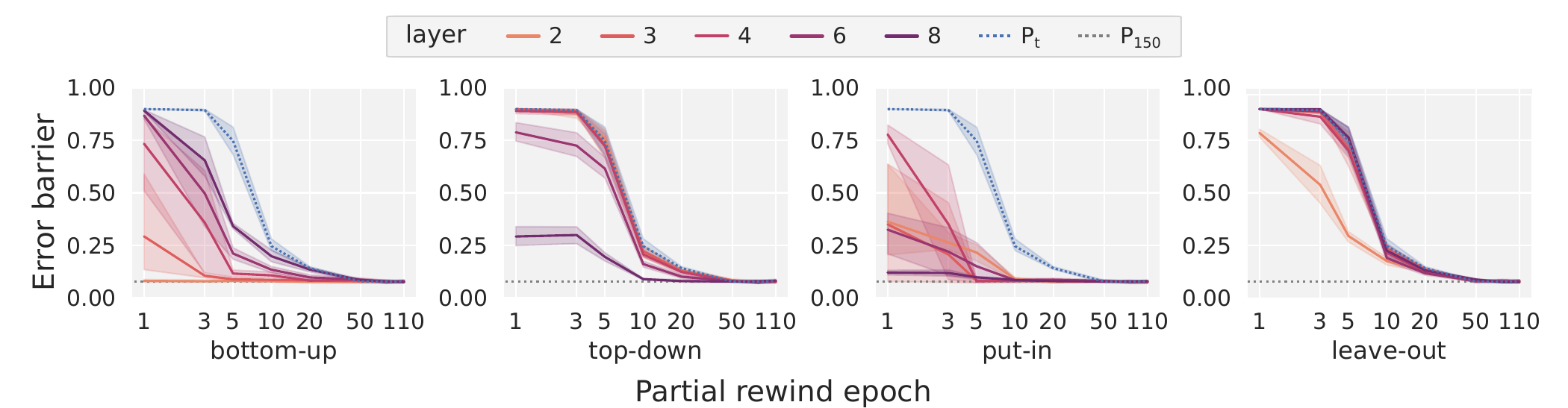}
\includegraphics[width=\textwidth]{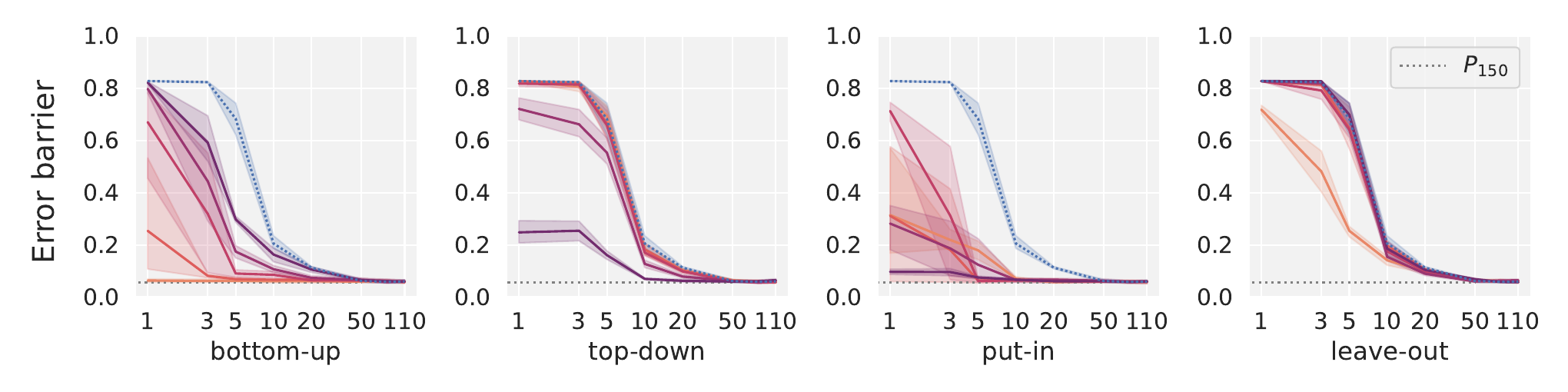}
\includegraphics[width=\textwidth]{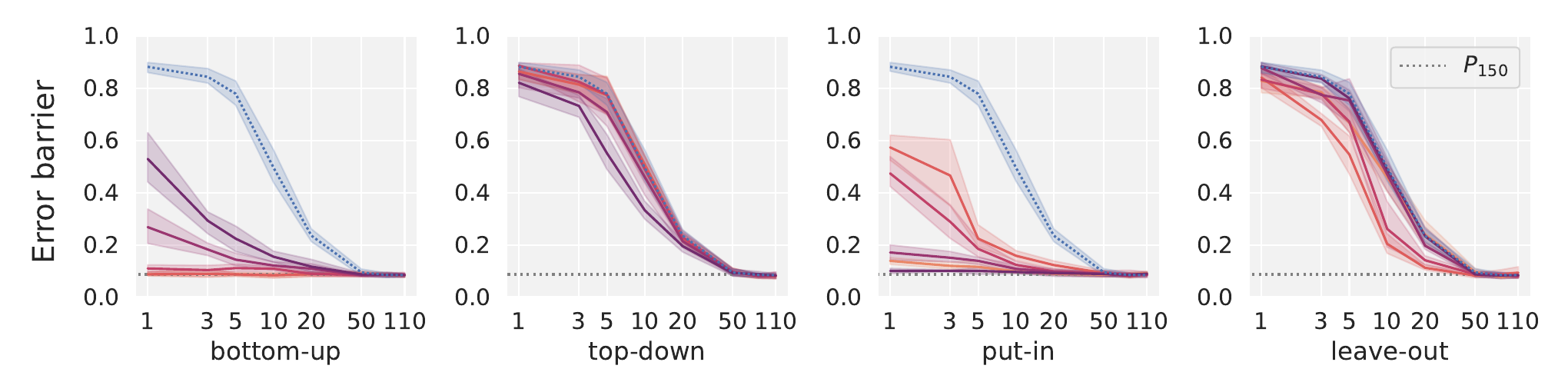}
\includegraphics[width=\textwidth]{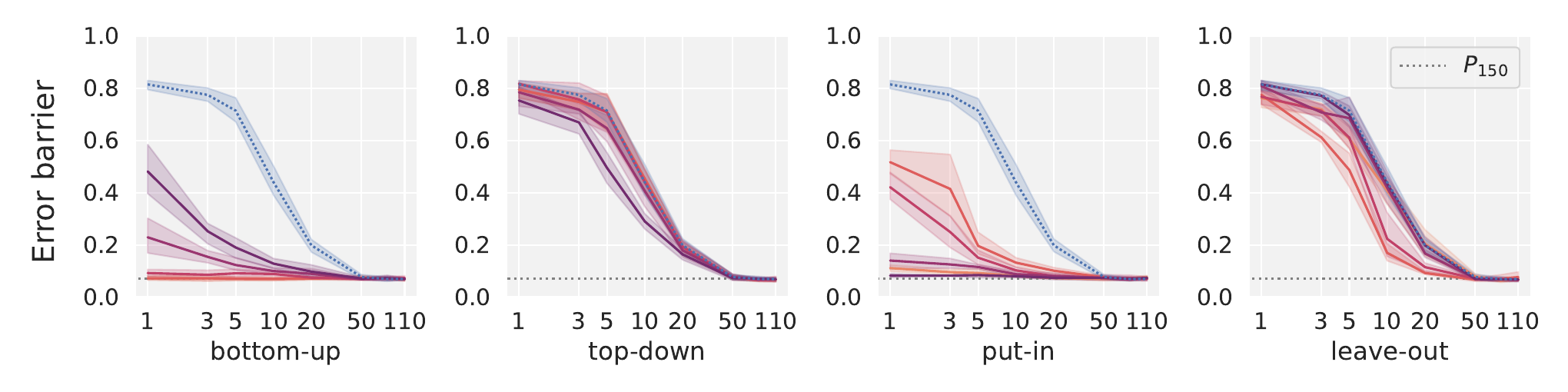}
\caption{
Error barriers at the end of training between a pair of networks under various permutations.
Format follows \cref{fig:partialalignment}.
\textbf{Rows 1 and 2:} train and test error barriers for VGG-16 trained on CIFAR-10.
\textbf{Rows 3 and 4:} train and test error barriers for ResNet-20 with $4 \times$ width (64 hidden channels) trained on CIFAR-10.
}
\label{fig:partialalignment-error}
\end{center}
\end{figure*}

\clearpage
\subsection{Network stability at initialization is neither necessary nor sufficient for permutation alignment}
\label{app:matchingifstable}

In this section we show that stability at initialization is neither necessary nor sufficient for permutation alignment. We look at the following training regimes for VGG16 mode architecture trained on CIFAR-10 dataset: 
\begin{enumerate}
    \item No learning rate warm-up; peak learning-rate set to 1e-3; weight-decay coefficient set to 1e-4.
    \item Single epoch of learning rate warm-up; peak learning-rate set to 1e-1; weight-decay coefficient set to 0 (no weight decay).
\end{enumerate}
In \cref{fig:initstabilitynotsufficient}, we give results for training regime (1). We see that in this training regime, models are stable at initialization w.r.t. SGD noise, i.e. with same initialization but different data orders for SGD, at the end of training, models belong in the same linearly connected mode. From the plots, we can see that both activation matching and weight matching fail to find a permutation that leads to linear mode connectivity. 

In \cref{fig:initstabilitynotnecessary}, we give results for training regime (2). We see that in this training regime, models are not stable at initialization w.r.t. SGD noise, i.e. with same initialization but different data orders for SGD, models at the end of training are not linearly connected. From the plots, we see that activation matching succeeds is finding a permutation that linearly connects the entire training trajectory. Thus, showing that stability at initialization is not a necessary condition for finding permutations that lead to simultaneous linear mode connectivity for entire training trajectories. 

 \begin{figure}[h!]
\begin{center}
\includegraphics[width=0.8\textwidth]{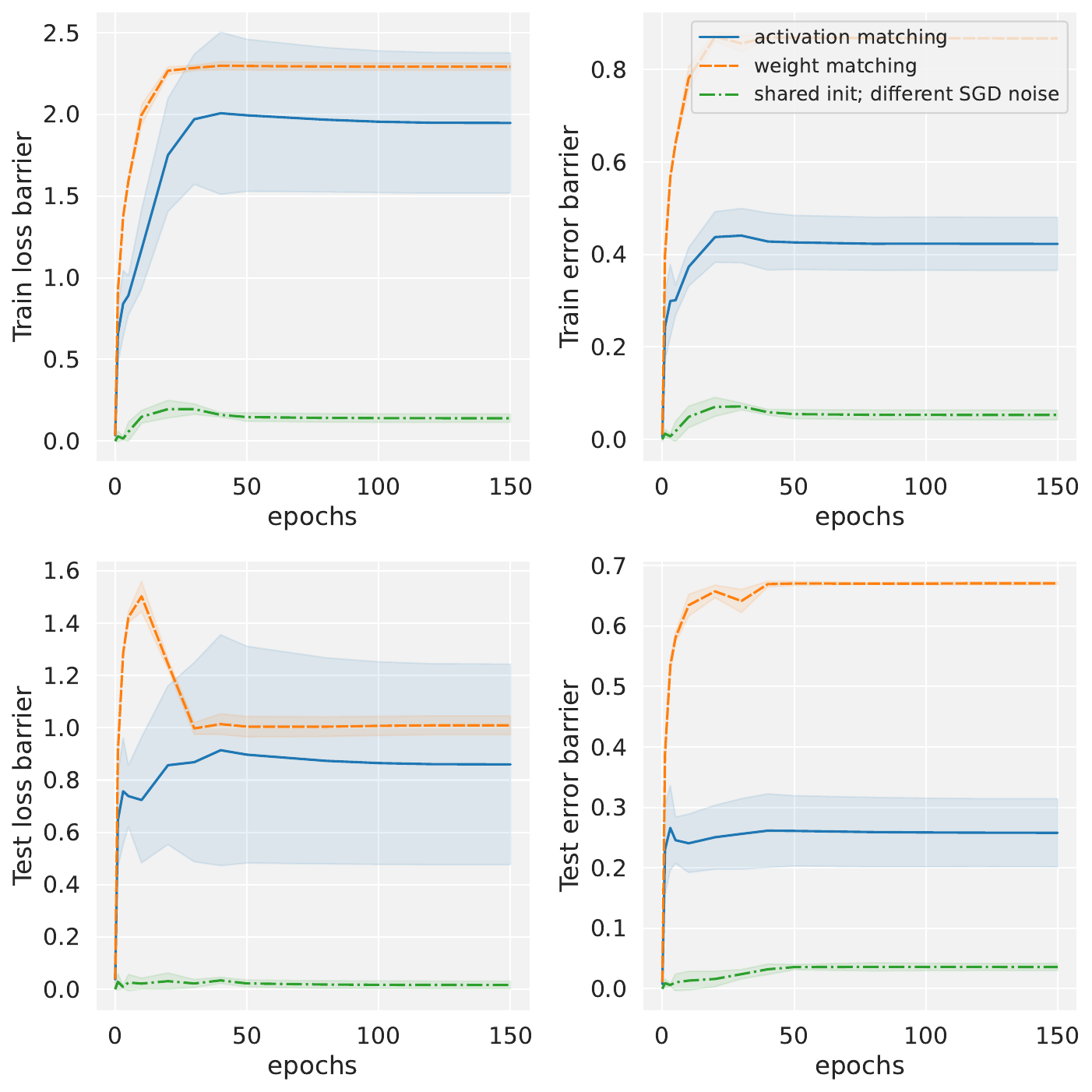}

\caption{
 Stability at initialization is not sufficient for finding permutation symmetries.
}
\label{fig:initstabilitynotsufficient}
\end{center}
\end{figure}

\begin{figure}[h!]
\begin{center}
\includegraphics[width=0.8\textwidth]{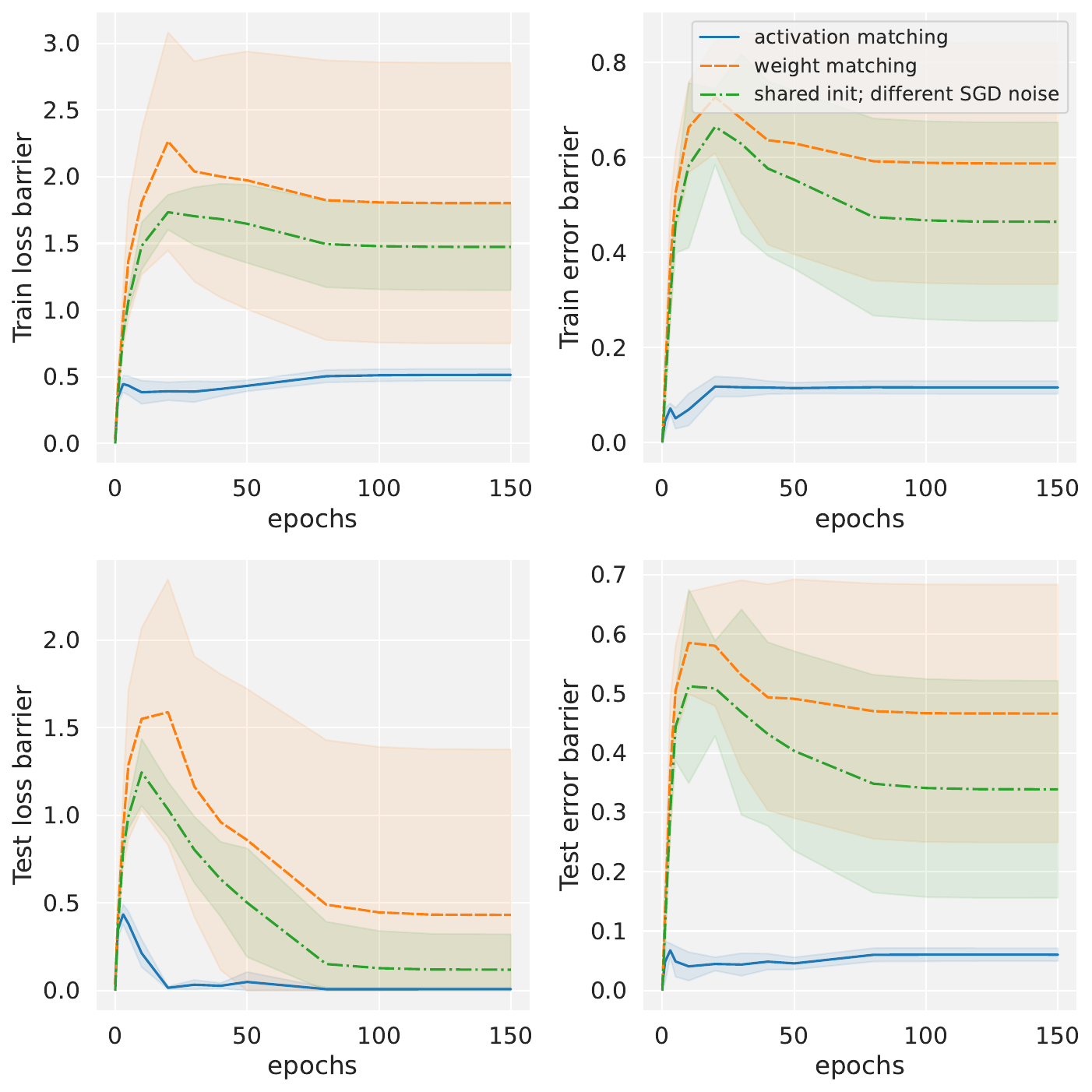}
\caption{
 Stability at initialization is not necessary for finding permutation symmetries.  
}
\label{fig:initstabilitynotnecessary}
\end{center}
\end{figure}

\clearpage
\paragraph{Fixed points of permutations found by weight matching.}

To verify that the evolution of error/loss barriers over training trajectories in \cref{sec:weightmatchingfails} correspond to changes in the actual permutations found by weight matching, we compute the fraction of fixed points shared by $P_t$ and $P_{\mathrm{end}}$ throughout training (see \cref{subsec:fixed-points} for details).

In \cref{fig:fixed-points} (\textbf{top}), we find that generally speaking, permutations from later in training have more fixed points (corresponding to lower error barrier), and that throughout training the fraction of fixed points is highest in early layers.

Additionally, to verify that the error barrier trends observed in layerwise weight matching (\cref{sec:partialalign}) correspond to changes between the permutations of different layers, we also compute the fraction of fixed points per layer as
\[
\frac{\operatorname{FP}_i}{n_i} = \frac{\operatorname{trace} \left( (P_t^{(i)})^{T} P_{\mathrm{end}}^{(i)} \right)}{n_i},
\]
where $P^{(i)}$ is the permutation for layer $i$ and $n_i$ is the size of the permutation.

In \cref{fig:fixed-points} (\textbf{bottom}), we find that the permutation of lower layers has a higher fraction of fixed points at earlier $t$.
This indicates that the permutation of lower layers more rapidly approaches $P_{\mathrm{end}}$ than the permutation of higher layers.

Combined with the results from \cref{fig:partialalignment,fig:partialalignment-error}, this provides strong evidence that weight alignment aligns features that are revealed in consecutive layers as training progresses.

\begin{figure*}[ht]
\begin{center}
\includegraphics[width=0.5\textwidth]{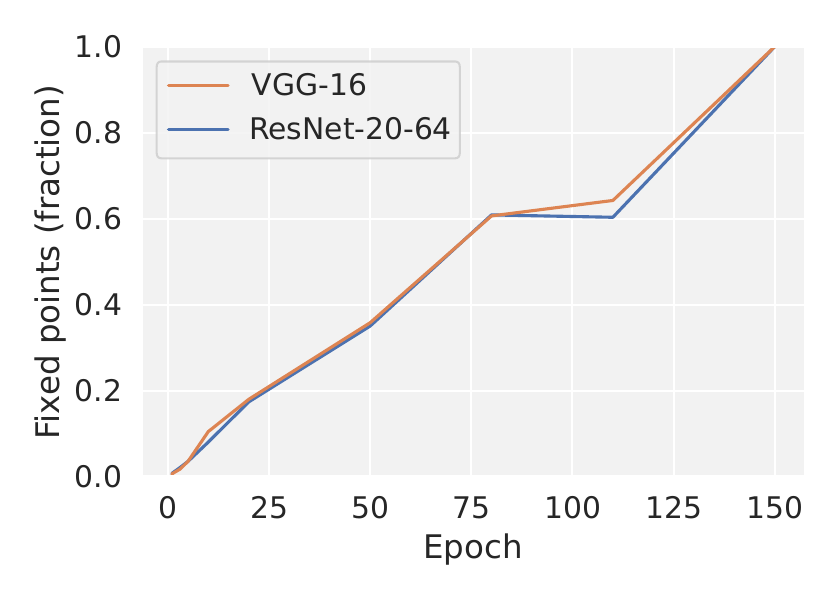}
\vskip -0.2in
\includegraphics[width=0.5\textwidth]{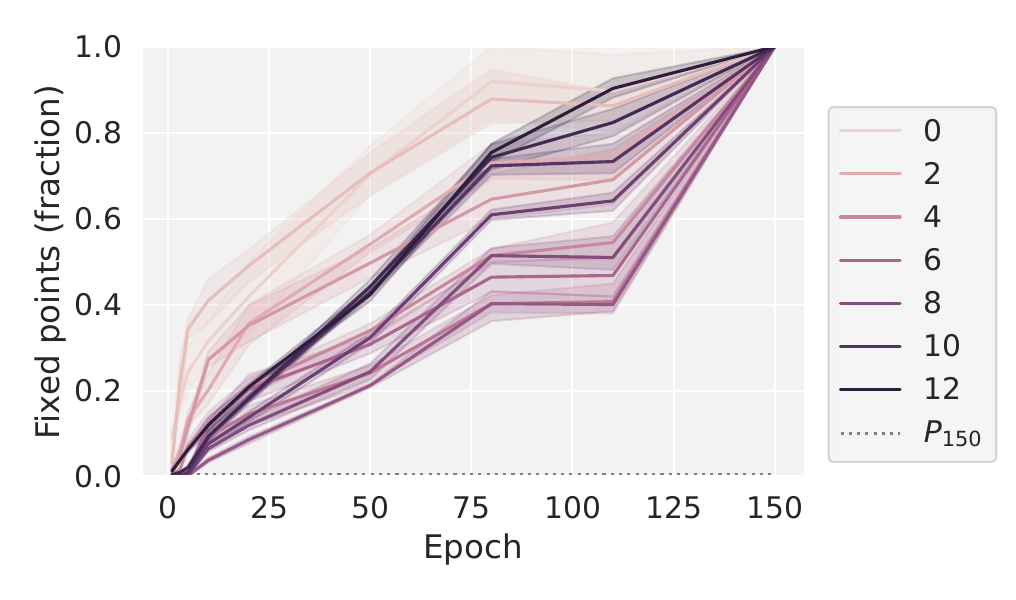}%
\includegraphics[width=0.5\textwidth]{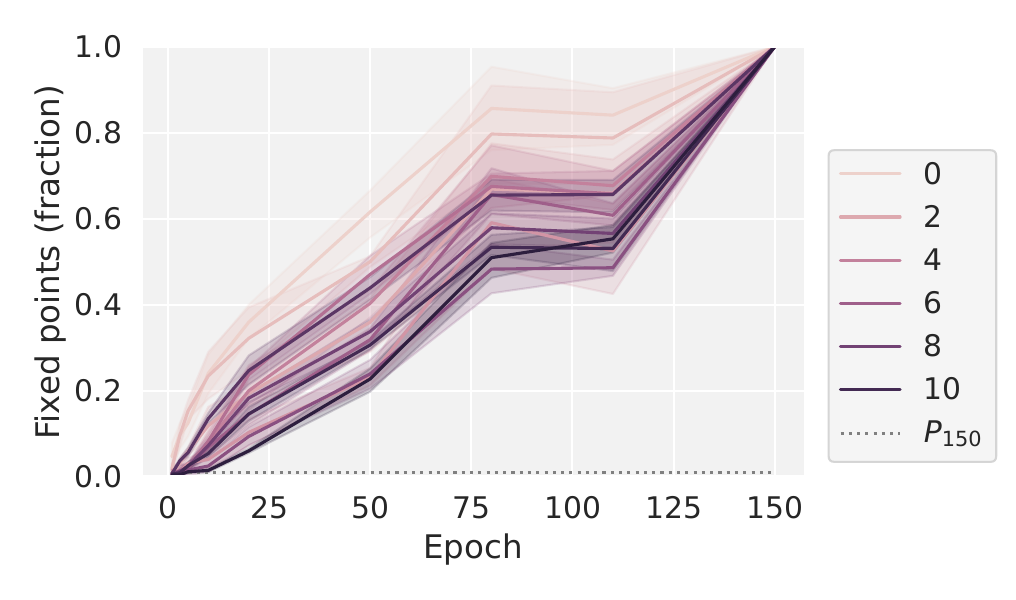}
\caption{
Fixed points for permutations $P_t$ and $P_{\mathrm{end}}$, over selected training epochs $t$ (x-axis).
\textbf{(Top)} Total fraction of fixed points between permutations.
\textbf{(Left)} Per-layer fixed points for VGG-16 models.
\textbf{(Right)} Per-layer fixed points for ResNet-20-64 models.
Dotted grey line indicates the baseline number of fixed points between the identity permutation and $P_{\mathrm{end}}$.
}
\label{fig:fixed-points}
\end{center}
\vskip -0.2in
\end{figure*}

\clearpage
\section{Landscape visualization plots}
\label{app:landscapevis}

\begin{figure}[ht]
\vskip 0.2in
\begin{center}
\centerline{\includegraphics[width=\textwidth]{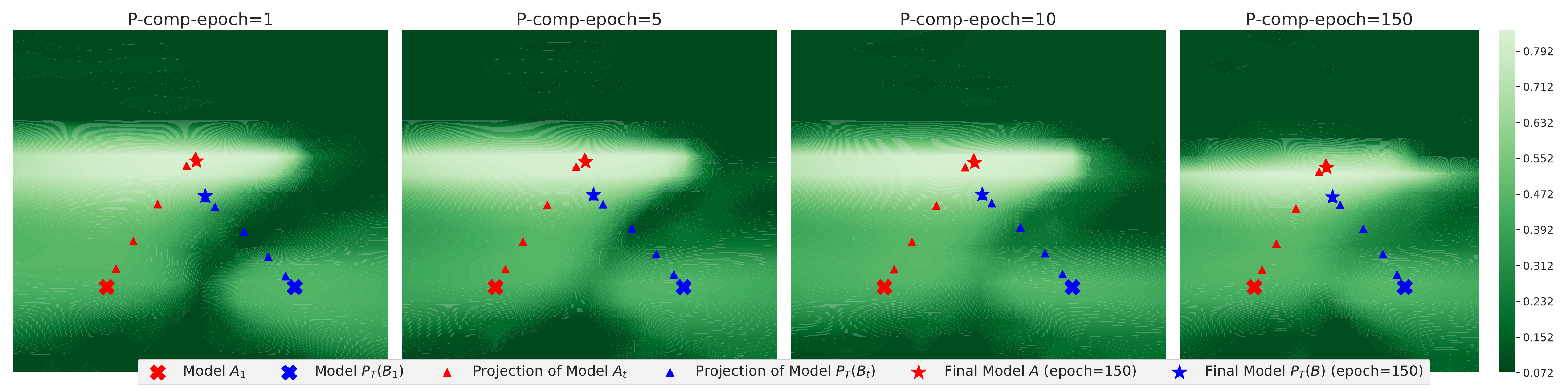}}
\caption{
Loss landscape projections. Each projection is computed based on $A_{150}$ (red star), $A_1$ (red cross), $P_t(B_1)$ (blue cross), for $t \in \{1,5,10,150\}$ from left to right. Note that epoch 150 is the end of training.
}
\label{fig:losslandscape}
\end{center}
\vskip -0.2in
\end{figure}

In \cref{fig:losslandscape}, we visualize the training trajectories in two dimensional projections of the loss landscape. 
Each of the projections is determined by the following three points: $A_{\mathrm{end}}$, $A_t$, and $P_{t^*}(B_t)$, where $t^*$ changes from left to right, and $t=1$ is fixed.
We see that for the projections computed early in training, the trajectories of $A_t$ and $B_t$ are initially separated by a high loss barrier, but $A_{\mathrm{end}}$ and $B_{\mathrm{end}}$ still end up in the same basin.
On the other hand, in the projections computed using a permutation found at a later $t^*$, the two trajectories are connected even from early in training. 
Additional visualizations for different projections can be found in \cref{app:landscapevis} (\cref{fig:landscapeoptbig}), which also includes a cartoon visualization of the landscape projections that we consider.

\begin{figure}[ht]
\vskip 0.2in
\begin{center}
\centerline{\includegraphics[width=0.9\columnwidth]{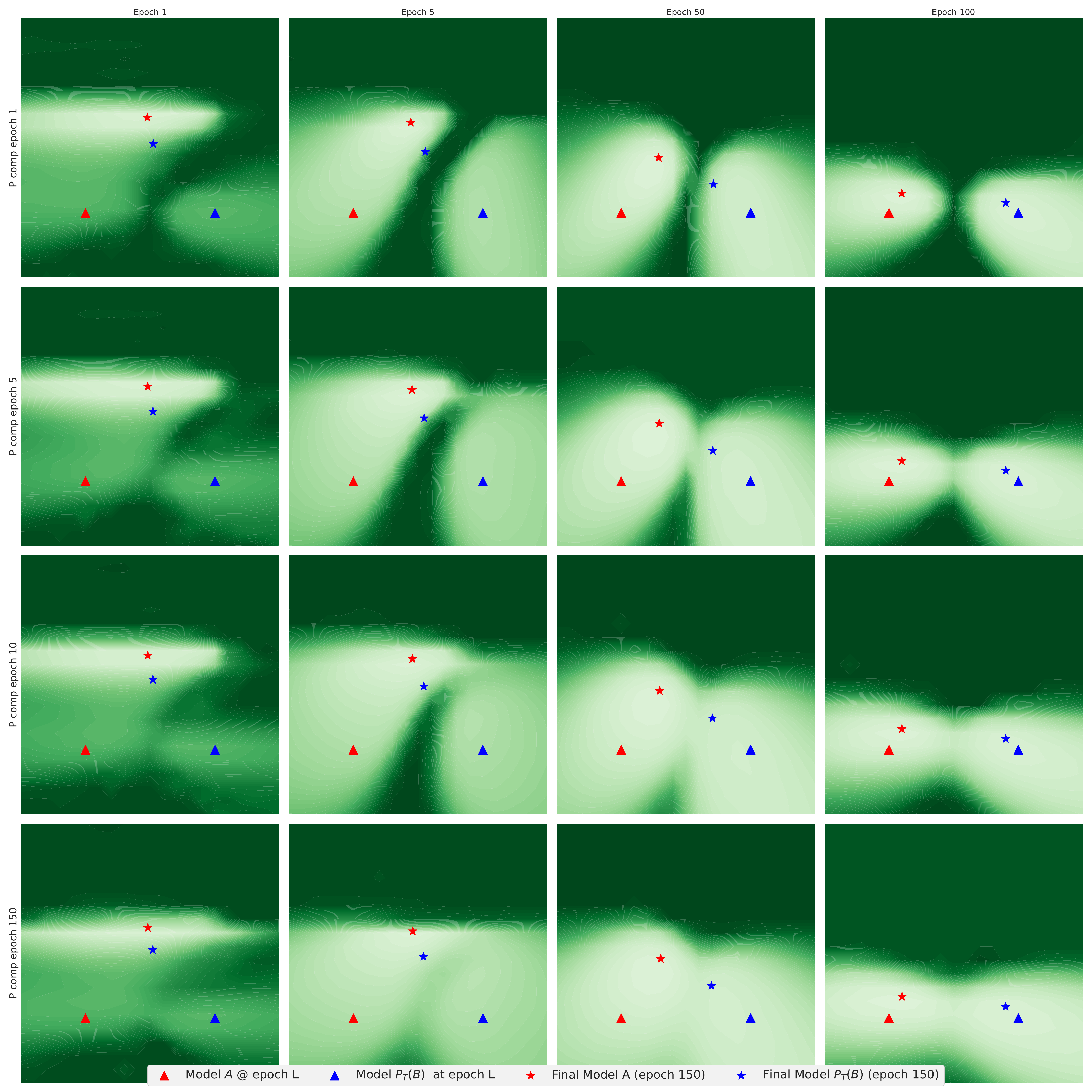}}
\centerline{\includegraphics[width=0.35\columnwidth]{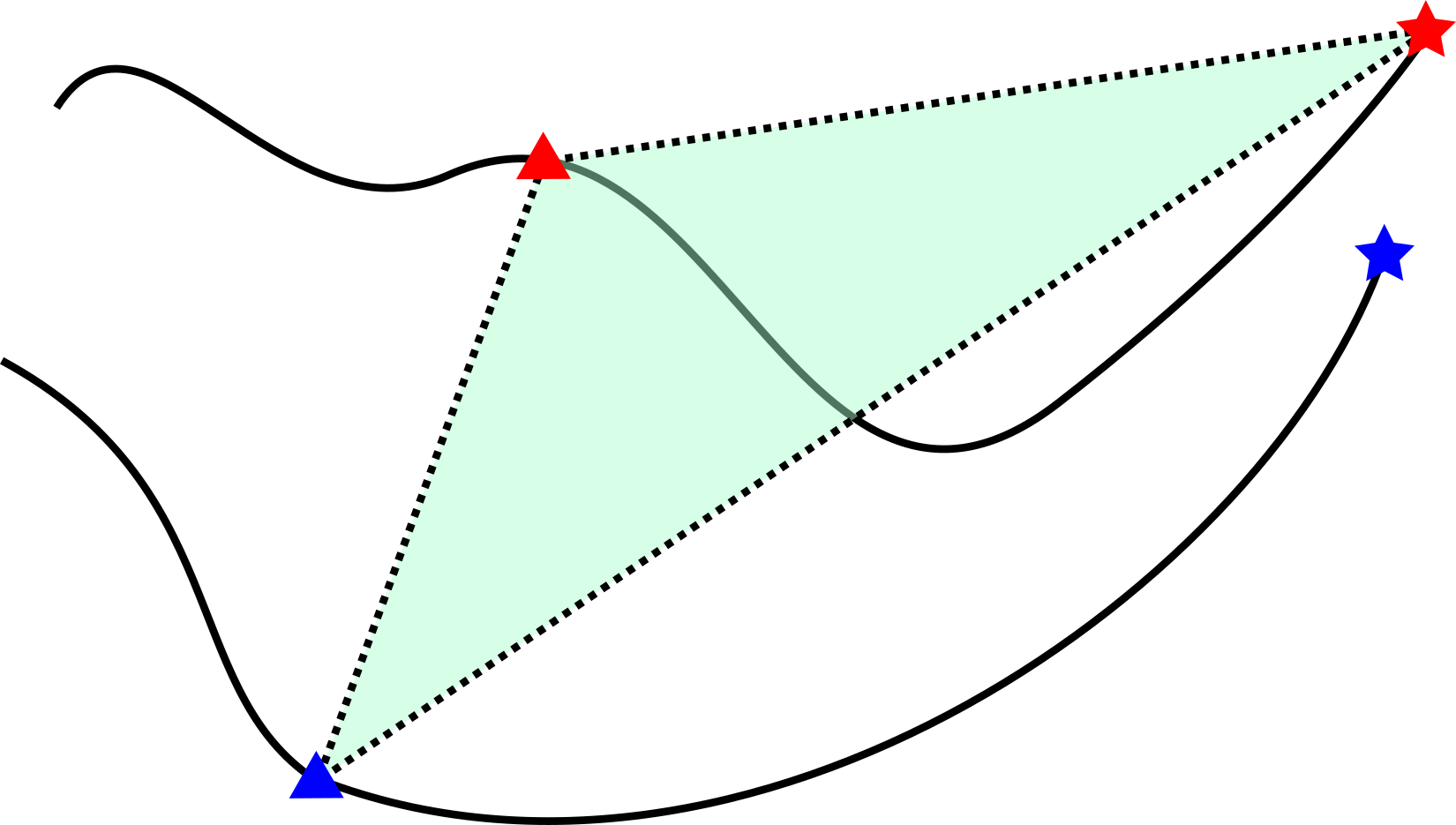}}
\caption{
\textbf{(Top)} Loss landscape visualizations for different projections and at different training times. All projections are created using the following three points: $A_{\mathrm{end}}$, $A_{t}$, and $P_{t^*}(B_{t})$. Each row corresponds to a different $t^{*} \in \{1,5,10,150\}$, where $t^*=150$ corresponds to the end of training. Each column corresponds to a different model checkpoint $t$.
\textbf{(Bottom)} Cartoon showing loss landscape cuts for the top figure.
}
\label{fig:landscapeoptbig}
\end{center}
\vskip -0.2in
\end{figure}

\end{document}